\documentclass[sn-apa]{sn-jnl}

\jyear{2021}%

\theoremstyle{thmstyleone}%
\newtheorem{theorem}{Theorem}
\newtheorem{lemma}[theorem]{Lemma}%

\theoremstyle{thmstyletwo}%

\theoremstyle{thmstylethree}%
\newtheorem{assumption}{Assumption}%

\RequirePackage{xspace}
\RequirePackage{mathtools}
\RequirePackage{amsmath, amsfonts, amssymb}
\RequirePackage{dsfont}
\DeclareRobustCommand{\eg}{e.g.,\@\xspace}
\DeclareRobustCommand{\ie}{i.e.,\@\xspace}
\DeclareRobustCommand{\wrt}{w.r.t.\@\xspace}
\DeclareRobustCommand{\cf}{cf.\@\xspace}
\DeclareRobustCommand{\iid}{i.i.d.\@\xspace}

\newcommand{\wt}[1]{\widetilde{#1}}
\newcommand{\vtheta}{\boldsymbol{\theta}}

\newcommand{\vphi}{\boldsymbol{\phi}}

\newcommand{\vx}{\boldsymbol{x}}
\newcommand{\vy}{\boldsymbol{y}}
\newcommand{\vh}{\boldsymbol{h}}

\newcommand{\Reals}{\mathbb{R}}
\newcommand{\Naturals}{\mathbb{N}}
\newcommand{\de}{\,\mathrm{d}}
\newcommand{\Hess}{\nabla^2}

\newcommand{\norm}[2][]{\left\Vert #2\right\Vert _{#1}}

\newcommand{\Xspace}{\mathcal{X}}

\newcommand{\Id}{\mathbb{I}}
\newcommand{\indi}[1]{\mathds{1}\left\{#1\right\}}
\DeclareMathOperator*{\EV}{\mathbb{E}}
\newcommand{\EVV}[2][]{\EV_{#1}\left[{#2}\right]}
\newcommand{\Gauss}{\mathcal{N}}
\DeclareMathOperator*{\Var}{\mathbb{V}\text{ar}}

\newcommand{\Pro}{\mathbb{P}}

\newcommand{\simiid}{\stackrel{\text{iid}}{\sim}}
\newcommand{\kl}{\mathop{\mathcal{KL}}}
\newcommand{\Dataset}{\ensuremath{\mathcal{D}}}
\newcommand{\hnabla}{\widehat{\nabla}}
\newcommand{\Aspace}{\mathcal{A}}
\newcommand{\Sspace}{\mathcal{S}}
\newcommand{\Tran}{p}
\newcommand{\Init}{\mu}
\newcommand{\Rew}{r}
\newcommand{\Rmax}{R}
\newcommand{\phimax}{M}

\newcommand{\Stat}{\rho}
\newcommand{\sm}{\xi_1}
\newcommand{\smm}{\xi_2}
\newcommand{\smmm}{\xi_3}

\newcommand{\Vv}{\mathbb{V}}

\usepackage{thm-restate}

\usepackage{enumitem}

\usepackage{nicefrac}

\usepackage{algorithm,algorithmicx,algpseudocode}
\algdef{SE}[DOWHILE]{Do}{doWhile}{\algorithmicdo}[1]{\algorithmicwhile\ #1}%

\usepackage{subfig}

\usepackage{etoolbox}
\makeatletter
\patchcmd{\ps@headings}
{\hbox to \hsize{\hfill Springer Nature 2021 \LaTeX\ template\hfill}}
{\hbox to \hsize{}}
{}
{}
\patchcmd{\ps@headings}
{\hbox to \hsize{\hfill Springer Nature 2021 \LaTeX\ template\hfill}}
{\hbox to \hsize{}}
{}
{}
\patchcmd{\ps@titlepage}
{\hbox to \hsize{\hfill Springer Nature 2021 \LaTeX\ template\hfill}}
{\hbox to \hsize{}}
{}
{}
\makeatother

\begin{document}

\title[Smoothing Policies and Safe Policy Gradients]{Smoothing Policies and Safe Policy Gradients}


\author*[1]{\fnm{Matteo} \sur{Papini}}\email{matteo.papini@upf.edu}

\author[2]{\fnm{Matteo} \sur{Pirotta}}\email{pirotta@fb.com}

\author[3]{\fnm{Marcello} \sur{Restelli}}\email{marcello.restelli@polimi.it}

\affil[1]{Universitat Pompeu Fabra, Barcelona, Spain}

\affil[2]{Meta}

\affil[3]{Politecnico di Milano, Milano, Italy}


\abstract{Policy Gradient (PG) algorithms are among the best candidates for the much-anticipated applications of reinforcement learning to real-world control tasks, such as robotics. However, the trial-and-error nature of these methods poses safety issues whenever the learning process itself must be performed on a physical system or involves any form of human-computer interaction. In this paper, we address a specific safety formulation, where both goals and dangers are encoded in a scalar reward signal and the learning agent is constrained to never worsen its performance, measured as the expected sum of rewards. By studying actor-only policy gradient from a stochastic optimization perspective, we establish improvement guarantees for a wide class of parametric policies, generalizing existing results on Gaussian policies. This, together with novel upper bounds on the variance of policy gradient estimators, allows us to identify meta-parameter schedules that guarantee monotonic improvement with high probability. The two key meta-parameters are the step size of the parameter updates and the batch size of the gradient estimates. Through a joint, adaptive selection of these meta-parameters, we obtain a policy gradient algorithm with monotonic improvement guarantees.}

\keywords{Reinforcement Learning, Safe Machine Learning, Policy Optimization, Policy Gradient, Monotonic Improvement}



\maketitle

\section{Introduction}\label{sec:intro}

Reinforcement Learning~(RL)~\citep{sutton2018reinforcement} has achieved astounding successes in games~\citep{mnih2015human,silver2018general,OpenAI_dota, alphastarblog}, matching or surpassing human performance in several occasions. However, the much-anticipated applications of RL to real-world tasks such as robotics~\citep{kober2013reinforcement}, autonomous driving~\citep{okuda2014survey} and finance~\citep{li2014online} seem still far. This technological delay may be due to the very nature of RL, which relies on the repeated interaction of the learning machine with the surrounding environment, \eg a manufacturing plant, a trafficked road, a stock market. The trial-and-error process resulting from this interaction is what makes RL so powerful and general. However, it also poses significant challenges in terms of sample efficiency~\citep{recht2018tour} and safety~\citep{amodei2016concrete}.

In reinforcement learning, the term \textit{safety} can actually refer to a variety of problems~\citep{garcia2015comprehensive}. The general concern is always the same: avoiding or limiting \textit{damage}. In financial applications, it is typically a loss of money. In robotics and autonomous driving, one should also consider direct damage to people and property. In this work, we do not make assumptions about the nature of the damage, but we assume it is entirely encoded in the scalar reward signal that is presented to the agent in order to evaluate its actions. Other works~\citep[\eg][]{turchetta2016safe} employ a distinct safety signal, separate from rewards.

A further distinction is necessary on the scope of safety constraints with respect to the agent's life. One may simply require the final behavior, the one that is deployed at the end of the learning process, to be safe. This is typically the case when learning is performed in simulation, but the final controller has to be deployed in the real world. The main challenges there are in transferring safety properties from simulation to reality~\citep[\eg][]{tan2018sim}. In other cases, learning must be performed, or at least finalized, on the actual system, because no reliable simulator is available~\citep[\eg][]{peters2008reinforcement}. In such a scenario, safety must be enforced for the whole duration of the learning process. This poses a further challenge, as the agent must necessarily go through a sequence of sub-optimal behaviors before learning its final policy. The problem of learning \textit{while} containing the damage is also known as \textit{safe exploration}~\citep{amodei2016concrete} and will be the focus of this work.\footnote{We only use "safe exploration" in the general sense of~\citep{amodei2016concrete}. Indeed, in this work, we are not concerned with how exploration should be performed to maximize efficiency, but only with ensuring safety in a context where some form of exploration is necessary. 
	Exploration in RL is a very profound problem with a vast research tradition~\citep{flp2019alttutorial}.	
	The problem of efficient exploration under performance-improvement constraints has been studied in the multi-armed-bandit literature under the name of \emph{conservative bandits}~\citep{Wu2016conservative,Kazerouni2017conservativelinear,garcelon2020improved}, with recent extensions to finite MDPs~\citep{garcelon2020conservative}. 
	In RL, safe exploration is mainly concerned with avoiding unsafe states during the learning process~\citep{hans2008safe,pecka2014safe,dalal2018safe,turchetta2016safe,berkenkamp2019safe}.
}

\citet{garcia2015comprehensive}~provide a comprehensive survey on safe RL, where the existing approaches are organized into two main families: methods that modify the exploration process directly in order to explicitly avoid dangerous actions~\citep[\eg][]{gehring2013smart}, and methods that constrain exploration in a more indirect way by modifying the reward optimization process. The former typically require some sort of external knowledge, such as human demonstrations or advice~\citep[\eg][]{abbeel2010autonomous,clouse1992teaching}. In this work, we only assume online access to a sufficiently informative reward signal and prior knowledge of some worst-case constants that are easy to obtain. 
Optimization-based methods (those belonging to the second class) are more suited for this scenario. A particular kind, identified by Garc{\'i}a and Fern{\'a}ndez as \textit{constrained criteria}~\citep{moldovan2012safe,dicastro2012policy,kadota2006discounted}, enforces safety by introducing constraints in the optimization problem, \ie reward maximization.\footnote{Notably, the approach proposed by~\citet{Chow2018lyapunov} lays between the two classes. It relies on the framework of constrained MDPs to guarantee the safety of a behavior policy during training via a set of local, linear constraints defined using an external cost signal. Similar techniques have been used by~\citet{Berkenkamp2017safembrl} to guarantee the ability to re-enter a ``safe region'' during exploration.} 

A typical constraint is that the agent's performance, \ie the sum of rewards, must never be less than a user-specified threshold~\citep{geibel2005risk, thomas2015high}, which may be the average performance of a trusted \emph{baseline policy}. 
Under the assumption that the reward signal also encodes danger, low performances can be matched with dangerous behaviors, so that the performance threshold works as a safety threshold. This falls into the general framework of \emph{Seldonian machine learning} introduced by~\citet{thomas2019preventing}.

If we only cared about the safety of the final controller, the traditional RL objective --- maximizing cumulated reward --- would be enough. However, most RL algorithms are known to yield oscillating performances \textit{during} the learning phase. Regardless of the final solution, the intermediate ones may violate the threshold, hence yield unsafe behavior. This problem is known as \textit{policy oscillation}~\citep{bertsekas2011approximate,wagner2011reinterpretation}.

A similar constraint, which confronts the policy oscillation problem even more directly, is \textit{Monotonic Improvment}~\citep[MI,][]{kakade2002approximately,pirotta2013safe}, and is the one adopted in this work. The requirement is that each new policy implemented by the agent during the learning process does not perform worse than the previous one. In this way, if the initial policy is safe, so will be all the subsequent ones.

The way safety constraints such as MI can be imposed on the optimization process depends, of course, on what kind of policies are considered as candidates and on how the optimization itself is performed. These two aspects are often tied and will depend on the specific kind of RL algorithm that is employed. \textit{Policy Search} or \emph{Optimization}~\citep[PO,][]{deisenroth2013survey} is a family of RL algorithms where the class of candidate policies is fixed in advance and a direct search for the best one within the class is performed. This makes PO algorithms radically different from value-based algorithms such as Deep Q-Networks~\citep{mnih2015human}, where the optimal policy is a byproduct of a learned value function. Although value-based methods gained great popularity from their successes in games, PO algorithms are better suited for real-world tasks, especially the ones involving cyber-physical systems. The main reasons are the ability of PO methods to deal with high-dimensional continuous state and action spaces, convergence guarantees~\citep{sutton2000policy}, robustness to sensor noise, and the superior control on the set of feasible policies. The latter allows introducing domain knowledge into the optimization process, possibly including some safety constraints.

In this work, we focus on Policy Gradient methods~\citep[PG,][]{sutton2000policy,peters2008reinforcement}, where the set of candidate policies is a class of parametric distributions and the optimization is performed via stochastic gradient ascent on the performance objective as a function of the policy parameters. In particular, we analyze the prototypical PG algorithm, \textsc{REINFORCE}~\citep{williams1992simple} and see how the MI constraints can be imposed by adaptively selecting its meta-parameters during the learning process. 
To achieve this, we study in more depth the stochastic gradient-based optimization process that is at the core of all PG methods~\citep{robbins1951stochastic}. In particular, we identify a general family of parametric policies that makes the optimization objective Lipschitz-smooth~\citep{nesterov2013introductory} and allows easy upper-bounding of the related smoothness constant. This family, referred to as \textit{smoothing policies}, includes commonly used policy classes from the PG literature, namely Gaussian and Softmax policies.
Using known properties of Lipschitz-smooth functions, we then provide lower bounds on the performance improvement produced by gradient-based updates, as a function of tunable meta-parameters.
This, in turn, allows identifying those meta-parameter schedules that guarantee MI with high probability. 
In previous work, a similar result was achieved only for Gaussian policies~\citep{pirotta2013adaptive,papini2017adaptive}.\footnote{See Section~\ref{sec:related} for a discussion of related approaches, like the popular TRPO~\citep{schulman2015trust}.}

The meta-parameters studied here are the \textit{step size} of the policy updates, or learning rate, and the \textit{batch size} of gradient estimates, \ie the number of trials that are performed within a single policy update. These meta-parameters, already present in the original \textsc{REINFORCE} algorithm, are typically selected by hand and fixed for the whole learning process~\citep{duan2018benchmarking}. Besides guaranteeing monotonic improvement, our proposed method removes the burden of selecting these meta-parameters. This safe, automatic selection within the \textsc{REINFORCE} algorithmic framework yields SPG, our Safe Policy Gradient algorithm.

The paper is organized as follows: in Section~\ref{sec:pre} we introduce the necessary background on Markov decision processes, policy optimization, and smooth functions. In Section~\ref{sec:smoothpg}, we introduce smoothing policies and show the useful properties they induce on the policy optimization problem, most importantly a lower bound on the performance improvement yielded by an arbitrary policy parameter update (Theorem~\ref{th:main}). In Section~\ref{sec:metaparams}, we exploit these properties to select the step size of \textsc{REINFORCE} in a way that guarantees MI with high probability when the batch size is fixed, then we achieve similar results with an adaptive batch size. In Section~\ref{sec:algo}, we design a monotonically improving policy gradient algorithm with adaptive batch size, called Safe Policy Gradient (SPG), and show how the latter can also be adapted to weaker improvement constraints. In Section~\ref{sec:related}, we offer a detailed comparison of our contributions with the most closely related literature. In Section~\ref{sec:exp} we empirically evaluate SPG on simulated control tasks. Finally, we discuss the limitations of our approach and propose directions for future work in Section~\ref{sec:end}.

\section{Preliminaries}\label{sec:pre}
In this section, we revise continuous Markov Decision Processes~\citep[MDPs,][]{puterman2014markov}, actor-only Policy Gradient algorithms~\citep[PG, ][]{deisenroth2013survey}, and some general properties of smooth functions.

\subsection{Markov Decision Processes}\label{ssec:mpd}
A Markov Decision Process~\citep[MDP,][]{puterman2014markov} is a tuple ${\mathcal{M}=\langle\Sspace,\Aspace,\Tran,\Rew,\gamma,\Init\rangle}$, comprised of a measurable state space
${\Sspace}$, a measurable action space ${\Aspace}$, a Markovian transition kernel ${\Tran:\Sspace\times\Aspace\to\Delta_\Sspace}$, where $\Delta_\Sspace$ denotes the set of probability distributions over $\Sspace$, a reward function ${\Rew:\Sspace\times\Aspace\to\Reals}$, a discount factor~$\gamma\in(0,1)$ and an initial-state distribution $\Init\in\Delta_\Sspace$.
We only consider bounded-reward MDPs, and denote with $\Rmax\ge\sup_{s\in\Sspace,a\in\Aspace}\vert \Rew(s,a)\vert $ (a known upper bound on) the maximum absolute reward. This is the only prior knowledge we have on the task.
The MDP is used to model the interaction of a rational agent with the environment.
We model the agent's behavior with a policy $\pi:\Sspace\to\Delta_\Aspace$, a stochastic mapping from states to actions. 
The initial state is drawn as $s_0\sim\mu$. For each time step $t=0,1,\dots$, the agent draws   an action $a_t\sim\pi(\cdot\vert s_t)$, conditional on the current state $s_t$. Then, the agent obtains a reward $r_{t+1}=\Rew(s_t,a_t)$ and the state of the environment transitions to $s_{t+1}\sim\Tran(\cdot\vert s_t, a_t)$. The goal of the agent is to maximize the expected sum of discounted rewards, or \textit{performance measure}:
\begin{align}
J(\pi) \coloneqq \EVV{\sum_{t=0}^\infty\gamma^t r_{t+1} \vert  s_0\sim\mu, a_t\sim\pi(\cdot\vert s_t), s_{t+1}\sim p(\cdot\vert s_t,a_t)}.
\end{align}
We focus on continuous MDPs, where states and actions are real vectors: $\Sspace\subseteq\Reals^{d_{\Sspace}}$ and $\Aspace\subseteq\Reals^{d_{\Aspace}}$. However, all the results naturally extend to the discrete case by replacing integrals with summations. See~\citep{puterman2014markov,bertsekas2004stochastic} on matters of measurability and integrability, which just require common technical assumptions. We slightly abuse notation by denoting probability measures (assumed to be absolutely continuous) and density functions with the same symbol.

Given an MDP, the purpose of RL is to find an optimal policy $\pi^*\in\arg\max_{\pi}J(\pi)$ without knowing the transition kernel $\Tran$ and the reward function $\Rew$ in advance, but only through interaction with the environment. To better characterize this optimization objective, it is convenient to introduce further quantities. We denote with $p_{\pi}$ the transition kernel of the Markov Process induced by policy $\pi$, \ie $\Tran_{\pi}(\cdot\vert s) \coloneqq \int_{\Aspace}\pi(a\vert s) \Tran(\cdot\vert s,a)\de a$.
The $t$-step transition kernel under policy $\pi$ is defined inductively as follows:
\begin{align}
&\Tran_{\pi}^{0}(\cdot\vert s) = \indi{s=s'},\nonumber \\
&\Tran_{\pi}^{1}(\cdot\vert s) \coloneqq \Tran_{\pi}(\cdot\vert s),\nonumber\\
&\Tran_{\pi}^{t+1}(\cdot\vert s) \coloneqq  \int_{\Sspace}\Tran_{\pi}^{t}(s'\vert s)\Tran_{\pi}(\cdot\vert s')\de s',\label{eq:pt}
\end{align}
for all $s\in\Sspace$ and $t\ge1$. The $t$-step transition kernel allows to define the following \emph{conditional state-occupancy} measure:
\begin{align}
\Stat_{s}^{\pi}(\cdot) = (1-\gamma)\sum_{t=0}^{\infty}\gamma^t\Tran_{\pi}^t(\cdot\vert s),\label{def:dmupi}
\end{align}
measuring the (discounted) probability of visiting a state starting from $s$ and following policy $\pi$.
The following property of $\Stat_{s}^\pi$ ---a variant of the \emph{generalized eigenfunction property} by~\citet[][Lemma 20]{ciosek2020expected}---will be useful (proof in Appendix~\ref{app:mdp}):
\begin{restatable}{proposition}{eigenfunction}\label{th:eigenfunction}
	Let $\pi$ be any policy and $f$ be any integrable function on $\Sspace$ satisfying the following recursive equation:
	\begin{equation*}
	f(s) = g(s) + \gamma\int_{\Sspace}\Tran_{\pi}(s'\vert s)f(s')\de s',
	\end{equation*}
	for all $s\in\Sspace$ and some integrable function $g$ on $\Sspace$. Then:
	\begin{equation*}
	f(s) = \frac{1}{1-\gamma}\int_{\Sspace}\Stat^{\pi}_{s}(s')g(s')\de s',
	\end{equation*}
	for all $s\in\Sspace$.
\end{restatable}

The state-value function $V^{\pi}(s)=\EVV[\pi]{\sum_{t=0}^{\infty}\Rew(S_t,A_t)\vert S_0=s}$ is the discounted sum of rewards obtained, in expectation, by following policy $\pi$ from state $s$, and satisfies Bellman's equation~\citep{puterman2014markov}:
\begin{align}\label{def:v}
V^{\pi}(s) = \EVV[a\sim\pi(\cdot\vert s)]{\Rew(s,a) + \gamma \EVV[s'\sim\Tran(\cdot\vert s,a)]{V^{\pi}(s')}},
\end{align}
Similarly, the action-value function:
\begin{align}\label{def:q}
Q^{\pi}(s,a) = \Rew(s,a) + \gamma\EVV[s'\sim\Tran(\cdot\vert s,a)]{V^{\pi}(s')},
\end{align} 
is the discounted sum of rewards obtained, in expectation, by taking action $a$ in state $s$ and following $\pi$ afterwards. 

The two value functions are closely related:
\begin{align}
&V^{\pi}(s) = \int_{\Aspace}\pi(a\vert s)Q^{\pi}(s,a)\de a,\label{eq:vtoq}\\
&Q^{\pi}(s,a) = \Rew(s,a) + \gamma\int_{\mathcal{S}}p(s'\vert s,a)V_{\pi}(s')\de s'.\label{eq:qtov}
\end{align}
For bounded-reward MDPs, the value functions are bounded for every policy~$\pi$:
\begin{align}
&\norm[\infty]{V^{\pi}} \leq \norm[\infty]{Q^{\pi}} \leq \frac{\Rmax}{1-\gamma},\label{eq:boundv}
\end{align}
where $\norm[\infty]{V^{\pi}} = \sup_{s\in\Sspace}\vert V^{\pi}(s)\vert $ and $\norm[\infty]{Q^{\pi}} = \sup_{s\in\Sspace,a\in\Aspace} \vert Q^{\pi}(s,a)\vert $.
Using the definition of state-value function we can rewrite the performance measure as follows:
\begin{align}\label{eq:defj}
J(\pi) = \int_{\Sspace}\Init(s)V^{\pi}(s)\de s = \frac{1}{1-\gamma}\int_{\Sspace}\Stat^\pi(s)\int_\Aspace\pi(a\vert s)\Rew(s,a)\de a\de s,
\end{align}
where: 
\begin{equation}\label{eq:dmupi}
\Stat^{\pi}(\cdot) = \int_{\Sspace}\Init(s)\Stat_{s}^{\pi}(\cdot) \de s,
\end{equation}
is the state-occupancy probability under the starting-state distribution $\mu$.

\subsection{Parametric policies}
In this work, we only consider parametric policies. Given a $d$-dimensional parameter vector $\vtheta\in\Theta\subseteq\Reals^d$, a parametric policy is a stochastic mapping from states to actions parametrized by $\vtheta$, denoted with $\pi_{\vtheta}$. The search for the optimal policy is thus limited to the policy class \linebreak$\Pi_{\Theta} = \left\{\pi_{\vtheta}\mid \vtheta \in \Theta\right\}$. This corresponds to finding an optimal parameter, \ie\linebreak $\vtheta^*\in\arg\max_{\vtheta\in\Theta}J(\pi_{\vtheta})$. For ease of notation, we often write $\vtheta$ in place of $\pi_{\vtheta}$ in function arguments and superscripts, \eg $J(\vtheta)$, $\Stat^{\vtheta}(s)$ and $V^{\vtheta}(s)$ in place of $J(\pi_{\vtheta})$, $\Stat^{\pi_{\vtheta}}$ and $V^{\pi_{\vtheta}}(s)$, respectively.\footnote{Note that $J:\Theta\to\Reals$, as a function of policy parameters, may have a different geometry than $J:\Pi_{\Theta}\to\Reals$, as a function of the policy. In particular, policy parametrization can be an additional source of non-convexity.}
We restrict our attention to policies that are twice differentiable \wrt $\vtheta$, for which the gradient $\nabla_{\vtheta}\pi_{\vtheta}(a\vert s)$ and the Hessian $\nabla_{\vtheta}^2\pi_{\vtheta}(a\vert s)$ are defined everywhere and finite.
For ease of notation, we omit the $\vtheta$ subscript in $\nabla_{\vtheta}$ when clear from the context. Given any twice-differentiable scalar function $f:\Theta\to\Reals$, we denote with $D_if$ the $i$-th gradient component, \ie $\frac{\partial f}{\partial \theta_i}$, and with $D_{ij}f$ the Hessian element of coordinates $(i,j)$, \ie $\frac{\partial^2f}{\partial\theta_i\partial\theta_j}$. We also write $\nabla f(\vtheta)$ to denote $\left.\nabla_{\widetilde{\vtheta}} f(\widetilde{\vtheta})\right\vert_{\widetilde{\vtheta} = \vtheta}$ when this does not introduce any ambiguity.

The Policy Gradient Theorem~\citep{sutton2000policy,konda1999actor} allows us to characterize the gradient of the performance measure $J(\vtheta)$ as an expectation over states and actions visited under $\pi_{\vtheta}$:\footnote{As observed by~\cite{nota2020policy}, it is important that the state-occupancy measure is discounted as in~\eqref{def:dmupi} for the Policy Gradient Theorem to hold. An intuitive way to see the discounted occupancy $\rho^\pi(s)$ is as the probability of visiting state $s$ in an indefinite-horizon undiscounted MDP that is reset to the initial state distribution with probability $1-\gamma$ at each step.}
\begin{align}\label{eq:thepgt}
\nabla J(\vtheta) = \frac{1}{1-\gamma}\int_{\Sspace}\Stat^{\vtheta}(s)\int_{\Aspace}\pi_{\vtheta}(a\vert s)\nabla\log\pi_{\vtheta}(a\vert s)Q^{\vtheta}(s,a)\de a \de s.
\end{align}
The gradient of the log-likelihood $\nabla\log\pi_{\vtheta}(\cdot\vert s)$ is called \textit{score function}, while the Hessian of the log-likelihood $\Hess \log\pi_{\vtheta}(\cdot\vert s)$ is sometimes called \textit{observed information}.

\subsection{Actor-only policy gradient}
In practice, we always consider finite episodes of length $T$. We call this the \textit{effective horizon} of the MDP, chosen to be sufficiently large so that the problem does not lose generality.\footnote{We consider infinite-horizon discounted MDPs in our theoretical analysis, but consider a finite horizon when introducing specific policy gradient estimators. This mismatch is justified by the following result: when the reward is uniformly bounded by $\Rmax$, by setting $T = O \left( \log(\Rmax/\epsilon) / (1-\gamma) \right)$, the discounted truncated sum of rewards is $\epsilon$-close to the infinite sum~\citep[see, \eg][Sec. 2.3.3]{kakade2003sample}. See Appendix~\ref{app:geom} for a way to remove this bias by randomizing the horizon.}
We denote with ${\tau\coloneqq(s_0,a_0,s_1,a_1,\dots,s_{T-1},a_{T-1})}$ a \textit{trajectory}, \ie a sequence of states and actions of length $T$ such that $s_0\sim\Init$, $a_t\sim\pi(\cdot\vert s_t)$, $s_{t}\sim \Tran(\cdot\vert s_{t-1}, a_{t-1})$ for $t=0,\dots,T-1$ and some policy $\pi$. In this context, the performance measure of a parametric policy $\pi_{\vtheta}$ can be defined as:
\begin{align}
J(\vtheta) = \EVV[\tau\sim p_{\vtheta}]{\sum_{t=0}^{T-1}\gamma^t\Rew(s_t,a_t)},
\end{align}
where $p_{\vtheta}(\tau)$ is the probability density of the trajectory $\tau$ that can be generated by following policy $\pi_{\vtheta}$, \ie $p_{\vtheta}(\tau)=\mu(s_0)\pi_{\vtheta}(a_0\vert s_0)p(s_1\vert s_0,a_0)\dots \pi_{\vtheta}(a_{T-1}\vert s_{T-1})$. Let $\Dataset\sim p_{\vtheta}$ be a batch $\{\tau_1,\tau_2,\dots,\tau_N\}$ of $N$ trajectories generated with $\pi_{\vtheta}$, \ie $\tau_i\sim p_{\vtheta}$ \iid for $i=1,\dots,N$. Let $\hnabla J(\vtheta{;}\Dataset)$ be an estimate of the policy gradient $\nabla J(\vtheta)$ based on $\Dataset$.
Such an estimate can be used to perform stochastic gradient ascent on the performance objective~$J(\vtheta)$:
\begin{align}\label{eq:spg}
\vtheta' \gets \vtheta + \alpha\hnabla J(\vtheta{;}\Dataset),
\end{align}
where $\alpha\geq0$ is a \textit{step size} and $N=\vert\Dataset\vert$ is called \textit{batch size}. 
This yields an \textit{Actor-only Policy Gradient} method, summarized in Algorithm \ref{alg:pg}.

\begin{algorithm}[t]
	\caption{Actor-only policy gradient}
	\label{alg:pg}
	\begin{algorithmic}[1] 
		\State \textbf{Input:} initial policy parameters $\vtheta_0$, step size $\alpha$, batch size $N$, number of iterations $K$
		\For{$k=0,\dots,K-1$}
		\State Collect $N$ trajectories with $\vtheta_k$ to obtain dataset $\Dataset_k$
		\State Compute policy gradient estimate $\hnabla J(\vtheta_k{;} \Dataset_k)$
		\State Update policy parameters as $\vtheta_{k+1} \gets \vtheta_k + \alpha\hnabla J(\vtheta_k{;}\Dataset_k)$
		\EndFor
	\end{algorithmic}
\end{algorithm}

Under mild conditions, this algorithm is guaranteed to converge to a local optimum~\citep{sutton2000policy}. This is reasonable since the objective $J(\vtheta)$ is non-convex in general.\footnote{Recent works show that policy gradient algorithms can converge to \emph{globally} optimal policies in some interesting special cases~\citep{bhandari2019global,zhang2020sample,agarwal2020optimality}.}
As for the gradient estimator, we can use \textsc{REINFORCE}~\citep{williams1992simple,glynn86stochastic}:\footnote{In the literature, the term \textsc{REINFORCE} is often used to denote actor-only policy gradient methods in general. In this paper, \textsc{REINFORCE} refers to the algorithm by~\citet{williams1992simple}, which also applies to more general stochastic optimization problems.}
\begin{align}\label{eq:reinforce}
\hnabla J(\vtheta{;}\Dataset) = \frac{1}{N}\sum_{i=1}^{N}\left(\sum_{t=0}^{T-1}\gamma^t\Rew(a_t^i,s_t^i) - b\right)\left(\sum_{t=0}^{T-1}\nabla\log\pi_{\vtheta}(a_t^i\vert s_t^i)\right),
\end{align}
or its refinement, \textsc{G(PO)MDP}~\citep{baxter2001infinite}, which typically suffers from less variance~\citep{peters2008reinforcement}:
\begin{align}\label{eq:gpomdp}
\hnabla J(\vtheta{;}\Dataset) = \frac{1}{N}\sum_{i=1}^{N}\sum_{t=0}^{T-1}\left[\left(\gamma^t\Rew(a_t^i,s_t^i) - b_t\right)\sum_{h=0}^{t}\nabla\log\pi_{\vtheta}(a_h^i\vert s_h^i)\right],
\end{align}
where the superscript on states and actions denotes the $i$-th trajectory of the dataset and $b$ is a (possibly time-dependent and vector-valued) control variate, or \textit{baseline}. Both estimators are unbiased for any action-independent baseline.\footnote{Also valid action-dependent baselines have been proposed. See~\citep{tucker2018mirage} for a discussion.}
\citet{peters2008reinforcement} prove that Algorithm \ref{alg:pg} with the \textsc{G(PO)MDP} estimator is equivalent to Monte-Carlo PGT~\citep[Policy Gradient Theorem,][]{sutton2000policy}, and provide variance-minimizing baselines for both \textsc{REINFORCE} and \textsc{G(PO)MDP}.

Algorithm~\ref{alg:pg} is called \textit{actor-only} to discriminate it from \textit{actor-critic} policy gradient  algorithms~\citep{konda1999actor}, where an approximate value function, or \textit{critic}, is employed in the gradient computation.
In this work, we will focus on actor-only algorithms, for which safety guarantees are more easily proven.\footnote{The distinction is not so sharp, as a critic can be seen as a baseline and vice-versa. We call \textit{critic} an \textit{explicit} value function estimate used in policy gradient estimation. 
} 
Generalizations of Algorithm \ref{alg:pg} include reducing the variance of gradient estimates through baselines and other stochastic-optimization techniques~\citep[\eg][]{papini2018stochastic,shen2019hessian,xu2020sample} using a vector step size~\citep{yu2006fast, papini2017adaptive}; making the step size \textit{adaptive}, \ie iteration and/or data-dependent~\citep{pirotta2013adaptive}; making the batch size $N$ also adaptive~\citep{papini2017adaptive}; applying a preconditioning matrix to the gradient, as in Natural Policy Gradient~\citep{kakade2002natural} and second-order methods~\citep{furmston2012unifying}.

\subsection{Smooth functions}\label{ssec:smoothfn}
In the following we denote with $\norm[p]{\vx}$ the $\ell_p$-norm of vector $\vx$, which is the Euclidean norm for $p=2$.
. For a matrix $A$, $\norm[p]{A}=\sup\{\norm[p]{Ax}:\norm[p]{x}=1\}$ denotes the induced norm, which is the spectral norm for $p=2$. When the $p$ subscript is omitted, we always mean $p=2$.

Let $g:\mathcal{X}\subseteq\Reals^d\to\Reals^n$ be a (non-convex) vector-valued function. We call $g$ \textit{Lipschitz continuous} if there exists $L>0$ such that, for every $\vx,\vx'\in\Xspace$:
\begin{align}
\norm[]{g(\vx')-g(\vx)} \leq L\norm[]{\vx'-\vx}.
\end{align} 
Let $f:\mathcal{X}\subseteq\Reals^d\to\Reals$ be a real-valued differentiable function. We call $f$ \textit{Lipschitz smooth} if its gradient is Lipschitz continuous, \ie there exists $L>0$ such that, for every $\vx,\vx'\in\Xspace$:
\begin{align}
\norm[]{\nabla f(\vx')- \nabla f(\vx)} \leq L\norm[]{\vx'-\vx}.
\end{align}
Whenever we want to specify the Lipschitz constant $L$ of the gradient, we call $f$ $L$-smooth.\footnote{The Lipschitz constant is usually defined as the \textit{smallest} constant satisfying the Lipschitz condition. In this paper, we accept \textit{any} constant for which the Lipschitz condition holds.}
We also call $L$ the \emph{smoothness constant} of $f$.
For a twice-differentiable function, the following holds:\footnote{The results from this section are well known in the optimization literature~\citep{nesterov2013introductory}. However, proofs of Lemma \ref{lem:mvi} and \ref{lem:smooth} are reported in Appendix~\ref{app:smooth} for the sake of completeness.}
\begin{restatable}{proposition}{mvi}\label{lem:mvi}
	Let $\Xspace$ be a convex subset of $\Reals^d$ and $f:\mathcal{X}\to\Reals$ be a twice-differentiable function. If the Hessian is uniformly bounded in spectral norm by $L>0$, \ie $\sup_{\vx\in\Xspace}\norm[2]{\Hess f(\vx)} \leq L$, then $f$ is $L$-smooth.
\end{restatable}
Lipschitz smooth functions admit a quadratic bound on the deviation from linear behavior:
\begin{restatable}[Quadratic Bound]{proposition}{smooth}\label{lem:smooth}
	Let $\Xspace$ be a convex subset of $\Reals^d$ and $f:\mathcal{X}\to\Reals$ be an $L$-smooth function. Then, for every $\vx,\vx'\in\Xspace$:
	\begin{align}
	\left\vert f(\vx ') - f(\vx ) - \left\langle \vx '- \vx , \nabla f(\vx )\right\rangle\right\vert \leq \frac{L}{2}\norm{\vx' - \vx }^2,
	\end{align}
	where $\langle\cdot,\cdot\rangle$ denotes the dot product.
\end{restatable}
This bound is often useful for optimization purposes~\citep{nesterov2013introductory}.

\section{Smooth Policy Gradient}\label{sec:smoothpg}
In this section, we provide lower bounds on performance improvement based on general assumptions on the policy class.

\subsection{Smoothing policies}\label{ssec:smoothpol}
We introduce a family of parametric stochastic policies having properties that we deem desirable for policy-gradient learning.
We call them \emph{smoothing}, as they are characterized by the smoothness of the performance measure:
\begin{restatable}{definition}{smoothpol}\label{def:smooth}
	Let $\Pi_{\Theta}=\{\pi_{\vtheta}\mid \vtheta\in\Theta\}$ be a class of twice-differentiable parametric stochastic policies, where $\Theta\subset\Reals^d$ is convex. We call it \emph{smoothing} if there exist non-negative constants $\sm,\smm,\smmm$ such that, for every state and in expectation over actions, the Euclidean norm of the score function:
	\begin{align}\label{eq:sm1}
	& \sup_{s\in\Sspace}\mathbb{E}_{a\sim\pi_{\vtheta}(\cdot\vert s)} \Big[ \norm{\nabla\log\pi_{\vtheta}(a\vert s)} \Big] \leq \sm, 
	\end{align}
	the squared Euclidean norm of the score function:
	\begin{align}\label{eq:sm2}
	& \sup_{s\in\Sspace}\mathbb{E}_{a\sim\pi_{\vtheta}(\cdot\vert s)} \Big[ \norm{\nabla\log\pi_{\vtheta}(a\vert s)}^2 \Big] \leq \smm, 
	\end{align}
	and the spectral norm of the observed information:
	\begin{align}\label{eq:sm3}
	& \sup_{s\in\Sspace}\mathbb{E}_{a\sim\pi_{\vtheta}(\cdot\vert s)} \Big[\norm{\Hess \log\pi_{\vtheta}(a\vert s)}\Big] \leq \smmm, 
	\end{align}
	are upper-bounded.
\end{restatable}
Note that the definition requires that the bounding constants $\sm,\smm, \smmm$ be independent of the policy parameters and the state. For this reason, the existence of such constants depends on the policy parameterization.\footnote{Notice that, by Jensen's inequality, one can always remove the first requirement~\eqref{eq:sm1} by letting $\sm=\sqrt{\smm}$, as observed by~\cite{yuan2021general}. However, a smaller value of $\sm$ can sometimes be obtained. See Lemma~\ref{lem:gauss} for an example.}
We call a policy class $(\sm,\smm,\smmm)$-smoothing when we want to specify the bounding constants.
In Appendix \ref{sec:polclass}, we show that some of the most commonly used policies, such as the Gaussian policy for continuous actions and the Softmax policy for discrete actions, are smoothing. The smoothing constants for these classes are reported in Table~\ref{tab:smoothing}.
In the following sections, we will exploit the smoothness of the performance measure induced by smoothing policies to develop a monotonically improving policy gradient algorithm. However, smoothing policies have other interesting properties. For instance, variance upper bounds for REINFORCE/G(PO)MDP with Gaussian policies~\citep{zhao2011analysis,pirotta2013adaptive} can be generalized to smoothing policies (see Appendix~\ref{sec:var} for details). Other nice properties of smoothing policies, such as Lipschitzness of the performance measure, are discussed in~\cite[][Lemma D.1]{yuan2021general}.

\begin{table}[t]
	\caption{Smoothing constants $\sm,\smm,\smmm$ and smoothness constant $L$ for Gaussian and Softmax policies, where $\phimax$ is an upper bound on the Euclidean norm of the feature function, $\Rmax$ is the maximum absolute-value reward, $\gamma$ is the discount factor, $\sigma$ is the standard deviation of the Gaussian policy and $\tau$ is the temperature of the Softmax policy. We also report the improved smoothness constant by~\cite{yuan2021general} as $L^\star$.}
	\centering
	\begin{tabular}{lcc}
		\toprule
		& \textbf{Gaussian} & \textbf{Softmax} \\
		\midrule
		$\sm$ & $\frac{2\phimax}{\sqrt{2\pi}\sigma}$ & $\frac{2\phimax}{\tau}$\\
		$\smm$ & $\frac{\phimax^2}{\sigma^2}$ & $\frac{4\phimax^2}{\tau^2}$\\
		$\smmm$ & $\frac{\phimax^2}{\sigma^2}$ & $\frac{2\phimax^2}{\tau^2}$\\
		$L$ & $\frac{2\phimax^2\Rmax}{\sigma^2(1-\gamma)^2}\left(1+\frac{2\gamma}{\pi(1-\gamma)}\right)$ & $\frac{2\phimax^2\Rmax}{\tau^2(1-\gamma)^2}\left(3+\frac{4\gamma}{1-\gamma}\right)$\\
		$L^\star$ & $\frac{2\phimax^2\Rmax}{\sigma^2(1-\gamma)^2}$ & $\frac{6\phimax^2\Rmax}{\tau^2(1-\gamma)^2}$\\
		\bottomrule
	\end{tabular}
	\label{tab:smoothing}
\end{table}

\subsection{Policy Hessian}\label{ssec:polhess}
We now show that the Hessian of the performance measure $\Hess J(\vtheta)$ for a smoothing policy has bounded spectral norm.
We start by writing the policy Hessian for a general parametric policy as follows. The result is well known~\citep{kakade2001optimizing}, but we report a proof in Appendix~\ref{app:pol} for completeness. Also, note that our smoothing-policy assumption is weaker than the typical one (uniformly bounded policy derivatives). See Appendix~\ref{sec:leibniz} for details.
\begin{restatable}{proposition}{polhess}\label{lem:polhess}
	Let $\pi_{\vtheta}$ be a smoothing policy. The Hessian of the performance measure is:
	\begin{align*}
	\Hess J(\vtheta) 
	&= \frac{1}{1-\gamma}\EV_{\substack{s\sim \Stat^{\vtheta}\\a\sim\pi_{\vtheta}(\cdot\vert s)}}
	\Big[ 
	\nabla\log\pi_{\vtheta}(a\vert s)\nabla^{\top} Q^{\vtheta}(s,a)
	+\nabla Q^{\vtheta}(s,a)\nabla^{\top}\log\pi_{\vtheta}(a\vert s)\\
	&\qquad+\left(\nabla\log\pi_{\vtheta}(a\vert s)\nabla^{\top}\log\pi_{\vtheta}(a\vert s)
	+\Hess \log\pi_{\vtheta}(a\vert s)  \right)Q^{\vtheta}(s,a) \Big].
	\end{align*}
\end{restatable}
For smoothing policies, we can bound the policy Hessian in terms of the constants from Definition~\ref{def:smooth}:
\begin{restatable}[]{lemma}{polhessbound}\label{lem:polhessbound}
	Given a $(\sm,\smm,\smmm)$-smoothing policy $\pi_{\vtheta}$, the spectral norm of the policy Hessian can be upper-bounded as follows:
	\begin{align*}
	\norm{\Hess J(\vtheta)} \leq \frac{\Rmax}{(1-\gamma)^2}\left(
	\frac{2\gamma\sm^2}{1-\gamma} + \smm + \smmm
	\right).
	\end{align*}
\end{restatable}
\begin{proof}
	By the Policy Gradient Theorem~\citep[see the proof of Theorem 1 by][]{sutton2000policy}:
	\begin{align}
	\nabla V^{\vtheta}(s)
	&=  \frac{1}{1-\gamma}\int_{\Sspace}\Stat_{s}^{\vtheta}(s')\int_{\Aspace}\pi_{\vtheta}(a\vert s')\nabla\log\pi_{\vtheta}(a\vert s')Q^{\vtheta}(s,a) \de a\de s'.\label{eq:gradv}
	\end{align}
	Using~\eqref{eq:gradv}, we bound the gradient of the value function in Euclidean norm:
	\begin{align}
	\norm{\nabla V^{\vtheta}(s)} 
	&\leq  
	\frac{1}{1-\gamma}\EVV[\substack{s'\sim \Stat_{s}^{\vtheta}\\a\sim\pi_{\vtheta}(\cdot\vert s')}]{\norm{\nabla\log\pi_{\vtheta}(a\vert s') Q^{\vtheta}(s',a)}} \nonumber\\
	&\leq \frac{\Rmax}{(1-\gamma)^2}\EVV[\substack{s'\sim \Stat^{\vtheta}_s\\a\sim\pi_{\vtheta}(\cdot\vert s')}]{\norm{\nabla\log\pi_{\vtheta}(a\vert s')}}\label{p:hessbound.13}\\
	&\leq \frac{\Rmax}{(1-\gamma)^2}\sup_{s'\in\Sspace}\EVV[a\sim\pi_{\vtheta}(\cdot\vert s')]{\norm{\nabla\log\pi_{\vtheta}(a\vert s')}}\nonumber\\
	&\leq \frac{\sm\Rmax}{(1-\gamma)^2},\label{p:hessbound.17}
	\end{align}
	where (\ref{p:hessbound.13}) is from the Cauchy-Schwarz inequality and (\ref{eq:boundv}), and (\ref{p:hessbound.17}) is from the smoothing-policy assumption. Next, we bound the gradient of the action-value function. From~\eqref{eq:qtov}:
	\begin{align}
	\norm{\nabla Q^{\vtheta}(s,a)} 
	&=\norm{\nabla \left(\Rew(s,a)+\gamma\EVV[s'\sim p(\cdot\vert s,a)]{V^{\vtheta}(s')}\right)} \\
	&=\gamma\norm{\EVV[s'\sim p(\cdot\vert s,a)]{\nabla V^{\vtheta}(s')}}\label{eq:xchange} \\
	&\le \gamma \EVV[s'\sim p(\cdot\vert s,a)]{\norm{\nabla V^{\vtheta}(s)}}
	\leq \frac{\gamma\sm\Rmax}{(1-\gamma)^2},
	\label{p:hessbound.4}
	\end{align}
	where the interchange of gradient and expectation in~\eqref{eq:xchange} is justified by the smoothing-policy assumption (see Appendix~\ref{sec:leibniz} for details) and (\ref{p:hessbound.4}) is from (\ref{p:hessbound.17}).
	Finally, from Proposition \ref{lem:polhess}:
	\begin{align}
	(1-\gamma)\norm{\Hess J(\vtheta)} 
	&\leq 
	\EVV[\substack{s\sim \Stat^{\vtheta}\\a\sim\pi_{\vtheta}(\cdot\vert s)}]{
		\norm{\nabla\log\pi_{\vtheta}(a\vert s)\nabla^{\top} Q^{\vtheta}(s,a)}}
	\nonumber\\&\qquad
	+\EVV[\substack{s\sim \Stat^{\vtheta}\\a\sim\pi_{\vtheta}(\cdot\vert s)}]{\norm{\nabla Q^{\vtheta}(s,a)\nabla^{\top}\log\pi_{\vtheta}(a\vert s)}}\nonumber\\
	&\qquad+\EVV[\substack{s\sim \Stat^{\vtheta}\\a\sim\pi_{\vtheta}(\cdot\vert s)}]{\norm{\nabla\log\pi_{\vtheta}(a\vert s)\nabla^{\top}\log\pi_{\vtheta}(a\vert s)
			Q^{\vtheta}(s,a)}} \nonumber\\
	&\qquad+\EVV[\substack{s\sim \Stat^{\vtheta}\\a\sim\pi_{\vtheta}(\cdot\vert s)}]{\norm{\Hess \log\pi_{\vtheta}(a\vert s)Q^{\vtheta}(s,a)}} \label{p:hessbound.14}\\
	&\leq
	2\EVV[\substack{s\sim \Stat^{\vtheta}\\a\sim\pi_{\vtheta}(\cdot\vert s)}]{
		\norm{\nabla\log\pi_{\vtheta}(a\vert s)}\norm{\nabla Q^{\vtheta}(s,a)}}
	\nonumber\\&\qquad
	+\EVV[\substack{s\sim \Stat^{\vtheta}\\a\sim\pi_{\vtheta}(\cdot\vert s)}]{\norm{\nabla\log\pi_{\vtheta}(a\vert s)}^2
		\left\vert Q^{\vtheta}(s,a)\right\vert } \nonumber\\
	&\qquad+\EVV[\substack{s\sim \Stat^{\vtheta}\\a\sim\pi_{\vtheta}(\cdot\vert s)}]{\norm{\Hess \log\pi_{\vtheta}(a\vert s)}\left\vert Q^{\vtheta}(s,a)\right\vert } \label{p:hessbound.15}\\
	&\leq
	\frac{2\gamma\sm\Rmax}{(1-\gamma)^2}\EVV[\substack{s\sim \Stat^{\vtheta}\\a\sim\pi_{\vtheta}(\cdot\vert s)}]{
		\norm{\nabla\log\pi_{\vtheta}(a\vert s)}}
	\nonumber\\&\qquad
	+\frac{\Rmax}{1-\gamma}\EVV[\substack{s\sim \Stat^{\vtheta}\\a\sim\pi_{\vtheta}(\cdot\vert s)}]{\norm{\nabla\log\pi_{\vtheta}(a\vert s)}^2} \nonumber\\
	&\qquad+\frac{\Rmax}{1-\gamma}\EVV[\substack{s\sim \Stat^{\vtheta}\\a\sim\pi_{\vtheta}(\cdot\vert s)}]{\norm{\Hess \log\pi_{\vtheta}(a\vert s)}} \label{p:hessbound.16}\\
	&\leq \frac{\Rmax}{(1-\gamma)}\left(
	\frac{2\gamma\sm^2}{1-\gamma} + \smm + \smmm
	\right),
	\end{align}
	where (\ref{p:hessbound.14}) is from Jensen inequality (all norms are convex) and the triangle inequality, (\ref{p:hessbound.15}) is from $\norm{\vx\vy^{\top}} = \norm{\vx}\norm{\vy}$ for any two vectors $\vx$ and $\vy$, (\ref{p:hessbound.16}) is from (\ref{eq:boundv}) and (\ref{p:hessbound.4}), and the last inequality is from the smoothing-policy assumption.
\end{proof}

\subsection{Smooth Performance}\label{ssec:perfbound}
For a smoothing policy, the performance measure $J(\vtheta)$ is Lipschitz smooth with a smoothness constant that only depends on the smoothing constants, the reward magnitude, and the discount factor. This result is of independent interest as it can be used to establish convergence rates for policy gradient algorithms~\citep{yuan2021general}.
\begin{restatable}{lemma}{smoothj}\label{lem:smoothj}
	Given a $(\sm,\smm,\smmm)$-smoothing policy class $\Pi_{\Theta}$, the performance measure $J(\vtheta)$ is $L$-smooth with the following smoothness constant:
	\begin{align}
	L = \frac{\Rmax}{(1-\gamma)^2}\left(
	\frac{2\gamma\sm^2}{1-\gamma} + \smm + \smmm
	\right).
	\end{align} 
\end{restatable}
\begin{proof}
	From Lemma \ref{lem:polhessbound}, $L$ is a bound on the spectral norm of the policy Hessian. From Lemma \ref{lem:mvi}, this is a valid Lipschitz constant for the policy gradient, hence the performance measure is $L$-smooth.
\end{proof}
The smoothness of the performance measure, in turn, yields the following property on the guaranteed performance improvement: 
\begin{restatable}[]{theorem}{mainth}\label{th:main}
	Let $\Pi_{\Theta}$ be a $(\sm,\smm,\smmm)$-smoothing policy class. For every $\vtheta, \vtheta'\in\Theta$:
	\begin{align*}
	J(\vtheta') - J(\vtheta)  \geq \left\langle\Delta\vtheta,\nabla J(\vtheta)\right\rangle - \frac{L}{2}\norm{\Delta\vtheta}^2,
	\end{align*}
	where $\Delta\vtheta = \vtheta'-\vtheta$ and $L = \frac{\Rmax}{(1-\gamma)^2}\left(
	\frac{2\gamma\sm^2}{1-\gamma} + \smm + \smmm
	\right)$.
\end{restatable}
\begin{proof}
	It suffices to apply Lemma \ref{lem:smooth} with the Lipschitz constant from Lemma \ref{lem:smoothj}.
\end{proof}
The smoothness constant $L$ for Gaussian and Softmax policies is reported in Table~\ref{tab:smoothing}.

In the following, we will exploit this property of smoothing policies to enforce safety guarantees on the policy updates performed by Algorithm~\ref{alg:pg}, \ie stochastic gradient ascent updates. However, Theorem~\ref{th:main} applies to any policy update $\Delta\vtheta\in\Reals^d$ as long as $\vtheta+\Delta\vtheta\in\Theta$.

Very recently,~\citet[][Lemma 4.4]{yuan2021general} provided an improved smoothness constant for smoothing policies:
\begin{equation}
	L^\star = \frac{\Rmax(\smm+\smmm)}{(1-\gamma)^2}.
\end{equation}
This is a significant step forward since it improves the dependence on the effective horizon by a $(1-\gamma)^{-1}$ factor. In Table~\ref{tab:smoothing} we report explicit expressions for $L^\star$ in the case of linear Gaussian and Softmax policies. We will use these superior smoothness constant in the numerical simulations of Section~\ref{sec:exp}.

\section{Optimal Safe Meta-Parameters}\label{sec:metaparams}
In this section, we provide  a step size for Algorithm~\ref{alg:pg} that maximizes a lower bound on the performance improvement for smoothing policies. This yields safety in the sense of Monotonic Improvement (MI), \ie non-negative performance improvements at each policy update:
\begin{equation}
	J(\vtheta_{k}) - J(\vtheta_{k+1}) \ge 0,
\end{equation}
at least with high probability.

In policy optimization, at each learning iteration $k$, we ideally want to find the policy update $\Delta\vtheta$ that maximizes the new performance $J(\vtheta_{k}+\Delta\vtheta)$, or equivalently:
\begin{align}\label{eq:maxdiff}
\max_{\Delta\vtheta} J(\vtheta_{k}+\Delta\vtheta) - J(\vtheta_k),
\end{align}
since $J(\vtheta_{k})$ is fixed. Unfortunately, the performance of the updated policy cannot be known in advance.\footnote{The performance of the updated policy could be estimated with off-policy evaluation techniques, but this would introduce an additional, non-negligible source of variance. The idea of using off-policy evaluation to select meta-parameters was explored by~\citet{paul2019fast}.} For this reason, we replace the optimization objective in (\ref{eq:maxdiff}) with a lower bound, \ie a \textit{guaranteed improvement}. 
In particular, taking Algorithm~\ref{alg:pg} as our starting point, we maximize the guaranteed improvement of a policy gradient update (line 5) by selecting optimal meta-parameters.
The solution of this meta-optimization problem provides a lower bound on the actual performance improvement. As long as this is always non-negative, MI is guaranteed.

\subsection{Adaptive Step Size -- Exact Framework}
To decouple the pure optimization aspects of this problem from gradient estimation issues, we first consider an \textit{exact} policy gradient update, \ie $\vtheta_{k+1}\gets\vtheta_k + \alpha\nabla J(\vtheta_k)$, where we assume to have a first-order oracle, \ie to be able to compute the exact policy gradient $\nabla J(\vtheta_k)$. This assumption is clearly not realistic, and will be removed in Section \ref{sec:adastepapp}. In this simplified framework, performance improvement can be guaranteed \textit{deterministically}. Furthermore, the only relevant meta-parameter is the step size $\alpha$ of the update.  
We first need a lower bound on the performance improvement $J(\vtheta_{k+1}) - J(\vtheta_k)$.
For a smoothing policy, we can use the following:
\begin{restatable}[]{theorem}{ltwobound}\label{th:ltwobound}
	Let $\Pi_{\Theta}$ be a $(\sm,\smm,\smmm)$-smoothing policy class. Let $\vtheta_k\in\Theta$ and $\vtheta_{k+1} = \vtheta_k + \alpha\nabla J(\vtheta_k)$, where $\alpha>0$. Provided $\vtheta_{k+1}\in\Theta$, the performance improvement of $\vtheta_{k+1}$ \wrt $\vtheta_k$ can be lower bounded as follows:
	\begin{align*}
	J(\vtheta_{k+1}) - J(\vtheta_k) \geq \alpha\norm[]{\nabla J(\vtheta_k)}^2
	- \alpha^2\frac{L}{2}\norm[]{\nabla J(\vtheta_k)}^2 \coloneqq B(\alpha{;}\vtheta_k), 
	\end{align*}
	where $L = \frac{\Rmax}{(1-\gamma)^2}\left(
	\frac{2\gamma\sm^2}{1-\gamma} + \smm + \smmm
	\right)$.
\end{restatable}
\begin{proof}
	This is a direct consequence of Theorem~\ref{th:main} with $\Delta\vtheta=\alpha\nabla J(\vtheta_k)$.
\end{proof}
This bound is in the typical form of performance improvement bounds~\citep[\eg][]{kakade2002approximately,pirotta2013adaptive,schulman2015trust,cohen2018diverse}: a positive term accounting for the anticipated advantage of $\vtheta_{k+1}$ over $\vtheta_k$, and a penalty term accounting for the mismatch between the two policies, which makes the anticipated advantage less reliable. In our case, the mismatch is measured by the curvature of the performance measure \wrt the policy parameters, via the smoothness constant $L$. 
This lower bound is quadratic in $\alpha$, hence we can easily find the optimal step size~$\alpha^*$.
\begin{restatable}[]{corollary}{safestep}\label{cor:safestep}
	Let $B(\alpha{;}\vtheta_k)$ be the guaranteed performance improvement of an exact policy gradient update, as defined in Theorem~\ref{th:ltwobound}.
	Under the same assumptions, $B(\alpha{;}\vtheta_k)$ is maximized by the constant step size $\alpha^*=\frac{1}{L}$, which guarantees the following non-negative performance improvement:
	\begin{align*}
	J(\vtheta_{k+1}) - J(\vtheta_k) \geq \frac{\norm{\nabla J(\vtheta_k)}^2}{2L}.
	\end{align*}
\end{restatable}
\begin{proof}
	We just maximize $B(\alpha{;} \vtheta_k)$ as a (quadratic) function of $\alpha$. The global optimum $B(\alpha^*{;}\vtheta_k) = \frac{\norm{\nabla J(\vtheta_k)}^2}{2L}$ is attained by $\alpha^*=\frac{1}{L}$. The improvement guarantee follows from Theorem~\ref{th:ltwobound}.
\end{proof}

\subsection{Adaptive Step Size -- Approximate Framework}\label{sec:adastepapp}
In practice, we cannot compute the exact gradient $\nabla J(\vtheta_k)$, but only an estimate~$\hnabla  J(\vtheta{;}\Dataset)$ obtained from a finite dataset $\Dataset$ of trajectories. In this section, $N$ denotes the \emph{fixed} size of $\Dataset$.
To find the optimal step size, we just need to adapt the performance-improvement lower bound of Theorem~\ref{th:ltwobound} to stochastic-gradient updates. Since sample trajectories are involved, this new lower bound will only hold with high probability.
To establish statistical guarantees, we make the following assumption on how the (unbiased) gradient estimate concentrates around its expected value:
\begin{assumption}\label{asm:conc}
	Fixed a parameter $\vtheta\in\Theta$, a batch size $N\in\Naturals$ and a failure probability $\delta\in(0,1)$, with probability at least $1-\delta$: 
	\begin{equation*}
	\norm{\hnabla J(\vtheta{;}\Dataset)-\nabla J(\vtheta)} \le \frac{\epsilon(\delta)}{\sqrt{N}},
	\end{equation*}
	where $\vert\Dataset\vert$ is a dataset of $N$ \iid trajectories collected with $\pi_{\vtheta}$ and $\epsilon:(0,1)\to\Reals$ is a known function.
\end{assumption}
We will discuss how this assumption is satisfied in cases of interest in Section~\ref{sec:algo} and Appendix~\ref{app:bounded}.
%
%
%
%
Under the above assumption, we can adapt Theorem~\ref{th:ltwobound} to the stochastic-gradient case as follows:
\begin{restatable}[]{theorem}{storacle}\label{th:storacle}
	Let $\Pi_{\Theta}$ be a $(\sm,\smm,\smmm)$-smoothing policy class.
	Let $\vtheta_{k}\in\Theta\subseteq\Reals^d$ 
	and $\vtheta_{k+1} = \vtheta_{k} + \alpha\hnabla  J(\vtheta_{k}{;} \Dataset_k)$, 
	where $\alpha\geq0$, $N=\vert\Dataset_k\vert\geq 1$. 
	Under Assumption~\ref{asm:conc}, provided $\vtheta_{k+1}\in\Theta$, the performance improvement of $\vtheta_{k+1}$ \wrt $\vtheta_{k}$ can be lower bounded, with probability at least $1-\delta_k$, as follows:
	\begin{equation}
	\label{eq:appbound}
	\begin{aligned}
	J(\vtheta_{k+1}) - J(\vtheta_{k}) &\geq \alpha\left(\norm{\hnabla  J(\vtheta_{k}{;} \Dataset_k)}- \frac{\epsilon(\delta_k)}{\sqrt{N}}\right)\nonumber\\&\qquad\times\max\left\{
	\norm{\hnabla  J(\vtheta_{k}{;} \Dataset_k)}, \frac{\norm{\hnabla  J(\vtheta_{k}{;} \Dataset_k)} + \frac{\epsilon(\delta_k)}{\sqrt{N}}}{2}
	\right\} \\
	&\qquad- \frac{\alpha^2L}{2}
	\norm{\hnabla  J(\vtheta_{k}{;} \Dataset_k)}^2\coloneqq \wt{B}_k(\alpha{;}N),
	\end{aligned}
	\end{equation}
	where $L = \frac{\Rmax}{(1-\gamma)^2}\left(
	\frac{2\gamma\sm^2}{1-\gamma} + \smm + \smmm
	\right)$.
\end{restatable}
\begin{proof}
	Consider the good event $E_k=\left\{\norm{\hnabla J(\vtheta{;}\Dataset)-\nabla J(\vtheta)} \le {\epsilon(\delta_k)}/{\sqrt{N}}\right\}$. By Assumption~\ref{asm:conc}, $E_k$ holds with probability at least $1-\delta_k$. For the rest of the proof, we will assume $E_k$ holds.
	
	Let $\epsilon_k\coloneqq\epsilon(\delta_k)/\sqrt{N}$ for short. Under $E_k$, by the triangular inequality:
	\begin{align}
	\norm{\nabla J(\vtheta_{k})} &\geq \norm{\hnabla  J(\vtheta_{k}{;} \Dataset_k)} - \norm{\nabla J(\vtheta_{k}) - \hnabla  J(\vtheta_{k}{;} \Dataset_k)} \nonumber\\
	&\geq \norm{\hnabla  J(\vtheta_{k}{;} \Dataset_k)} - \epsilon_k,\label{eq:16.5}
	\end{align}
	thus:
	\begin{align}
	\norm{\nabla J(\vtheta_{k})}^2 \geq 	\max\left\{\norm{\hnabla  J(\vtheta_{k}{;} \Dataset_k)} - \epsilon_k, 0\right\}^2.\label{eq:16.8}
	\end{align}
	Then, by the polarization identity:
	\begin{align}
	&\left\langle\hnabla J(\vtheta_{k}{;} \Dataset_k),\nabla J(\vtheta_{k})\right\rangle = \frac{1}{2}\left(\norm{\hnabla J(\vtheta_{k}{;} \Dataset_k)}^2 + \norm{\nabla J(\vtheta_{k})}^2 
	\nonumber\right.\\&\qquad\qquad\qquad\qquad\qquad\qquad
	\left.- \norm{\nabla J(\vtheta_{k}) - \hnabla  J(\vtheta_{k}{;} \Dataset_k)}^2\right)\nonumber\\
	&\qquad\qquad\geq \frac{1}{2}\left(\norm{\hnabla J(\vtheta_{k}{;} \Dataset_k)}^2 + \max\left\{\norm{\hnabla  J(\vtheta_{k}{;} \Dataset_k)} - \epsilon_k, 0\right\}^2 - \epsilon_k^2 \right),\nonumber
	\end{align}
	where the latter inequality is from (\ref{eq:16.8}).
	We first consider the case in\linebreak which $\norm{\hnabla  J(\vtheta_{k}{;} \Dataset_k)} > \epsilon_k$:
	\begin{align}
	\left\langle\hnabla J(\vtheta_{k}{;} \Dataset_k),\nabla J(\vtheta_{k})\right\rangle
	&\geq \frac{1}{2}\left(\norm{\hnabla J(\vtheta_{k}{;} \Dataset_k)}^2 + \left(\norm{\hnabla  J(\vtheta_{k}{;} \Dataset_k)} - \epsilon_k\right)^2 - \epsilon^2_k \right) \nonumber\\
	&= \left(\norm{\hnabla J(\vtheta_{k}{;} \Dataset_k)} - \epsilon_k\right)\norm{\hnabla J(\vtheta_{k}{;} \Dataset_k)}.
	\end{align}
	Then, we consider the case in which $\norm{\hnabla J(\vtheta_{k}{;} \Dataset_k)} \leq \epsilon_k$:
	\begin{align}
	\left\langle\hnabla J(\vtheta_{k}{;} \Dataset_k),\nabla J(\vtheta_{k})\right\rangle
	&\geq \frac{1}{2}\left(\norm{\hnabla J(\vtheta_{k}{;} \Dataset_k)}^2 -  \epsilon^2_k \right) \\
	&=\left(\norm{\hnabla J(\vtheta_{k}{;} \Dataset_k)} - \epsilon_k\right)\frac{\norm{\hnabla J(\vtheta_{k}{;} \Dataset_k)} + \epsilon_k}{2}.
	\end{align}
	The two cases can be unified as follows:
	\begin{align}
	&\left\langle\hnabla J(\vtheta_{k}{;} \Dataset_k),\nabla J(\vtheta_{k})\right\rangle
	\geq \left(\norm{\hnabla  J(\vtheta_{k}{;} \Dataset_k)}- \epsilon_k\right)\nonumber\\&\qquad\qquad\qquad\qquad\qquad\qquad\times\max\left\{
	\norm{\hnabla  J(\vtheta_{k}{;} \Dataset_k)}, \frac{\norm{\hnabla  J(\vtheta_{k}{;} \Dataset_k)} + \epsilon_k}{2}
	\right\}\label{eq:16.7}.
	\end{align}
	From Theorem~\ref{th:main} with $\Delta\vtheta=\alpha\hnabla J(\vtheta_{k}{;} \Dataset_k)$ we obtain:
	\begin{align}
	J(\vtheta_{k+1}) - J(\vtheta_{k}) &\geq \left\langle\vtheta_{k+1}-\vtheta_{k}, \nabla J(\vtheta_{k})\right\rangle - \frac{L}{2}\norm{\vtheta_{k+1}-\vtheta_{k}}^2 \nonumber\\
	& = \alpha\left\langle\hnabla J(\vtheta_{k}{;} \Dataset_k), \nabla J(\vtheta_{k})\right\rangle - \frac{\alpha^2L}{2}\norm{\hnabla J(\vtheta_{k}{;} \Dataset_k)}^2 \nonumber\\
	&\geq \alpha\left(\norm{\hnabla  J(\vtheta_{k}{;} \Dataset_k)}- \epsilon_k\right)\nonumber\\&\quad\times\max\left\{
	\norm{\hnabla  J(\vtheta_{k}{;} \Dataset_k)}, \frac{\norm{\hnabla  J(\vtheta_{k}{;} \Dataset_k)} + \epsilon_k}{2}
	\right\} \nonumber\\&\qquad- \frac{\alpha^2L}{2}\norm{\hnabla J(\vtheta_{k}{;} \Dataset_k)}^2,
	\end{align}
	where the last inequality is from (\ref{eq:16.7}).
\end{proof}

From Theorem~\ref{th:storacle} we can easily obtain an optimal step size, as done in the exact setting, provided the batch size is sufficiently large:
\begin{restatable}[]{corollary}{safestepapp}\label{cor:safestepapp}
	Let $\wt{B}(\alpha,N{;}\vtheta_k)$ be the guaranteed performance improvement of a stochastic policy gradient update, as defined in Theorem~\ref{th:storacle}.
	Under the same assumptions, provided the batch size satisfies:
	\begin{align}
	N\geq \frac{\epsilon^2(\delta_k)}{\norm{\hnabla J(\vtheta_k{;} \Dataset_k)}^2},\label{eq:largebatch}
	\end{align}
	$\wt{B}(\alpha, N{;}\vtheta_k)$ is maximized by the following adaptive step size:
	\begin{align}\label{eq:12.3}
	\alpha^*_k = \frac{1}{L}\left(
	1 - \frac{\epsilon(\delta_k)}{\sqrt{N}\norm{\hnabla J(\vtheta_k{;} \Dataset_k)}}
	\right),
	\end{align}
	which guarantees, with probability at least $1-\delta_k$, the following non-negative performance improvement:
	\begin{align}\label{eq:12.4}
	J(\vtheta_{k+1}) - J(\vtheta_k) \geq \frac{\left(
		\norm{\hnabla  J(\vtheta_{k}{;} \Dataset_k)} - \frac{\epsilon(\delta_k)}{\sqrt{N}}
		\right)^2}{2L}.
	\end{align}
\end{restatable}
\begin{proof}
	Let $N_0=\epsilon^2(\delta_k)\norm{\hnabla J(\vtheta_k{;} \Dataset_k)}^{-2}$. When $N\leq N_0$, the second argument of the $\max$ operator in (\ref{eq:16.7}) is selected. In this case, no positive improvement can be guaranteed and the optimal non-negative step size is $\alpha=0$. Thus, we focus on the case $N>N_0$. In this case, the first argument of the $\max$ operator is selected. Optimizing $\wt{B}(\alpha,N)$ as a function of $\alpha$ alone, which is again quadratic, yields (\ref{eq:12.3}) as the optimal step size and (\ref{eq:12.4}) as the maximum guaranteed improvement. 
\end{proof}
In this case, the optimal step size is adaptive, \ie time-varying and data-dependent.
The constant, optimal step size for the exact case (Corollary~\ref{cor:safestep}) is recovered in the limit of infinite data, \ie $N\to\infty$. In the following we discuss why this adaptive step size should not be used in practice, and propose an alternative solution.

\subsection{Adaptive Batch Size}\label{sec:adabatch}
The safe step size from Corollary~\ref{cor:safestepapp} requires the batch size to be large enough. As soon as the condition~\eqref{eq:largebatch} fails to hold, the user is left with the decision whether to interrupt the learning process or collect more data --- an undesirable property for a fully autonomous system. To avoid this, a large batch size must be selected \emph{from the start}, which results in a pointless waste of data in the early learning iterations. Even so, Equation~\eqref{eq:largebatch}, used as a stopping condition, would be susceptible to random oscillations of the stochastic gradient magnitude, interrupting the learning process prematurely.

As observed in~\citep{papini2017adaptive}, controlling also the batch size $N$ of the gradient estimation can be advantageous. Intuitively, a larger batch size yields a more reliable estimate, which in turn allows a safer policy gradient update.
The larger the batch size, the higher the guaranteed improvement, which would lead to selecting the highest possible value of $N$. However, we must take into account the cost of collecting the trajectories, which is non-negligible in real-world problems (\eg robotics). 
For this reason, we would like the meta-parameters to maximize the \emph{per-trajectory} performance improvement:
\begin{equation}
\alpha_k,N_k = \arg\max_{\alpha,N} \frac{J(\vtheta_k+\alpha\hnabla J(\vtheta_k{;}\Dataset)) - J(\vtheta_k)}{N},
\end{equation}
where $\Dataset$ is a dataset of $N$ i.i.d. trajectories sampled with $\pi_{\vtheta_k}$. 
We can then use the lower bound from Theorem~\ref{th:storacle} to find the jointly optimal safe step size and batch size, similarly to what was done in~\citep{papini2017adaptive} for the special case of Gaussian policies:

\begin{restatable}[]{corollary}{safebatch}\label{cor:safebatch}
	Let $\wt{B}_k(\alpha{;}N)$ be the lower bound on the performance improvement of a stochastic policy gradient update, as defined in Theorem~\ref{th:storacle}.
	Under the same assumptions, the continuous relaxation of ${\wt{B}_k(\alpha{;}N)}/{N}$ is maximized by the following step size $\alpha^*$ and batch size $N_k^*$:
	\begin{align}\label{eq:both}
	\begin{cases}
	& \alpha^* =\frac{1}{2L}\\
	& N^*_k = \frac{4\epsilon^2(\delta_k)}{\norm{\hnabla J(\vtheta_{k}{;}\Dataset_k)}^2}.
	\end{cases}
	\end{align}
	Using $\alpha^*$ and $\lceil N^*_k\rceil$ in the stochastic gradient ascent update guarantees, with probability at least $1-\delta_k$, the following non-negative performance improvement:
	\begin{align}\label{eq:guarantee}
	J(\vtheta_{k+1}) - J(\vtheta_k) \geq \frac{\norm{\hnabla J(\vtheta_k{;}\Dataset_k)}^2}{8L}.
	\end{align}
\end{restatable}
\begin{proof}
	Fix $k$ and let $\Upsilon(\alpha,N)=\wt{B}_k(\alpha{;}N)/N$ and $N_0=\epsilon^2(\delta_k)\big/\norm{\hnabla J(\vtheta_k{;} \Dataset_k)}^2$. 	We consider the continuous relaxation of $\Upsilon(\alpha,N)$, where $N$ can be any positive real number. 
	For $N\geq N_0$, the first argument of the $\max$ operator in (\ref{eq:appbound}) can be selected. Note that the second argument is always a valid choice, since it is a lower bound on the first one for every $N\geq1$. Thus, we separately solve the following constrained optimization problems:
	\begin{flalign}\label{eq:prob1}
	&&\begin{cases}
	&\max_{\alpha,N}\frac{1}{N}\left(\alpha\norm{\hnabla J(\vtheta_k{;} \Dataset_k)}\left(\norm{\hnabla J(\vtheta_k{;} \Dataset_k)}-\frac{\epsilon(\delta_k)}{\sqrt{N}}\right)\right.\\&\qquad\qquad\qquad\left.- \alpha^2\frac{L}{2}\norm{\hnabla J(\vtheta_k{;} \Dataset_k)}^2\right) \\
	&\text{s.t.} \quad\alpha\geq 0, \\
	&\phantom{\text{s.t.}} \quad N > \frac{\epsilon^2(\delta_k)}{\norm{\hnabla J(\vtheta_k{;} \Dataset_k)}^2}, 
	\end{cases} && && && && && &&
	\end{flalign}
	and:
	\begin{flalign}\hspace{-10pt}
	&&\label{eq:prob2}
	\begin{cases}
	&\max_{\alpha,N}\frac{1}{N}\left(\frac{\alpha}{2}\left(\norm{\hnabla J(\vtheta_k{;} \Dataset_k)}^2-\frac{\epsilon^2(\delta_k)}{N}\right) - \alpha^2\frac{L}{2}\norm{\hnabla J(\vtheta_k{;} \Dataset_k)}^2\right) \\
	&\text{s.t.} \quad\alpha\geq 0, \\
	&\phantom{\text{s.t.}} \quad N > 0.
	\end{cases} 
	\end{flalign}
	Both problems can be solved in closed form using KKT conditions. The first one (\ref{eq:prob1}) yields $\Upsilon^* = \norm{\hnabla J(\vtheta_k{;} \Dataset_k)}^4\big/\left(32L\epsilon^2(\delta_k)\right)$ with the values of $\alpha^*$ and $N^*_k$ given in (\ref{eq:both}). The second one (\ref{eq:prob2}) yields a worse optimum $\Upsilon^* = \norm{\hnabla J(\vtheta_k{;} \Dataset_k)}^4\big/\left(54L\epsilon^2(\delta_k)\right)$ with $\alpha=\frac{1}{3L}$ and $N=3\epsilon^2(\delta)\big/\norm{\hnabla J(\vtheta_k{;} \Dataset_k)}^2$. Hence, we keep the first solution. From Theorem~\ref{th:storacle}, using $\alpha^*$ and $N^*_k$ would guarantee $J(\vtheta_{k+1}) - J(\vtheta_k) \geq \norm{\hnabla J(\vtheta_k{;}\Dataset_k)}^2\big/\left(8L\right)$. Of course, only integer batch sizes can be used. However, for $N\geq N_0$, the right-hand side of (\ref{eq:appbound}) is monotonically increasing in $N$. Since $N_k^*\geq N_0$ and $\lceil N_k^*\rceil\geq N_k^*$, the guarantee (\ref{eq:guarantee}) is still valid when $\alpha^*$ and $\lceil N_k^*\rceil$ are employed in the stochastic gradient ascent update.
\end{proof}
In this case, the optimal step size is constant, and is exactly half the one for the exact case (Corollary~\ref{cor:safestep}). In turn, the batch size is adaptive: when the norm of the (estimated) gradient is small, a large batch size is selected. Intuitively, this allows to counteract the variance of the estimator, which is large relatively to the gradient magnitude. One may worry about the recursive dependence of $N_k^*$ on itself through $\Dataset_k$. We will overcome this issue in the next section.

\section{Algorithm}\label{sec:algo}
In this section, we leverage the theoretical results of the previous sections to design a reinforcement learning algorithm with monotonic improvement guarantees. For the reasons discussed above, we adopt the adaptive-batch-size approach from Section~\ref{sec:adabatch}. 

Corollary~\ref{cor:safebatch} provides a constant step size $\alpha^*$ and a schedule for the batch size $(\lceil N_k^*\rceil )_{k\ge 1}$ that jointly maximize per-trajectory performance improvement under a monotonic-improvement constraint. Plugging these meta-parameters into Algorithm~\ref{alg:pg}, we could obtain a safe policy gradient algorithm. Unfortunately, the closed-form expression for $N_k^*$ provided in~\eqref{eq:both} cannot be used directly. We must compute the batch size \emph{before} collecting the batch of trajectories $\Dataset_k$, but $N_k^*$ depends on $\Dataset_k$ itself. To overcome this issue, we collect trajectories in an incremental fashion until the optimal batch size is achieved. We call this algorithm Safe Policy Gradient (SPG), outlined in Algorithm~\ref{algo:safepg}. The user specifies the failure probability $\delta_k$ for each iteration $k$, while the smoothness constant $L$ and the concentration bound $\epsilon:(0,1)\to\Reals$ can be computed depending on the policy class and the gradient estimator (see Tables~\ref{tab:smoothing} and~\ref{tab:err}).

\begin{algorithm}[t]
	\caption{Safe Policy Gradient (SPG)}
	\label{algo:safepg}
	\begin{algorithmic}[1] 
		\State \textbf{Input:} initial policy parameter $\vtheta_0$, smoothness constant $L$, concentration bound $\epsilon$, failure probabilities $(\delta_k)_{k\ge 1}$, mini-batch size~$n$
		\State $\alpha=\frac{1}{2L}$ \Comment{fixed step size}
		\For{$k=1,2,\dots$}
		\State $i=0$, $\Dataset_{k,0} = \emptyset$
		\Do
		\State $i=i+1$
		\State Collect trajectory $\tau_{k,i}\sim p_{\vtheta_k}$ 
		\State $\Dataset_{k,i}=\Dataset_{k,i-1}\cup\{\tau_{k,i}\}$
		\State Compute policy gradient estimate $g_{k,i} = \hnabla  J(\vtheta_k{;} \Dataset_{k,i})$
		\State $\delta_{k,i}=\frac{\delta_k}{i(i+1)}$
		\doWhile{$i < \frac{4\epsilon^2(\delta_{k,i})}{\norm{g_{k,i}}^2}$}\label{l:stop}
		\State $N_k=i$, $\Dataset_k=\Dataset_{k,i}$ \Comment{adaptive batch size}
		\State Update policy parameters as $\vtheta_{k+1} \gets \vtheta_k + 
		\alpha\hnabla  J(\vtheta_k{;} \Dataset_k)$
		\EndFor
	\end{algorithmic}
\end{algorithm}

We can study the data-collecting process of SPG as a stopping problem. Fixed an iteration $k$, let $\mathcal{F}_{k,i}=\sigma(\{\tau_{k,1},\dots,\tau_{k,i-1}\})$ be the sigma-algebra generated by the first $i$ trajectories collected at that iteration. Let $\EV_{i}[X]$ be short for $\EV[X\vert \mathcal{F}_{i-1}]$.\footnote{In the analysis that follows, expectation without a subscript actually denotes $\EV[\cdot\vert \vtheta_k]$ for a fixed (outer) iteration $k$ of Algorithm~\ref{algo:safepg}. However, we do not need this level of detail since data are discarded at the end of each iteration, dependence between different iterations is only through $\vtheta_k$, and we can analyze each iteration in isolation until the very end of the section.} In Section~\ref{sec:metaparams} and~\ref{sec:adabatch} we assumed the Euclidean norm of the gradient estimation error to be bounded by $\epsilon(\delta)/\sqrt{N}$ with probability $1-\delta$ for some function $\epsilon:(0,1)\to\Reals_+$. For Algorithm~\ref{algo:safepg} to be well-behaved, we need gradient estimates to concentrate \emph{exponentially}, which translates into the following, stronger assumption:
\begin{assumption}\label{asm:expconc}
	Fixed a parameter $\vtheta\in\Theta$, a batch size $N\in\Naturals$ and a failure probability $\delta\in(0,1)$, with probability at least $1-\delta$: 
	\begin{equation*}
	\norm{\hnabla J(\vtheta{;}\Dataset)-\nabla J(\vtheta)} \le \frac{\epsilon(\delta)}{\sqrt{N}},
	\end{equation*}
	where $\vert\Dataset\vert$ is a dataset of $N$ \iid trajectories collected with $\pi_{\vtheta}$ and $\epsilon(\delta)=C\sqrt{d\log(6/\delta)}$ for some problem-dependent constant $C$ that is independent of $\delta$, $d$ and $N$.
\end{assumption}

This is satisfied by REINFORCE/G(PO)MDP with Softmax and Gaussian policies, as shown in Appendix~\ref{app:bounded}. In Table~\ref{tab:err} we summarize the value of the error bound $\epsilon(\delta)$ to be used in the different scenarios.
Equipped with this exponential tail bound we can prove that, at any given (outer) iteration of SPG, the data-collecting process (inner loop) terminates:
\begin{lemma}\label{lem:stopping}
	Fix an iteration $k$ of Algorithm~\ref{algo:safepg} and let $N_k$ the number of trajectories that are collected at that iteration. Under Assumption~\ref{asm:expconc}, provided $\norm{\nabla J(\vtheta_k)}>0$, $\EV[N_k]<\infty$.
\end{lemma}
\begin{proof}
	First, note that $N_k$ is a stopping time \wrt the filtration $(\mathcal{F}_{k,i})_{i\ge1}$. Consider the event $E_{k,i}=\big\{\norm{g_{k,i}-\nabla J(\vtheta_k)}\le{\epsilon(\delta_{k,i})}/{\sqrt{i}}\big\}$. By Assumption~\ref{asm:expconc}, $\Pro(\lnot E_{k,i})\le\delta_{k,i}$.
	This allows to upper bound the expectation of $N_k$ as follows:
	\begin{align}
		\EV[N_k] &\le \EV\left[\sum_{i=1}^\infty\mathbb{I}\left(\sqrt{i}<\frac{2\epsilon(\delta_{k,i})}{\norm{g_{k,i}}}\right)\right]\\
		&=\EV\left[\sum_{i=1}^\infty\mathbb{I}\left(\sqrt{i}<\frac{2\epsilon(\delta_{k,i})}{\norm{g_{k,i}}}, E_{k,i}\right)\right] + \EV\left[\sum_{i=1}^\infty\mathbb{I}\left(\sqrt{i}<\frac{2\epsilon(\delta_{k,i})}{\norm{g_{k,i}}}, \lnot E_{k,i}\right)\right] \\
		&\le\sum_{i=1}^\infty\mathbb{I}\left(\sqrt{i}<\frac{2\epsilon(\delta_{k,i})}{\norm{\nabla J(\vtheta_k)}-\epsilon(\delta_{k,i})/\sqrt{i}}\right) + \sum_{i=1}^\infty\Pro(\lnot E_{k,i}) \\
		&\le \min_{i\ge 1}\left\{\sqrt{i}\ge\frac{2\epsilon(\delta_{k,i})}{\norm{\nabla J(\vtheta_k)}-\epsilon(\delta_{k,i})/\sqrt{i}}\right\} + \sum_{i=1}^\infty \delta_{k,i} \\
		&\le\min_{i\ge 1}\left\{\norm{\nabla J(\vtheta_k)}\sqrt{i} \ge 3\epsilon(\delta_{k,i})\right\}) + \delta_k\sum_{i=1}^\infty \frac{1}{i(i+1)} \\
		&\le\min_{i\ge 1}\left\{\norm{\nabla J(\vtheta_k)}\sqrt{i} \ge 3C\sqrt{d\log(6i(i+1)/\delta_k)}\right\} + 1 \label{pp:by_asm_2} \\
		&\le\min_{i\ge 1}\left\{\norm{\nabla J(\vtheta_k)}^2i \ge 18C^2d\log(6i/\delta_k)\right\} + 1\\
		&\le \left\lceil \frac{36C^2d}{\norm{\nabla J(\vtheta_k)}^2}\log\frac{108C^2d}{\norm{\nabla J(\vtheta_k)}^2\delta_k}\right\rceil +1,
	\end{align}
	where~\eqref{pp:by_asm_2} is by Assumption~\ref{asm:expconc} and the last inequality is by Lemma~\ref{lem:lambert} assuming $\norm{\nabla J(\vtheta_k)}\le C$. If the latter is not true, we still get:
	\begin{align}
	\EV[N_k] &\le \min_{i\ge 1}\left\{\norm{\nabla J(\vtheta_k)}^2i \ge 18C^2d\log(6i/\delta_k)\right\} + 1 \\
	&\le\min_{i\ge 1}\left\{i \ge 18d\log(6i/\delta_k)\right\} + 1 \\
	&\le \lceil 36d\log(108d/\delta_k)\rceil+1.
	\end{align}
\end{proof}

\begin{table}[t]
	\caption{Gradient estimation error bound $\epsilon(\delta)$ for Gaussian and Softmax policies using REINFORCE (RE.), GPOMDP (GP.), or the random-horizon estimator discussed in Appendix~\ref{app:geom} (RH.) as gradient estimator, where $d$ is the dimension of the policy parameter, $\phimax$ is an upper bound on the max norm of the feature function, $\Rmax$ is the maximum absolute-valued reward, $\gamma$ is the discount factor, $T$ is the task horizon, $\sigma$ is the standard deviation of the Gaussian policy and $\tau$ is the temperature of the Softmax policy.}
	\begin{tabular}{ccc}
		\toprule
		&\textbf{Gaussian} & \textbf{Softmax} \\
		\midrule
		\textbf{RE.}&$\frac{4M\Rmax T(1-\gamma^{\top})}{\sigma(1-\gamma)}\sqrt{{14d\log(6/\delta)}}$ & $\frac{4M\Rmax T(1-\gamma^{\top})}{\tau(1-\gamma)}\sqrt{{2d\log(6/\delta)}}$ \\
		\textbf{GP.}&$\frac{4M\Rmax[1-\gamma^{\top} - T(\gamma^{T}-\gamma^{T+1})]}{\sigma(1-\gamma)^2} \sqrt{{14d\log(6/\delta)}}$ & $\frac{4M\Rmax[1-\gamma^{\top} - T(\gamma^{T}-\gamma^{T+1})]}{\tau(1-\gamma)^2}\sqrt{{2d\log(6/\delta)}}$ \\
		\textbf{RH.}&$\frac{4M\Rmax}{\sigma(1-\gamma^{1/2})^2} \sqrt{{14d\log(6/\delta)}}$ & $\frac{4M\Rmax}{\tau(1-\gamma^{1/2})^2}\sqrt{{2d\log(6/\delta)}}$ \\
		\bottomrule
	\end{tabular}
	\label{tab:err}
\end{table}

We can now prove that the policy updates of SPG are safe.
\begin{theorem}\label{th:mi}
	Consider Algorithm~\ref{algo:safepg} applied to a smoothing policy, where $\hnabla J$ is an unbiased policy gradient estimator.
	Under Assumption~\ref{asm:expconc}, for any iteration~${k\ge 1}$, provided $\nabla J(\vtheta_k)\neq0$, with probability at least $1-\delta_k$:
	\begin{equation*}
		J(\vtheta_{k+1}) - J(\vtheta_k) \ge \frac{\norm{\hnabla J(\vtheta_k{;}\Dataset_k)}^2}{8L}\ge 0.
	\end{equation*}
\end{theorem}
\begin{proof}
	Fix an (outer) iteration $k$ of Algorithm~\ref{algo:safepg} and let $g_{k,i}=\hnabla J(\vtheta_k{;}\Dataset_{k,i})$ for short.
	Using an unbiased policy gradient estimator we ensure $\EV_i[g_{k,i}-\nabla J(\vtheta_k)]=0$, so $X_i=g_{k,i}-\nabla J(\vtheta_k)$ is a martingale difference sequence adapted to $(\mathcal{F}_{k,i})_{i\ge1}$. We use an optional stopping argument to show that $g_{k,N_k}$ is an unbiased policy gradient estimate. Lemma~\ref{lem:stopping} shows that $N_k$ is a stopping time \wrt the filtration $(\mathcal{F}_{k,i})_{i\ge1}$ that is finite in expectation. Furthermore, by Assumption~\ref{asm:expconc}, integrating the tail:
	\begin{align}
		{\EV}_i[\norm{X_i}] &= \int_0^\infty \Pro\left(\norm{X}> x\vert \mathcal{F}_{k,i}\right)\de x \\&\le 6 \int_0^\infty \exp(-x^2i/(C^2d)) \de x \\&\le 6C\sqrt{\frac{\pi d}{4i}} \le 6C\sqrt{\frac{\pi d}{4}} &&\text{for all $i\ge 1$}.
	\end{align}
	Hence, by optional stopping (Lemma~\ref{lem:opt_stop}), $\EV[X_{N_k}]=0$.
	Since $X_{N_k}=\hnabla J(\vtheta_k{;}\Dataset_k)-\nabla J(\vtheta_k)$, we have $\EV[\hnabla J(\vtheta_k{;}\Dataset_k)]=\nabla J(\vtheta_k)$. This shows that the policy update of Algorithm~\ref{algo:safepg} is an unbiased policy-gradient update. By the stopping condition:
	\begin{equation}\label{eq:stopped_batchsize}
		N_k \ge \frac{4\epsilon^2(\delta_{k,N_k})}{\norm{\hnabla J(\vtheta_k{;}\Dataset_k)}^2}.
	\end{equation}
	Now consider the following good event:
	\begin{equation}
	E_k = \left\{\forall i\ge 1: \norm{g_{k,i}-\nabla J(\vtheta_k)}\le \epsilon(\delta_{k,i})/i\right\}.
	\end{equation}
	Under Assumption~\ref{asm:expconc}, by union bound:
	\begin{equation}
	\mathbb{P}\left(\lnot E_k\right) \le \sum_{i=1}^\infty \delta_{k,i} = \sum_{i=1}^\infty\frac{\delta_k}{i(i+1)} = \delta_k.
	\end{equation}
	So $E_k$ holds with probability at least $1-\delta_k$.
	Under $E_k$, the performance improvement guarantee is by Corollary~\ref{cor:safebatch}, Equation~\eqref{eq:stopped_batchsize}, and the choice of the step size~$\alpha$. 
\end{proof}

We have shown that the policy updates of SPG are safe with probability $1-\delta_k$, where the failure probability $\delta_k$ can be specified by the user for each iteration $k$. Typically, one would like to ensure monotonic improvement for the whole duration of the learning process. This can be achieved by appropriate confidence schedules. If the number of updates $K$ is fixed a priori, $\delta_k=\delta/K$ guarantees monotonic improvement with probability $1-\delta$. The same can be obtained by using an adaptive confidence schedule $\delta_k=\frac{\delta}{k(k+1)}$, even when the number of updates is not known in advance. Both results are easily shown by taking a union bound over $k\ge 1$. Notice how having an exponential tail bound like the one from Assumption~\ref{asm:expconc} is fundamental for the batch size to have a logarithmic dependence on the number of policy updates. 

\subsection{Towards a Practical Algorithm}\label{sec:relax}
The version of SPG we have just analyzed is very conservative. The price for guaranteeing monotonic improvement is slow convergence, even in small problems (see Section~\ref{sec:exp_lq} for an example).
In this section, we discuss possible variants and generalizations of Algorithm~\ref{algo:safepg} aimed at the development of a more practical method. In doing so, we still stay faithful to the principle of satisfying the safety requirement specified by the user with no compromises. We just list the changes here. See Appendix~\ref{app:relax} for a more rigorous discussion.

\paragraph{Improved smoothness constant.}
As mentioned in Section~\ref{ssec:perfbound}, we can use the improved smoothness constant by~\cite{yuan2021general}, denoted $L^\star$ in the following, which has a better dependence on the effective horizon. This yields a larger step size with the same theoretical guarantees, and allows to tackle problems with longer horizons in practice.

\paragraph{Mini-batches.} In the inner loop of Algorithm~\ref{algo:safepg}, instead of just one trajectory at a time, we can collect mini-batches of $n$ independent trajectories. For instance, $n\ge 2$ is required to employ the variance-reducing baselines discussed in Section~\ref{sec:pre}. Moreover, a carefully picked mini-batch size $n$ can make the early gradient estimates more stable, leading to an earlier stopping of the inner loop and a smaller batch size $N_k$. We leave the investigation of the optimal value of $n$ to future work.

\paragraph{Largest safe step size.}
The meta parameters of Algorithm~\ref{algo:safepg} were selected to maximize a lower bound on the per-trajectory performance improvement. Although we believe this is the most theoretically justified choice, we could gain some convergence speed by using a larger step size. From Theorem~\ref{th:storacle}, it is easy to check that $\alpha=1/L$ is the largest constant step size we can use with our choice of adaptive batch size from Algorithm~\ref{algo:safepg}. 
We leave the investigation of alternative safe combinations of batch size and (possibly adaptive) step size to future work.

\paragraph{Empirical Bernstein bound.}
The stopping condition of Algorithm~\ref{algo:safepg} (line~\ref{l:stop}) is based on a Hoeffding-style bound on the gradient estimation error. In the case of policies with bounded score function, such as Softmax policies (see Appendix~\ref{ssec:gibbs}), we can use an empirical Bernstein bound instead~\citep{maurer2009empirical}. This requires some modifications to the algorithm, but yields a smaller adaptive batch size with the same safety guarantees. See Appendix~\ref{app:relax} for details. Unfortunately, we cannot use the empirical Bernstein bound with the Gaussian policy because of its unbounded score function (see Appendix~\ref{ssec:gauss}).

\paragraph{Weaker safety requirements.}
Monotonic improvement is a very strong requirement, so we do expect an algorithm with strict monotonic improvement guarantees like SPG to be very data-hungry and slow to converge. However, with little effort, Algorithm~\ref{algo:safepg} can be modified to handle weaker safety requirements. A common one is the \emph{baseline constraint}~\cite[\eg]{garcelon2020conservative,laroche2019safe}, where the performance of the policy is required to never be (significantly) lower than the performance of a baseline policy $\pi_b$. In a real safety-critical application, the reward could be designed so that policies with performance greater than $J(\pi_b)$ are always safe. In other applications, $\pi_b$ can be an existing, reliable controller that the user wants to replace with an adaptive RL agent. In this case, assuming $\pi_{\vtheta_0}=\pi_b$, the baseline constraint guarantees that the learning agent never performs worse than the original controller. In our numerical simulations of Section~\ref{sec:exp}, we will consider a stronger version of the baseline constraint that we call \emph{milestone constraint}. In this case, the agent's policy must never perform (significantly) worse than the best performance observed so far. Formally, for all $k\ge 1$:
\begin{equation}
	J(\vtheta_{k+1}) \ge \lambda \max_{j=1,2,\dots,k}\{J(\vtheta_j)\},\label{eq:milestone}
\end{equation}
where $\lambda\in[0,1]$ is a user-defined significance parameter. The idea is as follows: every time the agent reaches a new level of performance (a \emph{milestone}), it should never do significantly worse than that. When $\lambda=1$, this reduces to monotonic improvement. When $\lambda<1$, some amount of performance oscillation is allowed, but this relaxation can significantly improve the learning speed. Of course, the user has full control on this trade-off through the meta-parameter $\lambda$. In Appendix~\ref{app:relax} we show that variants of Algorithm~\ref{algo:safepg} satisfy the milestone constraint (and other requirements, such as the baseline constraint) with probability $1-\delta$ for given significance $\lambda$ and failure probability $\delta$. We experiment with the milestone constraint in Section~\ref{sec:exp_cartpole}.

\section{Related Works}\label{sec:related}
In this section, we discuss previous results on {MI} guarantees for policy gradient algorithms.

The seminal work on monotonic performance improvement is by~\citet{kakade2002approximately}. In this work, policy gradient approaches are soon dismissed because of their lack of exploration, although they guarantee MI in the limit of an infinitesimally small step size. The authors hence focus on value-based RL, proposing the Conservative Policy Iteration (CPI) algorithm, where the new policy $\pi_{k+1}$is a mixture of the old policy $\pi_k$ and a greedy one $\pi_k^+$. The guaranteed improvement of this new policy~\citep[][Theorem 4.1]{kakade2002approximately} depends on the coefficient $\alpha$ of this convex combination, which plays a similar role as the learning rate in our Theorem~\ref{th:ltwobound}: 
\begin{equation}
J(\pi_{k+1}) - J(\pi_k) \ge \frac{\alpha}{(1-\gamma)}\EVV[\substack{s\sim \Stat^{\pi_k}\\a\sim\pi^{+}_k}]{A^{\pi_k}(s,a)} - \frac{2\alpha^2\gamma\epsilon}{(1-\gamma)^2(1-\alpha)},
\end{equation}
where $\epsilon=\max_{s\in\Sspace}\vert \EVV[a\sim\pi^{+}_k(\cdot\vert s)]{A^{\pi_k}(s,a)}\vert $ and $A^{\pi}(s,a)=Q^\pi(s,a)-V^\pi(s)$ denotes the advantage function of policy $\pi$.
In fact, both lower bounds have a positive term that accounts for the expected improvement of the new policy \wrt the old one, and a penalization term due to the mismatch between the two. 
The CPI approach is refined by~\citet{pirotta2013safe}, who propose the Safe Policy Iteration (SPI) algorithm~\citep[see also][]{metelli2021safe}. 

Specific performance improvement bounds for policy gradient algorithms were first provided by~\citet{pirotta2013adaptive} by adapting previous results on policy iteration~\citep[][]{pirotta2013safe} to continuous {MDP}s. However, the penalty term can only be computed for shallow Gaussian policies (App.~\ref{ssec:gauss}) in practice. The bound for the exact framework is:
\begin{align}\label{eq:pirotta}
J(\vtheta_{k+1}) - J(\vtheta_k) &\ge \alpha_k\norm{\nabla J(\vtheta_k)}^2 -\alpha_k^2 \frac{\phimax^2\Rmax}{\sigma^2(1-\gamma)^2}\left(\frac{\vert \Aspace\vert }{\sqrt{2\pi}\sigma}+\frac{\gamma}{2(1-\gamma)}\right)\nonumber\\&\qquad\qquad\qquad\qquad\qquad\times\norm[1]{\nabla J(\vtheta_k)}^2,
\end{align}
where $\vert \Aspace\vert $ denotes the volume of the action space. From Table~\ref{tab:smoothing}, our bound for the same setting is (Corollary~\ref{cor:safestep}):
\begin{equation*}
J(\vtheta_{k+1}) - J(\vtheta_k) \ge \alpha_k\norm{\nabla J(\vtheta_k)}^2 -\alpha_k^2 \frac{\phimax^2\Rmax}{\sigma^2(1-\gamma)^2}\left(1+\frac{2\gamma}{\pi(1-\gamma)}\right)\norm{\nabla J(\vtheta_k)}^2,
\end{equation*}
which has the same dependence on the step size, the policy standard deviation $\sigma$, the effective horizon $(1-\gamma)^{-1}$, the maximum reward $\Rmax$ and the maximum feature norm $\phimax$. Besides being more general, our penalty term does not depend on the problematic $\vert \Aspace\vert $ term (the action space is theoretically unbounded for Gaussian policies) and replaces the $l_1$ norm of~\eqref{eq:pirotta} with the smaller $l_2$ norm. Due to the different constants, we cannot say our penalty is always smaller, but the change of norm could make a big difference in practice, especially for large parameter dimension $d$.~\citet{pirotta2013adaptive} also study the approximate framework. However, albeit formulated in terms of the estimated gradient, their lower bound (Theorem 5.2) \emph{still pertains exact policy gradient updates}, since $\vtheta_{k+1}$ is defined as $\vtheta_k+\alpha_k\nabla J(\vtheta_k)$. This easy-to-overlook observation makes our Theorem~\ref{th:storacle} the first rigorous monotonic improvement guarantee for stochastic policy gradient updates of the form $\vtheta_{k+1}=\vtheta_k+\alpha_k\hnabla J(\vtheta_k)$.~\citet{pirotta2013adaptive} use their results to design an adaptive step-size schedule for REINFORCE and G(PO)MDP, similarly to what we propose in this paper, but limited to Gaussian policies.~\citet{papini2017adaptive} rely on the same improvement lower bound~\eqref{eq:pirotta} to design an adaptive-batch size algorithm, the most similar to our {SPG}. Again, their monotonic improvement guarantees are limited to shallow Gaussian policies.

Another related family of performance improvement lower bounds, inspired once again by~\citet{kakade2002approximately}, is that of {TRPO}. These are very general results that apply to arbitrary pairs of stochastic policies, although they are mostly used to construct policy gradient algorithms in practice. Specializing Theorem 1 by~\citet{schulman2015trust} to our setting and applying the {KL} lower bound suggested by the authors we can get the following:
\begin{align}\label{eq:trpopg}
J(\vtheta_{k+1}) - J(\vtheta_k) &\ge \frac{1}{1-\gamma}\EVV[\substack{s\sim \Stat^{\vtheta_k}\\a\sim\pi_{\vtheta_{k+1}}}]{A^{\vtheta_k}(s,a)} \nonumber\\ &\qquad-\frac{2\gamma\Rmax}{(1-\gamma)^3}\max_{s\in\Sspace}\left\{\kl(\pi_{\vtheta_k}(\cdot\vert s)\Vert \pi_{\vtheta_{k+1}}(\cdot\vert s))\right\},
\end{align}
where $\pi_{\vtheta}$ is a stochastic policy. Unfortunately, the lower bound for a policy gradient update (exact or stochastic) cannot be computed exactly. Approximations can lead to very good practical algorithms such as {TRPO}, but not to actually implementable algorithms with rigorous monotonic improvement guarantees like our {SPG}.~\citet{achiam2017} and~\citet{pajarinen2019compatible} are able to remove some approximations, but not all.\footnote{This is not a critique of the {TRPO} algorithm per se. Besides the celebrated empirical results, {TRPO} is also theoretically justified~\citep{neu2017unified}, only not best as a monotonically improving gradient-descent algorithm~\citep[see also][]{shani2020adaptive}.} If we were to derive a computable worst-case lower bound starting from~\eqref{eq:trpopg}, we would get a result similar to~\eqref{eq:pirotta}. In fact,~\citet{pirotta2013adaptive} explicitly upper-bound the {KL} divergence in their derivations, which is why the final result is limited to Gaussian policies. We overcome this difficulty by directly upper-bounding the curvature of the objective function (Lemma~\ref{lem:polhessbound}). Furthermore, Theorem~\ref{th:main} suggests that our theory is not limited to policy gradient updates. Arbitrary update directions are considered in~\citep{papini2020balancing}.

\citet{pirotta2015policy} provide performance improvement lower bounds (Lemma 8) and adaptive-step algorithms for policy gradients under Lipschitz continuity assumptions on the {MDP} and the policy. Our assumptions on the environment are much weaker since we only require boundedness of the reward. Intuitively, stochastic policies \emph{smooth out} the irregularities of the environment in computing expected return objectives. In turn, the results of~\citet{pirotta2015policy} also apply to deterministic policies.

\citet{cohen2018diverse} provide a general safe policy improvement strategy that can be applied also to policy gradient updates. However, it requires to maintain and evaluate a \textit{set} of policies per iteration instead of a single one.

As mentioned, \citet[][Lemma 4.4]{yuan2021general} also study policy gradient with smoothing policies, providing an improved smoothness constant and proving Lipschitz continuity of the objective function. However, their main focus is sample complexity of vanilla policy gradient.

\section{Experiments}\label{sec:exp}

In this section, we test our SPG algorithm on simulated control tasks. We first test Algorithm~\ref{algo:safepg} with monotonic improvement guarantees on a small continuous-control problem. We then experiment with the milestone-constraint relaxation proposed in Section~\ref{sec:relax} on a classic RL benchmark --- cart-pole balancing.

\begin{figure}[t]
	\includegraphics[width=\textwidth]{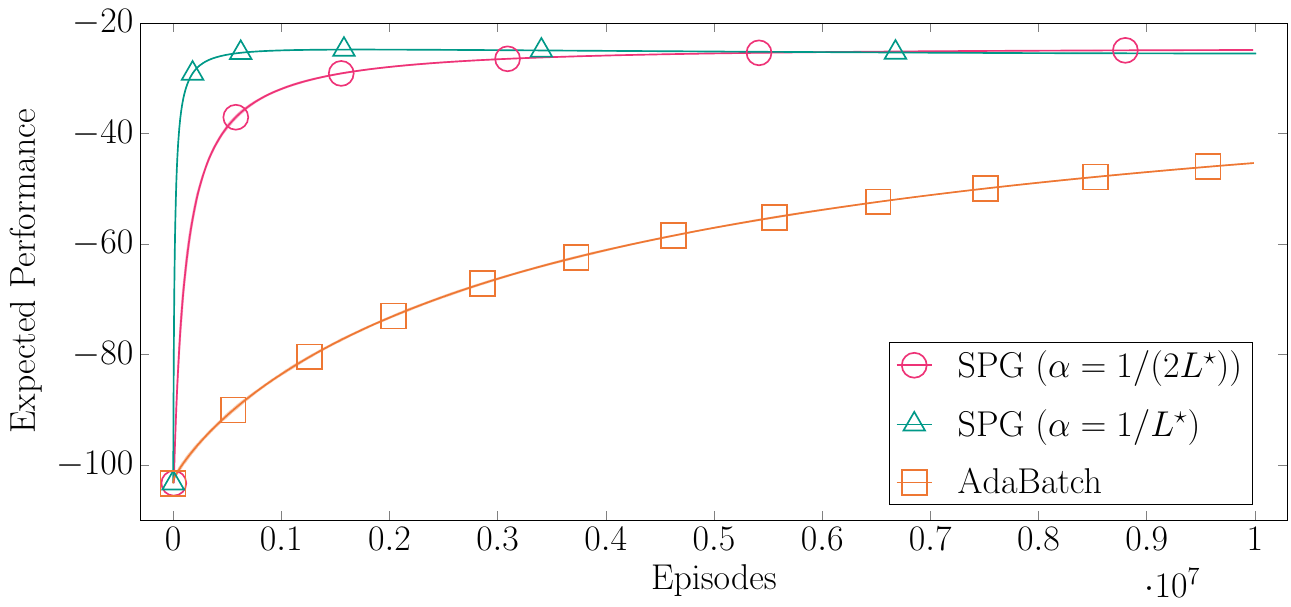}
	\caption{Performance of SPG and AdaBatch~\citep{papini2017adaptive} on the LQR task with Gaussian policy. Results are averaged over $5$ independent runs. The shaded areas correspond to $10$ standard deviations. A marker corresponds to $100$ policy updates.}
	\label{fig:lqg}
\end{figure}

\subsection{Linear-Quadratic Regulator with Gaussian Policy}\label{sec:exp_lq}
The first task is a $1$-dimensional Linear-Quadratic Regulator (LQR,~\cite{dorato1994linear}), a typical continuous-control benchmark. See Appendix~\ref{app:lqg} for a detailed task specification. 
We use a Gaussian policy (Appendix~\ref{ssec:gauss}) that is linear in the state, $\pi_{\vtheta}(a\vert s)=\mathcal{N}(a{;}\theta s, \sigma^2)$. 
The task horizon is $T=10$ and we use $\gamma=0.9$ as a discount factor. The policy mean parameter is initialized to $\theta_0=0$ and the variance is fixed as $\sigma=1$. For this task, the maximum reward (in absolute value) is $\Rmax=1$ and the only feature is the state itself, giving $\phimax=1$. Hence, the smoothness constant $L^\star\simeq 200$ is easily computed (see Table~\ref{tab:smoothing}). Similarly, the error bound can be retrieved from Table~\ref{tab:err}. 
We compare the SPG (Algorithm~\ref{algo:safepg}) with an existing adaptive-batch-size policy gradient algorithm for Gaussian policies~\citep{papini2017adaptive}, discussed in the previous section and labeled \emph{AdaBatch} in the plots.
SPG is run with a mini-batch size of $n=100$ (see Section~\ref{sec:relax}), and AdaBatch (in the version with Bernstein's inequality as recommended in the original paper) with an initial batch size of $N_0=100$. Both use the adaptive confidence schedule $\delta_k=\delta/(k*(k+1))$ discussed in Section~\ref{sec:algo}, with an overall failure probability of $\delta=0.05$. We also consider SPG with a twice-as-large step size $\alpha=1/L^\star$, as discussed in Section~\ref{sec:relax}.

Figure~\ref{fig:lqg} shows the expected performance of the algorithms on the LQR task. For this task, we are able to compute the expected performance in closed form given the policy parameters~\citep{peters2002policy}. This allows to filter out the oscillations due to the stochasticity of policy and environment, focusing on actual (expected) performance oscillations. It is also why the variability among different seeds is so small (note that, for this figure, shaded areas correspond to $10$ standard deviations. They correspond to a single standard deviation in the other figures). Performance is plotted against the total number of collected trajectories for fair comparison. The distribution of policy updates can be deduced from the markers.
We can see that indeed all the safe PG algorithms exhibit monotonic improvement.
SPG converges faster than AdaBatch. This is mostly due to the larger step size of SPG (we observed that the step size of SPG was more than $100$ times larger than the one of AdaBatch in most of the updates). This allows SPG to converge faster even with fewer policy updates. The variant of SPG with a larger step size ($\alpha=1/L^\star$) converges faster to a good policy, but the original version from Algorithm~\ref{algo:safepg} achieves higher performance on the long run. This indicates that maximizing the lower bound on per-trajectory performance improvement from Theorem~\ref{th:storacle} is indeed meaningful.

Figure~\ref{fig:lqg_batchsize} shows the batch size of the different algorithms.
The batch size of SPG is mostly larger than that of AdaBatch. From Section~\ref{sec:related} we know that the monotonic improvement guarantee of SPG is more rigorous, so a larger batch size is justified. Notice also that the batch size of SPG is smaller than that of AdaBatch in the early iterations, suggesting that the former is more adaptive.

\begin{figure}[t]
	\includegraphics[width=\textwidth]{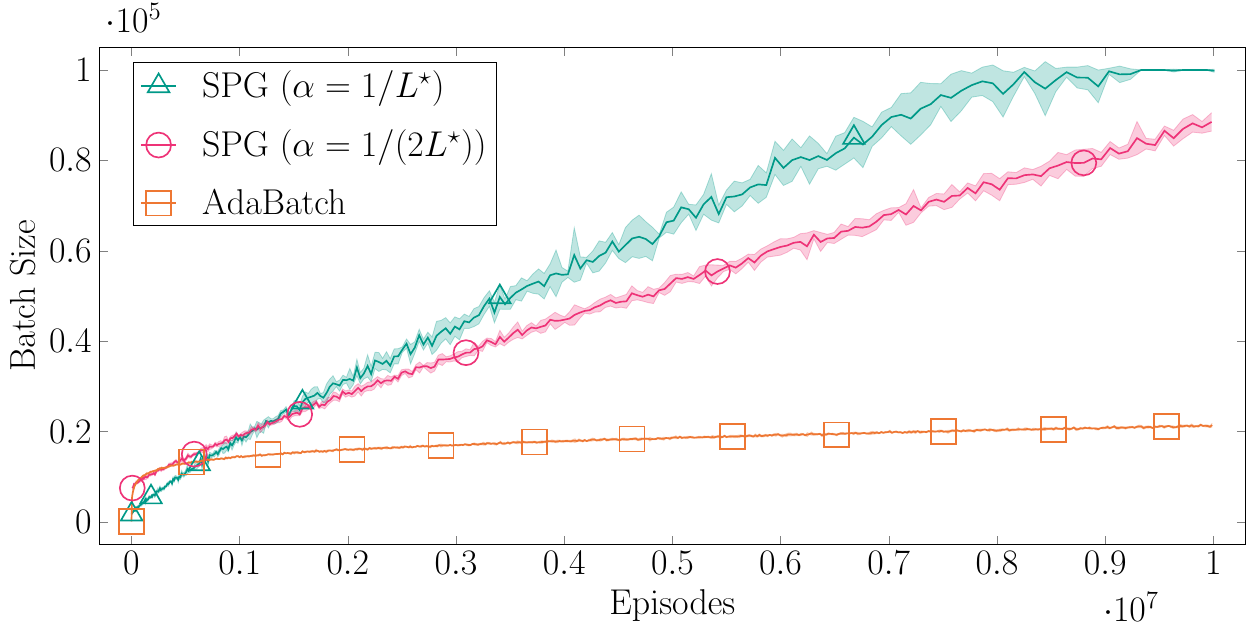}
	\caption{Batch size of SPG and AdaBatch on the LQR task. Results are averaged over $5$ independent runs. The shaded areas correspond to one standard deviation. A marker corresponds to $100$ policy updates}
	\label{fig:lqg_batchsize}
\end{figure}

\subsection{Cart-Pole with Softmax Policy}\label{sec:exp_cartpole}
The second task is cart-pole~\citep{barto1983neuronlike}. We use the implementation from \texttt{openai/gym}, which has $4$-dimensional continuous states and finite actions, $a\in\{1,2\}$. See Appendix~\ref{app:cartpole} for further details. The policy is Softmax (Appendix~\ref{ssec:gibbs}), linear in the state: $\pi_{\vtheta}(a\vert s)\propto\exp(\vtheta_a^\top s)$, with a separate parameter for each action ($\vtheta = [\vtheta_1; \vtheta_2]$). We use a fixed temperature $\tau=1$, initial policy parameters set to zero (this corresponds to a uniform policy) and $\gamma=0.9$ as a discount factor. For SPG, we employ all the practical variants proposed in Section~\ref{sec:relax}. In particular, since the Softmax policy has a bounded score function, we can use the empirical Bernstein bound. Note that we could not have done the same for the LQG task since the score function of the Gaussian policy is unbounded (see Appendix~\ref{app:bounded}). Moreover, we consider the relaxed milestone constraint for different values of the significance parameter, $\lambda\in\{0.1,0.2,0.4\}$. The overall failure probability is always $\delta=0.2$, the mini-batch size is $n=100$, and the step size is $\alpha=1/L^\star$.\footnote{Although both theory and our LQR experiments indicate that $\alpha=1/(2L^\star)$ is ultimately the best choice, we prioritize convergence speed over long-term performance on this larger task.} We compare with GPOMDP with the same step size but a fixed batch size of $N=100$, which comes with no safety guarantees, and corresponds to $\lambda=0$. In Figure~\ref{fig:cartpole} we plot the performance against the total number of collected trajectories. As expected, a more relaxed constraint yields faster convergence. However, no significant performance oscillations are observed, not even in the case of GPOMDP, suggesting that the choice of meta-parameters is still over-conservative. In Figure~\ref{fig:cartpole2} (left) we report the evolution of the batch size of SPG during the learning process. Note how, in this case, the batch size seems to converge to a constant value. In Figure~\ref{fig:cartpole2} (right) we illustrate the milestone constraint. The solid line is the performance of SPG with $\lambda=0.1$, while the dotted line is the performance lower-threshold enforced by the milestone constraint, representing $90\%$ of the highest performance achieved so far. As desired, the actual performance never falls under the threshold.

\begin{figure}[t]
	\includegraphics[width=\textwidth]{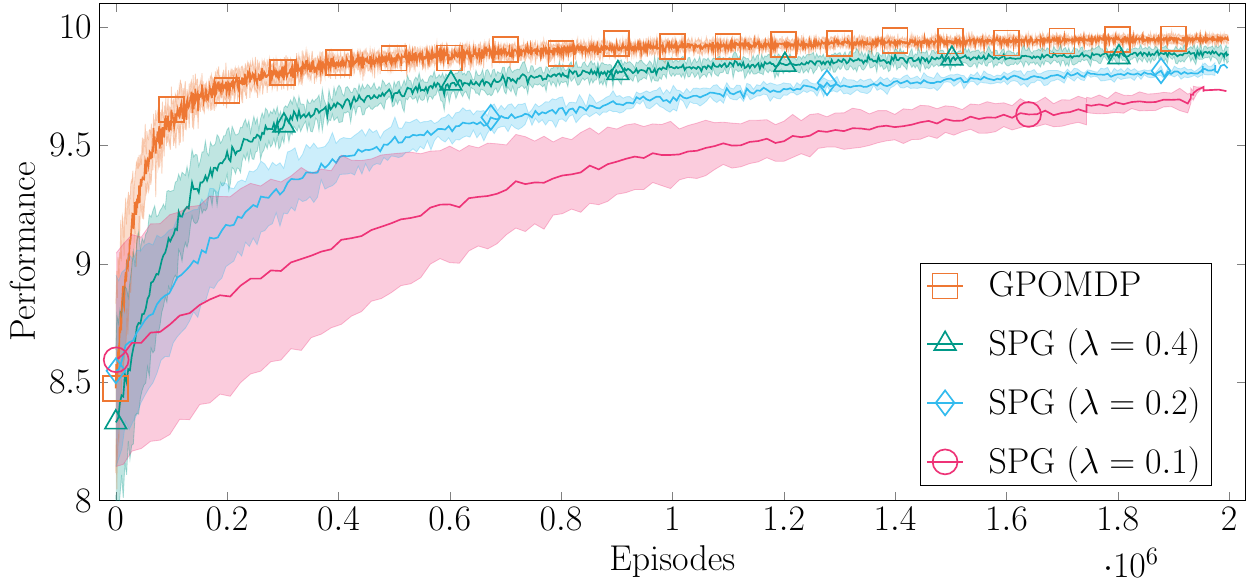}
	\caption{Performance of GPOMDP and SPG (for different values of the significance parameter $\lambda$) on the cart-pole task with Softmax policy. Results are averaged over $5$ independent runs. The shaded areas correspond to one standard deviation. A marker corresponds to $1000$ policy updates.}
	\label{fig:cartpole}
\end{figure}

\section{Conclusion}\label{sec:end}
We have identified a general class of policies, called smoothing policies, for which the performance measure (expected total reward) is a smooth function of policy parameters. We have exploited this property to select meta-parameters for actor-only policy gradient that guarantee monotonic performance improvement. We have shown that an adaptive batch size can be used in combination with a constant step size for improved efficiency, especially in the early stages of learning. We have designed a monotonically improving policy gradient algorithm, called Safe Policy Gradient (SPG), with adaptive batch size. 
We have shown how SPG can also be applied to weaker performance-improvement constraints.
Finally, we have tested SPG on simulated control tasks. 

\begin{figure}[t]
	\includegraphics[width=\textwidth]{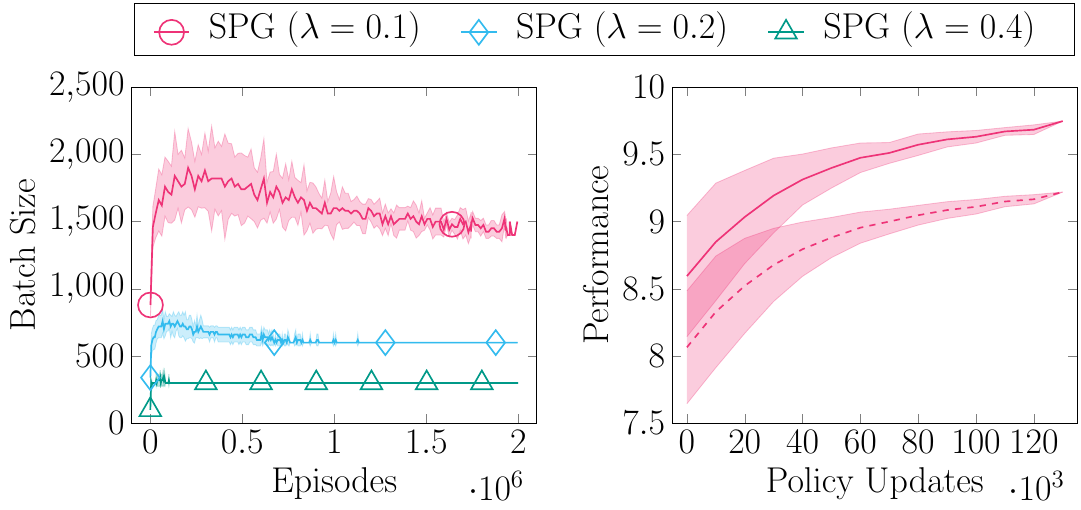}
	\caption{Further results for SPG on the cart-pole task. On the left, the batch size is plotted against the total number of trajectories. A marker corresponds to $1000$ policy updates. On the right, the performance at each policy update (solid line) is compared with the performance threshold (dashed line) when $\lambda=0.1$. In both plots, shaded areas correspond to one standard deviation.}
	\label{fig:cartpole2}
\end{figure}

Albeit the safety motivations are clearly of practical interest, our contribution is mostly theoretical. The meta-parameters proposed in Section~\ref{sec:metaparams} and used in SPG are based on worst-case problem-dependent constants that are known and easy to compute, but can be very large. This would lead to over-conservative behavior in most problems of interest. However, we believe that this work provides a solid starting point to develop safe and efficient policy gradient algorithms that are rooted in theory. 

To conclude, we propose some possible ideas for future work that are aimed to close this gap between theory and practice. While we used the empirical Bernstein bound to characterize the gradient estimation error for Softmax policies, the same cannot be done for Gaussian policies due to their unbounded score function. Tighter concentration inequalities should be studied for this case. 
The convergence rate of SPG should also be studied. The main challenge here is the growing batch size. The numerical simulations of section~\ref{sec:exp_lq} suggest the the growth is sublinear. Moreover, we have observed convergence to a fixed batch size under the weaker milestone constraint in Section~\ref{sec:exp_cartpole}. It is also worth to investigate whether SPG can be combined with stochastic variance-reduction techniques~\citep[\eg][]{papini2018stochastic,yuan2020stochastic}. Convergence to global optima should also be investigated, as is now common in the policy optimization literature~\citep{bhandari2019global,zhang2020sample,agarwal2020optimality}.
Actor-critic algorithms~\citep{konda1999actor} are more used than actor-only algorithms in practice~\cite[\eg][]{haarnoja2018soft} due to their reduced variance. Thus, extending our improvement guarantees to this class of algorithms is also important. The main challenge lies in handling the bias due to the critic. A promising first step is to consider \emph{compatible} critics that yield unbiased gradient estimates~\citep{sutton2000policy,konda1999actor}.
Although the class of smoothing policies is very broad, we have restricted our attention to Gaussian and Softmax policies with given features. Other policy classes, such as beta policies~\citep{chou2017improving} should be considered. Most importantly, \emph{deep} policies should be considered that also learn the features from data, especially given their success in practice~\citep{duan2018benchmarking}. See Appendix~\ref{app:deep} for a brief discussion.
Other possible extensions include generalizing the monotonic improvement guarantees to other concepts of safety, such as learning under constraints, or risk-averse RL~\citep{bisi2020risk}.
Finally, the conservative approach adopted in this work could prevent \emph{exploration}, making some tasks very hard to learn. We studied the case of Gaussian policies with adaptive standard deviation in~\citep{papini2020balancing}. Future work should consider the trade-off between safety, efficiency and exploration in greater generality.

\backmatter

%
%
%

\bmhead{Acknowledgments}
The authors would like to thank Gergely Neu for his suggestions on how to improve the Safe Policy Gradient algorithm and its theoretical analysis.

\clearpage

%
%
%
%
%
%
%
%

\begin{appendices}

\section{Omitted Proofs}\label{app:proof}

\subsection{Markov Decision Processes}\label{app:mdp}
\begin{lemma}\label{th:dspi}
	For all $\pi:\Sspace\to\Delta_{\Aspace}$ and $s_0\in\Sspace$:
	\begin{equation*}
	\Stat^{\pi}_{s_0}(\cdot) = (1-\gamma) \indi{s=s_0} + \gamma \int_{\Sspace}\Stat^{\pi}_{s_0}(s)p(\cdot\vert s) \de s.
	\end{equation*}
\end{lemma}
\begin{proof}
	\begin{align*}
	\gamma \int_{\Sspace}\Stat^{\pi}_{s_0}(s)p(\cdot\vert s) \de s & = \gamma \int_{\Sspace}(1-\gamma)\sum_{t=0}^{\infty}\gamma^t\Tran_{\pi}^t(s\vert s_0)\Tran_{\pi}(\cdot\vert s) \de s \\
	&= \gamma(1-\gamma)\sum_{t=0}^{\infty}\gamma^{t}\int_{\Sspace}\Tran_{\pi}^t(s\vert s_0)\Tran_{\pi}(\cdot\vert s) \de s \\
	&= (1-\gamma)\sum_{t=0}^{\infty}\gamma^{t+1}\int_{\Sspace}\Tran_{\pi}^t(s\vert s_0)\Tran_{\pi}(\cdot\vert s) \de s \\
	&= (1-\gamma)\sum_{t=0}^{\infty}\gamma^{t+1}\Tran_{\pi}^{t+1}(\cdot\vert s_0) 
	\\& = (1-\gamma)\sum_{t=1}^{\infty}\gamma^{t}\Tran_{\pi}^{t}(\cdot\vert s_0)
	\\& = (1-\gamma)\sum_{t=0}^{\infty}\gamma^{t}\Tran_{\pi}^{t}(\cdot\vert s_0) - (1-\gamma) 
	\\& 
	= \Stat_{s_0}^{\pi}(\cdot) - (1-\gamma) \indi{s=s_0}.
	\end{align*}
\end{proof}

\eigenfunction*
\begin{proof}
	\begin{align}
	\int_{\Sspace}\Stat^{\pi}_{s}(s')g(s')\de s' 
	&= \int_{\Sspace}\Stat^{\pi}_{s}(s')f(s')\de s' - \int_{\Sspace}\Stat^{\pi}_{s}(s')\gamma\int_{\Sspace}\Tran_{\pi}(s''\vert s')f(s'')\de s''\de s' \nonumber\\
	&= \int_{\Sspace}\Stat^{\pi}_{s}(s')f(s')\de s' - \int_{\Sspace} \gamma\int_{\Sspace} \Stat_{s}^{\pi}(s')\Tran_{\pi}(s''\vert s')\de s' f(s'')\de s'' \nonumber\\
	&= \int_{\Sspace}\Stat^{\pi}_{s}(s')f(s')\de s' - \int_{\Sspace} \left(\Stat^{\pi}_{s}(s'') - (1-\gamma)\indi{s''=s}\right) f(s'')\de s'' \label{eq:c121} \\
	&= (1-\gamma) f(s),\nonumber
	\end{align}
	where~\eqref{eq:c121} is from Lemma~\ref{th:dspi}.
\end{proof}

\subsection{Lipschitz-Smooth Functions}\label{app:smooth}
The following results, reported in Section~\ref{sec:pre}, are well known in the literature~\citep{nesterov1998introductory}, but we also report proofs for the sake of completeness:

\mvi*
\begin{proof}
	Let $\vx,\vx'\in\Xspace$, $\vh\coloneqq\vx'-\vx$ and $g:[0,1]\to\Reals$, $g(\lambda)\equiv \nabla_{\vx}f(\vx+\lambda\vh)$. Convexity of $\Xspace$ guarantees $\vx+\lambda\vh\in \Xspace$ for $\lambda\in[0,1]$. 
	Twice-differentiability of $f$ implies $\nabla_{\vx}f$ is continuous, which in turn implies $g$ is continuous. From the Fundamental Theorem of Calculus:
	\begin{align}
	\nabla_{\vx}f(\vx') - \nabla_{\vx}f(\vx)
	&= \nabla_{\vx}f(\vx+\vh) - \nabla_{\vx}f(\vx)
	= g(1) - g(0) = \int_0^1g'(\lambda)\de\lambda\nonumber\\
	&= \int_0^1 \vh^{\top}\Hess_{\vx}f(\vx+\lambda\vh)\de\lambda.
	\end{align}
	Hence:
	\begin{align}
	\norm{\nabla_{\vx}f(\vx') - \nabla_{\vx}f(\vx)} 
	&= \norm[2]{\int_0^1 \vh^{\top}\Hess_{\vx}f(\vx+\lambda\vh)\de\lambda}\nonumber\\
	&\leq \int_0^1 \norm[2]{\Hess_{\vx}f(\vx+\lambda\vh)\vh}\de\lambda\nonumber\\
	&\leq \int_0^1 \norm[2]{\Hess_{\vx}f(\vx+\lambda\vh)}\norm[2]{\vh}\de\lambda\label{eq:2.1}\\
	&\leq L\norm[2]{\vh} = L\norm[2]{\vx'-\vx},
	\end{align}
	where (\ref{eq:2.1}) is from the consistency of induced norms, \ie $\norm[p]{A\vx}\leq\norm[p]{A}\norm[p]{\vx}$.
\end{proof}

\smooth*
\begin{proof}
	Let $\vx,\vx'\in\Xspace$, $\vh\coloneqq \vx'-\vx$ and $g: [0,1]\to\Reals$, $g(\lambda) \equiv f(\vx + \lambda \vh)$. Convexity of $\Xspace$ guarantees $\vx+\lambda\vh\in \Xspace$ for $\lambda\in[0,1]$. Lipschitz smoothness implies continuity of $f$, which in turn implies $g$ is continuous. From the Fundamental Theorem of Calculus:
	\begin{align}
	f(\vx') - f(\vx) = g(1) - g(0) = \int_{0}^{1}g'(\lambda)\de \lambda.
	\end{align}
	Hence:
	\begin{align}
	\left\vert f(\vx') - f(\vx) - \left\langle\vx'-
	\right.\right.&\left.\left.
	\vx,\nabla_{\vx}f(\vx)\right\rangle\right\vert
	= \left\vert\int_{0}^{1}g'(\lambda)\de \lambda  - \left\langle\vh,\nabla_{\vx}f(\vx)\right\rangle\right\vert \nonumber\\
	&= \left\vert\int_{0}^{1}
	\left\langle\vh, \nabla_{\vx}f(\vx+\lambda \vh)\right\rangle
	\de \lambda  - \left\langle\vh,\nabla_{\vx}f(\vx)\right\rangle\right\vert \nonumber\\
	&= \left\vert\int_{0}^{1}
	\left\langle\vh, \nabla_{\vx}f(\vx+\lambda \vh) - \nabla_{\vx}f(\vx)\right\rangle
	\de \lambda\right\vert \nonumber\\
	&\leq \int_{0}^{1}\left\vert
	\left\langle\vh, \nabla_{\vx}f(\vx+\lambda \vh) - \nabla_{\vx}f(\vx)\right\rangle
	\right\vert\de \lambda \nonumber\\
	&\leq \int_{0}^{1}
	\norm[2]{\nabla_{\vx}f(\vx+\lambda \vh) - \nabla_{\vx}f(\vx)}\norm[2]{\vh}
	\de \lambda \label{eq:3.4}\\
	&\leq L\norm[2]{\vh}^2\int_{0}^{1}\lambda
	\de \lambda \label{eq:3.5}\\
	&= \frac{L}{2}\norm[2]{\vx' - \vx}^2,\nonumber
	\end{align}
	where (\ref{eq:3.4}) is from the Cauchy-Schwartz inequality and (\ref{eq:3.5}) is from the Lipschitz smoothness of $f$.
\end{proof}

\subsection{Smoothing Policies and Differentiability}\label{sec:leibniz}
Our proofs of the results of Section~\ref{ssec:polhess} rely on the interchange of integrals (\wrt states and actions) and derivatives (\wrt policy parameters). In the policy gradient literature~\cite[\cf][]{sutton2000policy,konda1999actor,kakade2001optimizing}, these are typically justified by assuming the derivatives of the policy are bounded uniformly over states and actions, that is:
\begin{align}
	\left\vert\frac{\partial}{\partial\theta_i}\pi_{\vtheta}(a\vert s)\right\vert \le C_1, && \left\vert\frac{\partial^2}{\partial\theta_i\partial\theta_j}\pi_{\vtheta}(a\vert s)\right\vert \le C_2, \label{asm:uniform}
\end{align}
for all $s\in\Sspace$, $a\in\Aspace$, $\vtheta\in\Theta\subseteq\Reals^d$, and $i,j=1,2,\dots,d$. The policy gradient itself originally relies on this assumption~\citep{konda1999actor}, although weaker requirements are possible~\citep[see][Section 5.1, for a recent discussion]{bhandari2019global}. The main problem with~\eqref{asm:uniform} is that the uniform bounds may depend on huge quantities such as the diameter of the parameter space. Even worse, for (linear) Gaussian policies, the first derivative is unbounded:
\begin{equation}
	\nabla \pi_{\vtheta}(a\vert s) = \pi_{\vtheta}(a\vert s)\frac{a-\vtheta^\top\phi(s)}{\sigma^2}\phi(s),
\end{equation}
even when $\phi(s)$ is bounded, since $a\in\Aspace=\Reals$. However, these policies are smoothing (see Appendix~\ref{ssec:gauss}).

The following application of the Leibniz Integral Rule~\citep[\cf][Theorem 6.28]{klenke2013probability} shows that our smoothing-policy assumption (Definition~\ref{def:smooth}), can replace the stronger~\eqref{asm:uniform} in differentiating expectations:
\begin{lemma}\label{lem:leibniz}
Let	$\{\pi_{\vtheta}\vert\vtheta\in\Theta\}$, be a class of smoothing policies and $f:\Sspace\times\Aspace\to\Reals$ be any function such that $\sup_{a\in\Aspace}\vert f(s,a)\vert$ is integrable on $\Sspace$. Then $\int_{\Sspace}\int_{\Aspace} \pi_{\vtheta}(a\vert s)f(s,a)\de a\de s$ is twice differentiable and:
\begin{align}
	&\frac{\partial}{\partial\theta_i}\int_{\Sspace}\int_{\Aspace} \pi_{\vtheta}(a\vert s)f(s,a)\de a\de s= \int_{\Sspace}\int_{\Aspace} \frac{\partial}{\partial\theta_i}\pi_{\vtheta}(a\vert s)f(s,a)\de a\de s,\\
	&\frac{\partial^2}{\partial\theta_i\partial\theta_j}\int_{\Sspace}\int_{\Aspace} \pi_{\vtheta}(a\vert s)f(s,a)\de a\de s= \int_{\Sspace}\int_{\Aspace} \frac{\partial^2}{\partial\theta_i\partial\theta_j}\pi_{\vtheta}(a\vert s)f(s,a)\de a\de s,
\end{align}
for all $i,j=1,2,\dots,d$.
\end{lemma}
\begin{proof}
	Let $B_s=\sup_{a\in\Aspace}\vert f(s,a)\vert$ and fix an index $i\le d$. Let:
	\begin{equation}
		u_s(\vtheta) = \int_{\Aspace}\pi_{\vtheta}(a\vert s)f(s,a)\de a.
	\end{equation}
	By definition:
	\begin{equation}
		\frac{\partial}{\partial\theta_i}u_s(\vtheta) = \lim_{h\to 0}\frac{u(\vtheta + h\boldsymbol{e}_i) - u(\vtheta)}{h},
	\end{equation}
	where $\boldsymbol{e}_i$ is the element of the canonical basis of $\Reals^d$ corresponding to the $i$-th coordinate. By linearity of integration:
	\begin{equation}
		\frac{\partial}{\partial\theta_i}u_s(\vtheta) = \lim_{h\to 0}\int_\Aspace\underbrace{\frac{\pi_{\vtheta+h\boldsymbol{e}_i}(a\vert s)-\pi_{\vtheta}(a\vert s)}{h}f(s,a)}_{g_{\vtheta}(s,a)}\de a.\label{eq:dct}
	\end{equation}
	By assumption, $\pi_{\vtheta}(a\vert s)$ is differentiable, so it is continuous. Fix an $h\in\Reals$. By the mean value theorem, there exist a $\overline{\vtheta}$ on the segment connecting $\vtheta$ and $\vtheta+h\boldsymbol{e}_i$ such that:
	\begin{equation}
		\frac{\pi_{\vtheta+h\boldsymbol{e}_i}(a\vert s)-\pi_{\vtheta}(a\vert s)}{h} = \frac{\partial}{\partial\theta_i}\pi_{\vtheta}(a\vert s)\Bigg\vert_{\vtheta=\overline{\vtheta}}.
	\end{equation}
	Hence, by upper bounding the $l_\infty$ norm with the $l_2$ norm:
	\begin{equation}
		\left\vert g_{\vtheta}(s,a)\right\vert 
		\le B_s\norm{\nabla_{\overline{\vtheta}}\pi_{\overline{\vtheta}}(a\vert s)}.
	\end{equation}
	By the smoothing-policy assumption, $\Theta$ is convex, so $\overline{\vtheta}\in\Theta$, and again by the smoothing-policy assumption:
	\begin{equation}
	\int_{\Aspace}\norm{\nabla_{\overline{\vtheta}}\pi_{\overline{\vtheta}}(a\vert s)}\de a 
	\le \int_{\Aspace}\pi_{\overline{\vtheta}}(a\vert s)\norm{\nabla_{\overline{\vtheta}}\log\pi_{\overline{\vtheta}}(a\vert s)}\de a  \le \sm,\label{eq:bct}
	\end{equation}
	showing that $g_{\vtheta}(s,a)$ is bounded by a function that is integrable \wrt $a$. By the dominated convergence theorem, we can interchange the limit and the integral in~\eqref{eq:dct} to obtain:
	\begin{equation*}
		\frac{\partial}{\partial\theta_i}u_s(\vtheta) = \int_\Aspace\lim_{h\to 0}\frac{\pi_{\vtheta+h\boldsymbol{e}_i}(a\vert s)-\pi_{\vtheta}(a\vert s)}{h}f(s,a)\de a
		= \int_\Aspace\frac{\partial}{\partial\theta_i}\pi_{\vtheta}(a\vert s)f(s,a)\de a.
	\end{equation*}
	By~\eqref{eq:bct} and Holder's inequality, $\vert{\partial}/{\partial\theta_i}u_s(\vtheta)\vert \le B_s\sm$, which is integrable on $\Sspace$. We can then use the same interchange argument to show that:
	\begin{equation}
		\frac{\partial}{\partial\theta_i}\int_\Sspace u_s(\vtheta) \de s = \int_\Sspace \frac{\partial}{\partial\theta_i}u_s(\vtheta) \de s = \int_{\Sspace}\int_\Aspace\frac{\partial}{\partial\theta_i}\pi_{\vtheta}(a\vert s)f(s,a)\de a.
	\end{equation}
	
	For the second derivative, we can just repeat the whole argument from the previous paragraph on $\frac{\partial}{\partial\theta_i}u_s(\vtheta)$. Continuity of the integrand, which is necessary to apply the mean value theorem, follows from twice differentiability of the policy. To apply the dominated convergence theorem, we use the following:
	\begin{align*}
		&\int_{\Aspace}\norm{\nabla^2\pi_{\vtheta}(a\vert s)}\de a 
		=\int_{\Aspace}\pi_{\vtheta}(a\vert s)\norm{\nabla\log\pi_{\vtheta}(a\vert s)\nabla^\top\log\pi_{\vtheta}(a\vert s) + \nabla^2\log\pi_{\vtheta}(a\vert s)}\de a \\
		&\qquad\le \int_{\Aspace}\pi_{\vtheta}(a\vert s)\norm{\nabla\log\pi_{\vtheta}(a\vert s)}^2\de a + \int_{\Aspace}\pi_{\vtheta}(a\vert s)\norm{\nabla^2\log\pi_{\vtheta}(a\vert s)}\de a \\
		&\qquad\le \smm + \smmm,
	\end{align*}
	by the triangular inequality and the smoothing-policy assumption.
\end{proof}

With some work, one can use Lemma~\ref{lem:leibniz} to justify all the interchanges of differentiation and integrals from Section~\ref{ssec:polhess} and Appendix~\ref{app:pol}, as the original derivations~\citep{sutton2000policy,kakade2001optimizing} were justified by~\eqref{asm:uniform}.

\subsection{Policy Hessian}\label{app:pol}

In the following, the interchange of differentiation and integrals is justified by our smoothing-policy assumption. See Appendix~\ref{sec:leibniz} for details.

	\polhess*
	\begin{proof}
	We first compute the Hessian of the state-value function:
	\begin{align}
	\Hess V^{\vtheta}(s)
	&= \Hess\int_{\Aspace}\pi_{\vtheta}(a\vert s)Q^{\vtheta}(s,a)\de a\nonumber\\ 
	&=\int_{\Aspace}\nabla\left[ \pi_{\vtheta}(a\vert s)\left(\nabla^{\top}\log\pi_{\vtheta}(a\vert s)Q^{\vtheta}(s,a) +
	\nabla^{\top} Q^{\vtheta}(s,a)
	\right)\right]\de a \label{p:polhess.4}\\
	&=\int_{\Aspace} \pi_{\vtheta}(a\vert s)\left[(\Hess \log\pi_{\vtheta}(a\vert s) + \nabla\log\pi_{\vtheta}(a\vert s)\nabla^{\top}\log\pi_{\vtheta}(a\vert s))Q^{\vtheta}(s,a) 
	\nonumber\right.\\
	&\quad\left.+
	\nabla\log\pi_{\vtheta}(a\vert s)\nabla^{\top} Q^{\vtheta}(s,a) 
	+ \nabla Q^{\vtheta}(s,a)\nabla^{\top}\log\pi_{\vtheta}(a\vert s)
	+ \Hess Q^{\vtheta}(s,a)
	\right]\de a \label{p:polhess.5}\\
	&=\int_{\Aspace} \pi_{\vtheta}(a\vert s)\left[\phantom{\int}(\Hess \log\pi_{\vtheta}(a\vert s) + \nabla\log\pi_{\vtheta}(a\vert s)\nabla^{\top}\log\pi_{\vtheta}(a\vert s))Q^{\vtheta}(s,a) 
	\nonumber\right.\\
	&\quad\left.+
	\nabla\log\pi_{\vtheta}(a\vert s)\nabla^{\top} Q^{\vtheta}(s,a)
	+ \nabla Q^{\vtheta}(s,a)\nabla^{\top}\log\pi_{\vtheta}(a\vert s)
	\nonumber\right.\\
	&\quad\left.
	+ \Hess \left(\Rew(s,a)+\gamma\int_{\Sspace} \Tran(s'\vert s,a)V^{\vtheta}(s')\de s'\right)\right]\de a \label{p:polhess.6}\\
	&=\int_{\Aspace}\pi_{\vtheta}(a\vert s)\left[(\Hess \log\pi_{\vtheta}(a\vert s) + \nabla\log\pi_{\vtheta}(a\vert s)\nabla^{\top}\log\pi_{\vtheta}(a\vert s))Q^{\vtheta}(s,a) 
	\nonumber\right.\\
	&\quad\left.+
	\nabla\log\pi_{\vtheta}(a\vert s)\nabla^{\top} Q^{\vtheta}(s,a)
	+ \nabla Q^{\vtheta}(s,a)\nabla^{\top}\log\pi_{\vtheta}(a\vert s)\right]\de a
	\nonumber\\&\quad+ 
	\gamma\int_{\Sspace}p_{\vtheta}(s'\vert s)\Hess V^{\vtheta}(s')\de s'\nonumber\\
	&= g(s) + \frac{\gamma}{1-\gamma}\int_{\Sspace}\Stat_{s}^{\vtheta}(s')g(s')\de s',\label{p:polhess.3}
	\end{align}
	where 
	\begin{align*}
	g(s)&=\int_{\Aspace}\pi_{\vtheta}(a\vert s)\left[\left(
	\nabla\log\pi_{\vtheta}(a\vert s)\nabla^{\top} \log\pi_{\vtheta}(a\vert s) + \Hess \log\pi_{\vtheta}(a\vert s)\right)
	Q^{\vtheta}(s,a)\right.\\&\left.\qquad+\nabla \log\pi_{\vtheta}(a\vert s)\nabla^{\top} Q^{\vtheta}(s,a)
	+ \nabla Q^{\vtheta}(s,a)\nabla^{\top}\log\pi_{\vtheta}(a\vert s)\right]\de a,
	\end{align*}
	(\ref{p:polhess.4}) is from the log-derivative trick, (\ref{p:polhess.5}) is from another application of the log-derivative trick, (\ref{p:polhess.6}) is from (\ref{def:q}), and (\ref{p:polhess.3}) is from Lemma \ref{th:eigenfunction} with $\Hess V^{\vtheta}(s')$ as the recursive term.
	Computing the Hessian of the performance measure is then trivial:
	\begin{align}
	\Hess J(\vtheta) = \Hess \int_{\Sspace}\Init(s)V^{\vtheta}(s)\de s = \int_{\Sspace}\Init(s)\Hess V^{\vtheta}(s)\de s\label{p:polhess.7},
	\end{align}
	where the first equality is from (\ref{eq:defj}). Combining (\ref{p:polhess.3}), (\ref{p:polhess.7}) and (\ref{eq:dmupi}) we obtain the statement of the lemma.
\end{proof}

\subsection{Auxiliary Lemmas}
\begin{lemma}\label{lem:lambert}
	For any $a,b>0$ such that $ab>1$, a sufficient condition for $x\ge a\log(bx)$ is $x\ge2a\log(ab)$.
\end{lemma}
\begin{proof}
	This can be deduced from the properties of the Lambert function. However, it is easier to verify it directly. Letting $x=2a\log(ab)$, the first inequality becomes:
	\begin{align}
		2a\log(ab) \ge a\log(2ab\log(ab)) = a\log(ab) + a\log(2\log(ab)),
	\end{align}
	and $\log(2\log(y)) \le \log y$ for any $y>1$. Finally, notice that $x-a\log(bx)$ is increasing for $x>a$, and $2a\log(ab)>a$ for $ab>1$.
\end{proof}

\begin{lemma}[Optional Stopping]\label{lem:opt_stop}
	Let $(X_t)_{t\ge 1}$ be a $d$-dimensional vector-valued martingale difference sequence and $\tau$ be a stopping time, both with respect to a filtration $(\mathcal{F}_t)_{t\ge 0}$. If $\EV[\tau]<\infty$ and there exists $c\ge0$ such that $\EV[\norm{X_t}\vert \mathcal{F}_{t-1}] \le c$ for every $t\ge 1$, then $\EV[X_\tau]=0$.
\end{lemma}
\begin{proof}
	Consider any martingale $Y_t$ such that $X_t=Y_{t}-Y_{t-1}$. We are going to apply Doob's optional stopping theorem~\citep[See Thm 12.5.9 from][]{grimmett2020probability}\footnote{In the theorem, it is also required that $\Pro(\tau<\infty)=1$, but this is implied by $\EV[\tau]<\infty$ since $\tau$ is nonnegative.} to each element $Y_t^{(i)}$ of $Y_t$, where $i=1,\dots,d$. Sufficient conditions for $\EV[Y_{\tau}^{(i)}]=\EV[Y_0^{(i)}]$ are:
	\begin{enumerate}
		\item $\EV[\tau]<\infty$,
		\item $\EV[\vert Y^{(i)}_{t+1}-Y_t^{(i)}\vert\mid \mathcal{F}_{t}] \le c$ for all $t\ge 0$. 
	\end{enumerate}
	The first one is by hypothesis. For the second one:
	\begin{align}
		\max_{i\in[d]}\EV[\vert Y^{(i)}_{t+1}-Y_t^{(i)}\vert\mid \mathcal{F}_{t}] &= \max_{i\in[d]}\EV[\vert X_{t+1}^{(i)}\vert\mid \mathcal{F}_{t}] \nonumber\\&\le \EV[\norm[\infty]{X_{t+1}}\vert\mid \mathcal{F}_{t}] \nonumber\\&\le \EV[\norm[2]{X_t}\vert\mid \mathcal{F}_{t}] \le c,
	\end{align} 
	where the last inequality is by hypothesis. So, by optional stopping, $\EV[Y_{\tau}^{(i)}]=\EV[Y_0^{(i)}]$ for all $i\in[d]$. We can repeat the same argument for $\tau-1$. Hence $\EV[X_{\tau}]=\EV[Y_{\tau}]-\EV[Y_{\tau-1}] = \EV[Y_0]-\EV[Y_0]=0$.
\end{proof}

\section{Common Smoothing Policies}\label{sec:polclass}
In this section, we show that some of the most commonly used parametric policies are smoothing and provide the corresponding Lipschitz constants for the policy gradient. 

\subsection{Gaussian policy}\label{ssec:gauss}
Consider a scalar-action, fixed-variance, shallow Gaussian policy:\footnote{In this section, $\pi$ with no subscript always denotes the mathematical constant.}
\begin{align}\label{eq:gauss1}
\pi_{\vtheta}(a\vert s) = \Gauss\left(a\vert\vtheta^{\top}\vphi(s), \sigma^2\right) =  \frac{1}{\sqrt{2\pi}\sigma}\exp\left\{-\frac{1}{2}\left(\frac{a - \vtheta^{\top}\vphi(s)}{\sigma}\right)^2\right\},
\end{align}
where $\vtheta\in\Theta\subseteq\Reals^d$, $\sigma>0$ is the standard deviation, and $\vphi:\Sspace\to\Reals^d$ is a vector-valued feature function that is bounded in \emph{Euclidean norm}, \ie $\sup_{s\in\Sspace}\norm{\vphi(s)}<\infty$. This common policy turns out to be smoothing.
\begin{restatable}[]{lemma}{gauss}\label{lem:gauss}
	Let $\Pi_{\Theta}$ be the set of Gaussian policies defined in (\ref{eq:gauss1}), with parameter set $\Theta$, standard deviation $\sigma$ and feature function $\vphi$. Let $\phimax$ be a non-negative constant such that $\sup_{s\in\Sspace}\norm{\vphi(s)}\leq\phimax$. Then $\Pi_{\Theta}$ is $(\sm,\smm,\smmm)$-smoothing with the following constants:
	\begin{align*}
	&\sm = \frac{2\phimax}{\sqrt{2\pi}\sigma}, &\smm = \smmm = \frac{\phimax^2}{\sigma^2}.
	\end{align*}
	The corresponding Lipschitz constant of the policy gradient is:
	\begin{align}
	L = \frac{2\phimax^2\Rmax}{\sigma^2(1-\gamma)^2}\left(1+\frac{2\gamma}{\pi(1-\gamma)}\right).
	\end{align}
\end{restatable}
\begin{proof}
	Fix a $\vtheta\in\Theta$. Let $x \equiv \frac{a-\vtheta^{\top}\vphi(s)}{\sigma}$. Note that $\Aspace = \Reals$ and $\de a = \sigma \de x$.
	We need the following derivatives:
	\begin{align}
	&\nabla \log \pi_{\vtheta}(a\vert s) = \frac{\vphi(s)}{\sigma}x,\\
	&\Hess \log\pi_{\vtheta}(a\vert s) = -\frac{\vphi(s)\vphi(s)^{\top}}{\sigma^2}.
	\end{align}
	
	First, we compute $\sm$:
	\begin{align}
	\EVV[a\sim\pi_{\vtheta}(\cdot\vert s)]{\norm{\nabla\log\pi_{\vtheta}(a\vert s)}}
	&= \int_{\Reals}\frac{1}{\sqrt{2\pi}\sigma}e^{- \nicefrac{x^2}{2}}
	\norm{\frac{\vphi(s)}{\sigma}x}\sigma \de x \nonumber\\
	&\leq\frac{\phimax}{\sqrt{2\pi}\sigma}\int_{\Reals}e^{- \nicefrac{x^2}{2}}\vert x\vert\de x \nonumber\\
	&= \frac{2\phimax}{\sqrt{2\pi}\sigma} \coloneqq \sm.
	\end{align}
	
	Then, we compute $\smm$:
	\begin{align}
	\EVV[a\sim\pi_{\vtheta}(\cdot\vert s)]{\norm{\nabla\log\pi_{\vtheta}(a\vert s)}^2}
	&=\int_{\Reals}\frac{1}{\sqrt{2\pi}\sigma}e^{- \nicefrac{x^2}{2}}
	\norm{\frac{\vphi(s)}{\sigma}x}^2\sigma \de x\nonumber\\
	&\leq\frac{\phimax^2}{\sqrt{2\pi}\sigma^2}\int_{\Reals}e^{- \nicefrac{x^2}{2}}x^2\de x\nonumber\\
	&= \frac{\phimax^2}{\sigma^2} \coloneqq \smm.
	\end{align}
	Finally, we compute $\smmm$:
	\begin{align}
	\EVV[a\sim\pi_{\vtheta}(\cdot\vert s)]{\norm{\Hess \log\pi_{\vtheta}(a\vert s)}}
	&=\int_{\Reals}\frac{1}{\sqrt{2\pi}\sigma}e^{- \nicefrac{x^2}{2}}
	\norm{\frac{\vphi(s)}{\sigma}x}^2\sigma \de x\nonumber\\
	&\leq\frac{\phimax^2}{\sqrt{2\pi}\sigma^2}\int_{\Reals}e^{- \nicefrac{x^2}{2}}x^2\de x\nonumber\\
	&= \frac{\phimax^2}{\sigma^2} \coloneqq \smmm.
	\end{align}
	From these constants, the Lipschitz constant of the policy gradient is easily computed (Lemma~\ref{lem:smoothj}).
\end{proof}

\subsection{Softmax policy}\label{ssec:gibbs}
Consider a fixed-temperature, shallow Softmax policy for a discrete action space:
\begin{align}\label{eq:gibbs}
\pi_{\vtheta}(a\vert s) = \frac{
	\exp\left\{\frac{\vtheta^{\top}\vphi(s,a)}{\tau}\right\}}{\sum_{a'\in\Aspace}
	\exp\left\{\frac{\vtheta^{\top}\vphi(s,a')}{\tau}\right\}},
\end{align}
where $\vtheta\in\Theta\subseteq\Reals^d$, $\tau>0$ is the temperature, and $\vphi:\Sspace\times\Aspace\to\Reals^d$ is a vector-valued feature function that is bounded in \emph{Euclidean norm}, \ie $\sup_{s\in\Sspace,a\in\Aspace}\norm{\vphi(s,a)}<\infty$.
This policy is smoothing.
\begin{restatable}[]{lemma}{gibbs}\label{lem:gibbs}
	Let $\Pi_{\Theta}$ be the set of Softmax policies defined in (\ref{eq:gibbs}), with parameter set $\Theta$, temperature $\tau$ and feature function $\vphi$. Let $\phimax$ be a non-negative constant such that $\sup_{s\in\Sspace,a\in\Aspace}\norm{\vphi(s,a)}\leq\phimax$. Then, $\Pi_{\Theta}$ is ($\sm,\smm,\smmm$)-smoothing with the following constants:
	\begin{align*}
	&\sm = \frac{2\phimax}{\tau}, &\smm =\frac{4\phimax^2}{\tau^2}, & &\smmm = \frac{2\phimax^2}{\tau^2}.
	\end{align*}
	The corresponding Lipschitz constant of the policy gradient is:
	\begin{align}
	L = \frac{2\phimax^2\Rmax}{\tau^2(1-\gamma)^2}\left(3+\frac{4\gamma}{1-\gamma}\right).
	\end{align}
\end{restatable}
\begin{proof}
	In this case, we can simply bound $\norm{\nabla\log\pi_{\vtheta}(a\vert s)}$ and $\norm{\Hess \log\pi_{\vtheta}(a\vert s)}$ uniformly over states and actions. The smoothing conditions follow trivially. We need the following derivatives:
	\begin{align}
	&\nabla\log\pi_{\vtheta}(a\vert s) =\frac{1}{\tau}\left(\vphi(s,a) - \EVV[a'\sim\pi_{\vtheta}(\cdot\vert s)]{\vphi(s,a')}\right),\\
	&\Hess \log\pi_{\vtheta}(a\vert s) = \frac{1}{\tau^2}\EVV[a'\sim\pi_{\vtheta}(\cdot\vert s)]{\vphi(s,a')\left(\EVV[a''\sim\pi_{\vtheta}(\cdot\vert s)]{\vphi(s,a'')} - \vphi(s,a')\right)^{\top}}.
	\end{align}
	First, we compute $\sm$ and $\smm$:
	\begin{align}
	\norm{\nabla\log\pi_{\vtheta}(a\vert s)}
	&\leq \frac{1}{\tau}\left(\norm{\vphi(s,a)} + \norm{\EVV[a'\sim\pi_{\vtheta}(\cdot\vert s)]{\vphi(s,a')}}\right)\nonumber\\
	&\leq \frac{2\phimax}{\tau},\label{eq:12.1}
	\end{align}
	hence $\sup_{s\in\Sspace}\EVV[a\sim\pi_{\vtheta}]{\norm{\nabla\log\pi_{\vtheta}(a\vert s)}} \leq \frac{2\phimax}{\tau} \coloneqq \sm$ and $\sup_{s\in\Sspace}\EVV[a\sim\pi_{\vtheta}]{\norm{\nabla\log\pi_{\vtheta}(a\vert s)}^2} \leq \frac{4\phimax^2}{\tau^2} \coloneqq \smm$.
	
	Finally, we compute $\smmm$:
	\begin{align}
	\norm{\Hess\log\pi_{\vtheta}(a\vert s)} 
	&\leq\frac{1}{\tau^2}\EV_{a'\sim\pi_{\vtheta}(\cdot\vert s)}\left[\norm{\vphi(s,a')\left(
		\EV_{a''\sim\pi_{\vtheta}(\cdot\vert s)}\left[\vphi(s,a'')\right] - \vphi(s,a')
		\right)^{\top}}\right] \nonumber\\
	&\leq\frac{1}{\tau^2}\EV_{a'\sim\pi_{\vtheta}(\cdot\vert s)}\left[
	\norm{\vphi(s,a')}\norm{\EV_{a''\sim\pi_{\vtheta}(\cdot\vert s)}\left[\vphi(s,a'') - \vphi(s,a')\right]}
	\right]\nonumber\\
	&\leq \frac{1}{\tau^2}\EV_{a'\sim\pi_{\vtheta}(\cdot\vert s)}\left[
	\norm{\vphi(s,a')}\EV_{a''\sim\pi_{\vtheta}(\cdot\vert s)}\left[\norm{\vphi(s,a'')} + \norm{\vphi(s,a')}\right]
	\right]\nonumber\\
	&\leq \frac{2\phimax^2}{\tau^2},\label{eq:12.2}
	\end{align}
	hence $\sup_{s\in\Sspace}\EVV[a\sim\pi_{\vtheta}]{\norm{\Hess\log\pi_{\vtheta}(a\vert s)}} \leq \frac{2\phimax^2}{\tau^2} \coloneqq \smmm$.
	From these constants, the Lipschitz constant of the policy gradient is easily computed (Lemma~\ref{lem:smoothj}).
\end{proof}
Note the similarity with the Gaussian constants from Lemma \ref{lem:gauss}. The temperature parameter $\tau$ plays a similar role to the standard deviation $\sigma$.

The smoothness constants for Gaussian and Softmax policies are summarized in Table~\ref{tab:smoothing}.

\subsection{Preliminary Results on Deep Policies}\label{app:deep}
The policies we have considered so far rely on given feature maps from state (and action) space to low-dimensional linear space. For many applications, a linear map is not expressive enough to represent good policies. Deep policies~\citep{duan2018benchmarking} use Neural Networks (NN) to extract more powerful representations from data. Here we provide a first analysis on how the properties of the NN affect the smoothing properties of the policy. 

As an example, consider a Gaussian policy with mean parametrized by a NN, that is:
\begin{equation}
	\pi_{\vtheta}(a\vert s) \sim \Gauss(a\vert \mu_{\vtheta}(s),\sigma^2),
\end{equation}
where $\mu_{\vtheta}:\Sspace\to\Aspace$ is a NN with weights $\vtheta$. The score function is then:
\begin{equation}
	\nabla_{\vtheta}\log\pi_{\vtheta}(a\vert s) = \frac{a-\mu_{\vtheta}(s)}{\sigma^2}\nabla_{\vtheta}\mu_{\vtheta}(s),
\end{equation}
and its second-order counterpart:
\begin{equation}
	\nabla_{\vtheta}\log\pi_{\vtheta}(a\vert s) = \frac{a-\mu_{\vtheta}(s)}{\sigma^2}\nabla^2_{\vtheta}\mu_{\vtheta}(s) - \frac{\nabla_{\vtheta}\mu_{\vtheta}(s)\nabla^\top_{\vtheta}\mu_{\vtheta}(s)}{\sigma^2}.
\end{equation}
For the policy to be smoothing, we need bounds on the gradient and Hessian of the NN \wrt its weights (in Euclidean and spectral norm, respectively), both uniformly over the state space. This may suggest the use of activation functions that are smooth and have bounded derivatives for any input, such as $\tanh$ or sigmoid activations. We shall study the impact of the network architecture on the smoothing constants in future work.

\section{Exponential Concentration of Policy Gradient Estimators}\label{app:bounded}
In this section, we provide exponential tail inequalities for REINFORCE and G(PO)MDP (see Section~\ref{sec:pre}) policy gradient estimators with policy classes of interest. 
For the G(PO)MDP estimator, it is useful to notice that it can be equivalently written as~\citep{sutton2000policy,peters2008reinforcement}:
\begin{align}\label{eq:pgt}
\hnabla J(\vtheta{;}\Dataset) = \frac{1}{N}\sum_{i=1}^{N}\sum_{t=0}^{T-1}\left[\nabla\log\pi_{\vtheta}(a_t^i\vert s_t^i)\sum_{h=t}^{T-1}\gamma^hR(a_h^i,s_h^i)\right],
\end{align}
just by reordering. For simplicity, we will consider estimators without variance-reducing baselines.

First, let us consider the case of a bounded score function:

\begin{lemma}\label{lem:bounded_score}
	Let $\norm{\nabla \log \pi_{\vtheta}(a\vert s)}\le W$ for all $\vtheta\in\Reals^d$, $s\in\Sspace$ and $a\in\Aspace$. Then, for any $\vtheta\in\Reals^d$, with probability $1-\delta$:
	\begin{equation}
		\norm{\hnabla J(\vtheta)-\nabla J(\vtheta)} \le 2WR_T\sqrt{\frac{2d\log(6/\delta)}{N}},
	\end{equation}
	where $R_T=\frac{\Rmax T(1-\gamma^{\top})}{1-\gamma}$ for REINFORCE and $R_T=\Rmax \frac{1-\gamma^{\top} - T(\gamma^{T}-\gamma^{T+1})}{(1-\gamma)^2}$ for~G(PO)MDP.
\end{lemma}
\begin{proof}
	For REINFORCE let:
	\begin{align}
		&\mathcal{R}_t(\tau) \coloneqq \sum_{h=0}^{T-1}\gamma^h r_h \le \frac{\Rmax(1-\gamma^{\top})}{1-\gamma} \coloneqq \overline{\mathcal{R}}_t \qquad \text{for all $t\ge 0$}, \\ 
		&R_T\coloneqq\sum_{t=0}^{T-1}\overline{\mathcal{R}}_t = \frac{\Rmax T(1-\gamma^{\top})}{1-\gamma}.
	\end{align}
	For G(PO)MDP, let:
	\begin{align}
		&\mathcal{R}_t(\tau) \coloneqq \sum_{h=t}^{T-1}\gamma^h r_h \le \frac{\Rmax(\gamma^t - \gamma^{\top})}{1-\gamma} \coloneqq \overline{\mathcal{R}}_t, \\
		 & R_T\coloneqq\sum_{t=0}^{T-1}\overline{\mathcal{R}}_t = \Rmax \frac{1-\gamma^{\top} - T(\gamma^{T}-\gamma^{T+1})}{(1-\gamma)^2}.
	\end{align}
	Let $S^{d-1}=\{v\in\Reals^d: \norm{v}=1\}$ be the unit sphere in $\Reals^d$. Fix a vector $v\in S^{d-1}$ and let $\hnabla J(\vtheta)$ denote the policy gradient estimate obtained from a single trajectory $\tau$ sampled from $p_{\vtheta}$. For both gradient estimators:
	\begin{align}
		\langle v, \hnabla J(\vtheta)\rangle &= \sum_{t=0}^{T-1}\langle v, \nabla \log\pi_{\vtheta}(a_t\vert s_t)\rangle \mathcal{R}_t(\tau)\\
		&\le  \sum_{t=0}^{T-1}\norm{\nabla\log\pi_{\vtheta}(a_t\vert s_t)}\mathcal{R}_t(\tau) \\
		&\le WR_T,\label{eq:sg_bound}
	\end{align}
	where the first inequality uses the fact that, for any $x\in\Reals^d$, $\norm{x} = \max_{v\in S^{d-1}}\langle v,x\rangle$.
	By linearity of expectation and unbiasedness of the gradient estimator, $\EV[\langle v,\hnabla J(\vtheta)\rangle] = \langle v, \nabla J(\vtheta)\rangle$. Hence, by~\eqref{eq:sg_bound} and Hoeffding's inequality, with probability $1-\delta_v$:
	\begin{align}
	\langle v,\hnabla J(\vtheta{;}\Dataset)-\nabla J(\vtheta)\rangle \le WR_T\sqrt{\frac{2\log(1/\delta_v)}{N}},
	\end{align}
	where $N=\vert\Dataset\vert$. To turn this into a bound on the Euclidean norm, we need a covering argument. For arbitrary $\eta>0$, consider an $\eta$-cover $C_{\eta}$ of $S^{d-1}$, that is:
	\begin{equation}\label{eq:cover}
		\max_{v\in S^{d-1}, w\in C_{\eta}}\norm{v-w} \le \eta.
	\end{equation}
	It is a well known result that a finite cover $C_\eta$ exists such that $\vert C_{\eta}\vert\le (3/\eta)^d$. Then, with probability $1-\delta$:
	\begin{align}
	\norm{\hnabla J(\vtheta)-\nabla J(\vtheta)} &=\max_{v\in S^{d-1}}\langle v,\hnabla J(\vtheta{;}\Dataset)-\nabla J(\vtheta)\rangle \\
	&\le \max_{v\in C_\eta}\langle v,\hnabla J(\vtheta{;}\Dataset)-\nabla J(\vtheta)\rangle + \eta\norm{\hnabla J(\vtheta)-\nabla J(\vtheta)} \\
	&\le WR_T\sqrt{\frac{2\log(\vert C_\eta\vert/\delta)}{N}} + \eta\norm{\hnabla J(\vtheta)-\nabla J(\vtheta)} \\
	&\le WR_T\sqrt{\frac{2d\log\left(\frac{3}{\eta\delta}\right)}{N}} + \eta\norm{\hnabla J(\vtheta)-\nabla J(\vtheta)},
	\end{align}
	where the first inequality is by Cauchy-Schwarz inequality and definition of $C_\eta$, the second one is by union bound over the finite elements of $C_\eta$, and the last inequality uses the covering number of the sphere in $\Reals^d$. Finally, by letting $\eta=1/2$:
	\begin{equation}
		\norm{\hnabla J(\vtheta)-\nabla J(\vtheta)} \le \frac{WR_T}{1-\eta}\sqrt{\frac{2d\log\left(\frac{3}{\eta\delta}\right)}{N}} = 2WR_T\sqrt{\frac{2d\log(6/\delta)}{N}}.
	\end{equation}
\end{proof}

The Softmax policy described in Appendix~\ref{ssec:gibbs} satisfies the assumption of Lemma~\ref{lem:bounded_score} with $W=\frac{2M}{\tau}$ where $M$ is an upper bound on $\norm{\phi(s)}$, as shown in the proof of Lemma~\ref{lem:gibbs}.

Unfortunately, the Gaussian policy class from Appendix~\ref{ssec:gauss} is not covered by Lemma~\ref{lem:bounded_score}, since its score function is unbounded. Motivated by the broad use of Gaussian policies in applications, we provide an ad-hoc bound for this class:

\begin{lemma}\label{lem:sg_score}
	Let $\Pi_{\Theta}$ be the class of shallow Gaussian policies from Lemma~\ref{lem:gauss}. Then, for any $\vtheta\in\Theta$, with probability $1-\delta$:
		\begin{equation}
	\norm{\hnabla J(\vtheta)-\nabla J(\vtheta)} \le \frac{4MR_T}{\sigma}\sqrt{\frac{14d\log(6/\delta)}{N}},
	\end{equation}
	where $R_T=\frac{\Rmax T(1-\gamma^{\top})}{1-\gamma}$ for REINFORCE and $R_T=\Rmax \frac{1-\gamma^{\top} - T(\gamma^{T}-\gamma^{T+1})}{(1-\gamma)^2}$ for~G(PO)MDP.
\end{lemma}
\begin{proof}
	Let $\mathcal{R}_t(\tau)$, $\overline{\mathcal{R}}_t$, and $R_T$ be defined (differently for the two gradients estimators) as in the proof of Lemma~\ref{lem:bounded_score}. Again, let $S^{d-1}$ be the unit sphere in $\Reals^d$ , fix a vector $v\in S^{d-1}$ and let $\hnabla J(\vtheta)$ denote the policy gradient estimate obtained from a single trajectory $\tau$ sampled from $p_{\vtheta}$. Consider the filtration $(\mathcal{F}_t)_{t=0}^{T-1}$ where $\mathcal{F}_t = \sigma(s_0,a_0,\dots,s_{t})$ is the sigma-algebra representing all the knowledge up to the $t$-th state included. Conditional on $s_t$, $a_t\sim\mathcal{N}(\vtheta^\top\phi(s),\sigma^2)$. Hence, conditionally on $\mathcal{F}_t$:
	\begin{align}
		\langle v, \nabla \log \pi_{\vtheta}(a_t\vert s_t)\rangle = \frac{a_t-\vtheta^\top\phi(s_t)}{\sigma^2}\langle v,\phi(s_t)\rangle \sim \mathcal{N}\left(0, \frac{\langle v,\phi(s_t)\rangle^2}{\sigma^2}\right).\label{eq:gaussian_score}
	\end{align}
	Let $X_t=\langle v, \nabla \log \pi_{\vtheta}(a_t\vert s_t)\rangle$ for brevity. Since $X_t$ is $\mathcal{F}_t$-measurable and $\EV[X_t\vert \mathcal{F}_{t-1}]=0$, $(X_t)_t$ is a martingale difference sequence adapted to $(\mathcal{F}_t)_t$. Furthermore,~\eqref{eq:gaussian_score} shows that, for any $\lambda>0$:
	\begin{equation}
		\EV[\exp(\lambda X_t)\vert \mathcal{F}_t] = \exp\left(\frac{\lambda\langle v,\phi(s_t)\rangle^2}{2\sigma^2}\right) \le \exp\left(\frac{\lambda\norm{\phi(s_t)}^2}{2\sigma^2}\right) \le \exp\left(\frac{\lambda M^2}{2\sigma^2}\right),
	\end{equation}
	where the first inequality is by $\norm{x}=\max_{v\in S^{d-1}}\langle v,x\rangle$ for any $x\in\Reals^d$. Hence, $X_t$ is conditionally $(M/\sigma)$-subgaussian and $\overline{\mathcal{R}}_tX_t$ is conditionally $(\overline{\mathcal{R}}_tM/\sigma)$-subgaussian. Using Azuma's inequality, for any $b>0$:\footnote{We use the version by~\citet{shamir2011variant}.}
	\begin{align}
		\Pro\left(\langle v,\hnabla J(\vtheta)\rangle>b\right) &= \Pro\left(\sum_{t=0}^{T-1}X_t\mathcal{R}_t(\tau)>b\right) \\
		&\le \Pro\left(\sum_{t=0}^{T-1}X_t\overline{\mathcal{R}}_t>b\right) \\
		&\le \exp\left(-\frac{\sigma^2b^2}{56M^2\sum_{t=0}^{T-1}\overline{\mathcal{R}}_t^2}\right)\\
		&\le \exp\left(-\frac{\sigma^2b^2}{56M^2R_T^2}\right),
	\end{align}
	showing that $\langle v,\hnabla J(\vtheta)\rangle$ is $\sqrt{28}\phimax R_T/\sigma$-subgaussian. From this and $\EV[\langle v,\hnabla J(\vtheta{;}\Dataset)\rangle]=\langle v,\nabla J(\vtheta)\rangle$, using Hoeffding's inequality for averages of i.i.d. subgaussian random variables:
		\begin{align}
	\langle v,\hnabla J(\vtheta{;}\Dataset)-\nabla J(\vtheta)\rangle \le \frac{MR_T}{\sigma}\sqrt{\frac{56\log(1/\delta_v)}{N}},
	\end{align}
	with probability $1-\delta_v$.
	Finally, using the same covering argument as in the proof of Lemma~\ref{lem:bounded_score}, with probability $1-\delta$:
	\begin{equation}
	\norm{\hnabla J(\vtheta)-\nabla J(\vtheta)} \le \frac{2MR_T}{\sigma}\sqrt{\frac{56d\log(6/\delta)}{N}}.
	\end{equation}
\end{proof}

The values of $\epsilon(\delta)$ for Gaussian and Softmax policies are summarized in Table~\ref{tab:err}.

\subsection{Empirical Bernstein Bound}
For bounded-score policies (such as the Softmax), we can improve Lemma~\ref{lem:bounded_score} using an empirical Bernstein inequality~\citep{maurer2009empirical}:
\begin{lemma}\label{lem:empirical_bernstein}
	Let $\norm{\nabla \log \pi_{\vtheta}(a\vert s)}\le W$ for all $\vtheta\in\Reals^d$, $s\in\Sspace$ and $a\in\Aspace$. Then, for any $\vtheta\in\Reals^d$, with probability $1-\delta$:
	\begin{equation}
		\norm{\hnabla J(\vtheta;\Dataset)-\nabla J(\vtheta)} \le \sqrt{\frac{8d\widehat{V}\log(12/\delta)}{N}}+ \frac{14dWR_T\log(6/\delta)}{3(N-1)}.
	\end{equation}
	where $\hat{V}=\frac{1}{N-1}\sum_{i=1}^{N}\norm{\hnabla J(\vtheta;\tau_i)-\hnabla J(\vtheta;\Dataset)}^2$, and $R_T$ is defined as in Lemma~\ref{lem:bounded_score}.
\end{lemma}
\begin{proof}
	Recall that $\Dataset=\{\tau_1,\dots,\tau_N\}$ is a set of trajectories sampled independently from $p_{\vtheta}$.
	Let $\hnabla J(\vtheta;\tau_i)$ denote the policy gradient estimate obtained from trajectory $\tau_i$, and recall $J(\vtheta{;}\Dataset)=\frac{1}{N}\sum_{i=1}^{N}\hnabla J(\vtheta;\tau_i)$ denotes the sample mean.
	Fix a vector $v\in S^{d-1}$, the unit sphere in $\Reals^d$, and let $X_i=\langle v, \hnabla J(\vtheta;\tau_i)\rangle$ for short. Then, as shown in~\eqref{eq:sg_bound}:
	\begin{align}
		\vert X_i\vert\le WR_T,\label{eq:sg_bound_bis}
	\end{align}
	and $\EV[X_i] = \langle v, \nabla J(\vtheta)\rangle$. Moreover, $(X_i)_{i=1}^N$ are \iid (conditionally on $\vtheta$, which is fixed in this case). By Theorem 4 from~\citep{maurer2009empirical}, with probability $1-\delta_v$:
	\begin{align}
		\langle v,\hnabla J(\vtheta{;}\Dataset)-\nabla J(\vtheta)\rangle \le \sqrt{\frac{2\widehat{V}_v\log(2/\delta_v)}{N}} + \frac{7WR_T\log(2/\delta_v)}{3(N-1)},
	\end{align}
	where  and $\widehat{V}_v$ is the (unbiased) sample variance of the $(X_i)_{i=1}^N$:
	\begin{align}
		\widehat{V}_v &= \frac{1}{N-1}\sum_{i=1}^{N}\left\langle v, \hnabla J(\vtheta;\tau_i)-\hnabla J(\vtheta;\Dataset)\right\rangle^2 \\
		&\le \frac{1}{N-1}\sum_{i=1}^{N}\norm{\hnabla J(\vtheta;\tau_i)-\hnabla J(\vtheta;\Dataset)}^2 \coloneqq \hat{V},
	\end{align}
	where the inequality is by Cauchy-Schwarz and  $\norm{v}=1$. Since $\hat{V}$ does not depend on $v$, we can use the same covering argument as in the proof of Lemma~\ref{lem:bounded_score} to obtain the desired result.
\end{proof}
To use this concentration inequality in SPG, Algorithm~\ref{algo:safepg} must be modified, as discussed in Appendix~\ref{app:relax}.

\subsection{Infinite-Horizon Estimators}\label{app:geom}
To obtain an unbiased estimate of the gradient for the original infinite-horizon performance measure considered in the paper, we can modify our simulation protocol as suggested in~\citep[][]{bedi2021sample}. Consider a random-horizon G(PO)MDP estimator that, for each episode:
\begin{enumerate}
	\item Samples a random horizon $T\sim \mathrm{Geom}(1-\gamma^{t/2})$ from a geometric distribution;
	\item Generates a trajectory $\tau$ of length $T$ with the current policy $\pi_{\vtheta}$;
	\item Outputs $\hnabla J(\vtheta;\tau,T)=\sum_{t=0}^{T-1}\left(\gamma^{t/2}\Rew(a_t^i,s_t^i)\sum_{h=0}^{t}\nabla\log\pi_{\vtheta}(a_h^i\vert s_h^i)\right)$.
\end{enumerate}
The result can be averaged over a batch of independent trajectories, each with its own independently sampled random length. This policy gradient estimator is unbiased~\citep[][Lemma 1]{bedi2021sample}. The random horizon should be accounted for in the concentration bounds of Lemma~\ref{lem:bounded_score},~\ref{lem:sg_score}, and~\ref{lem:empirical_bernstein}. However, note that the term $R_T$, for the G(PO)MDP estimator, is bounded as follows:
\begin{equation}
	R_T=\Rmax \frac{1-\gamma^{\top} - T(\gamma^{T}-\gamma^{T+1})}{(1-\gamma)^2} \le \frac{\Rmax}{(1-\gamma)^2},
\end{equation}
for any $T\ge 0$. Hence, Lemma~\ref{lem:bounded_score},~\ref{lem:sg_score}, and~\ref{lem:empirical_bernstein} all hold for the random-horizon estimator with $R_T=\Rmax/(1-\gamma^{1/2})^2$. The corresponding error bounds are reported in Table~\ref{tab:err}. We leave a more refined analysis of the variance and tail behavior of this random-horizon estimator to future work.

\section{Variance of Policy Gradient Estimators}\label{sec:var}
In this section, we provide upper bounds on the variance of the (finite-horizon) \textsc{REINFORCE} and \textsc{G(PO)MDP} estimators, generalizing existing results for Gaussian policies~\citep{zhao2011analysis,pirotta2013adaptive} to smoothing policies.
We begin by bounding the variance of the \textsc{REINFORCE} estimator:
\begin{restatable}{lemma}{reinforcevar}\label{lem:reinforcevar}
	Given a $(\sm,\smm,\smmm)$-smoothing policy class $\Pi_{\Theta}$ and an effective task horizon $T$, for every $\vtheta\in\Theta$, the variance of the \textsc{REINFORCE} estimator (with zero baseline) is upper-bounded as follows:
	\begin{align}
	\Var\left[\hnabla J(\vtheta{;} \Dataset)\right] \leq \frac{T\smm\Rmax^2(1-\gamma^{\top})^2}{N(1-\gamma)^2}.
	\end{align}
\end{restatable}
\begin{proof}
	Let $g_{\vtheta}(\tau) \coloneqq \left(\sum_{t=0}^{T-1}\gamma^t\Rew(a_t,s_t)\right)\left(\sum_{t=0}^{T-1}\nabla\log\pi_{\vtheta}(a_t\vert s_t)\right)$ with $s_t,a_t\in\tau$ for $t=0,\dots,T-1$.
	Using the definition of \textsc{REINFORCE}~(\ref{eq:reinforce}) with $b=0$:
	\begin{align}
	\Var_{\Dataset\sim p_{\vtheta}}&\left[\hnabla J(\vtheta{;} \Dataset)\right] 
	=\frac{1}{N}\Var_{\tau\sim p_{\vtheta}}\left[g_{\vtheta}(\tau)\right] \nonumber\\
	&\leq \frac{1}{N}\EV_{\tau\sim p_{\vtheta}}\left[\norm{g_{\vtheta}(\tau)}^2\right]\nonumber\\
	&\leq\frac{\Rmax^2(1-\gamma^{\top})^2}{N(1-\gamma)^2}\EV_{\tau\sim p_{\vtheta}}\left[\norm{\sum_{t=0}^{T-1}\nabla\log\pi_{\vtheta}(a_t\vert s_t)}^2\right]\nonumber\\
	&\leq \frac{\Rmax^2(1-\gamma^{\top})^2}{N(1-\gamma)^2}\sum_{i=1}^{m}\EV_{\tau\sim p_{\vtheta}}\left[\sum_{t=0}^{T-1}\left(D_i\log\pi_{\vtheta}(a_t\vert s_t)\right)^2 \right.\nonumber\\&\qquad+\left. 2\sum_{t=0}^{T-2}\sum_{h=t+1}^{T-1}D_i\log\pi_{\vtheta}(a_t\vert s_t)D_i\log\pi_{\vtheta}(a_h\vert s_h)\right] \nonumber\\
	& = \frac{\Rmax^2(1-\gamma^{\top})^2}{N(1-\gamma)^2}\EV_{\tau\sim p_{\vtheta}}\left[\sum_{t=0}^{T-1}\norm{\nabla\log\pi_{\vtheta}(a_t\vert s_t)}^2 \right] \label{eq:9.1}\\
	& = \frac{\Rmax^2(1-\gamma^{\top})^2}{N(1-\gamma)^2}\sum_{t=0}^{T-1}\EVV[s_0\sim\Init]{\dots
		\EV_{a_t\sim\pi_{\vtheta}(\cdot\vert s_t)}\left[\norm{\nabla\log\pi_{\vtheta}(a_t\vert s_t)}^2\,\middle\vert\, s_t\right]\dots}\nonumber\\
	&\leq \frac{T\smm\Rmax^2(1-\gamma^{\top})^2}{N(1-\gamma)^2},
	\end{align}
	where (\ref{eq:9.1}) is from the following:
	\begin{align}
	&\EVV[\tau\sim p_{\vtheta}]{\sum_{t=0}^{T-2}\sum_{h=t+1}^{T-1}D_i\log\pi_{\vtheta}(a_t\vert s_t)D_i\log\pi_{\vtheta}(a_h\vert s_h)} \nonumber\\
	&\quad=\sum_{t=0}^{T-2}\EV_{s_0\sim\Init}\left[\dots\EV_{a_t\sim\pi_{\vtheta}(\cdot\vert s_t)}\left[
	D_i\log\pi_{\vtheta}(a_t\vert s_t)
	\right.\right.\\&\quad\quad\left.\left.
	\sum_{h=t+1}^{T-1}\EV_{s_{t+1}\sim p(\cdot\vert s_t,a_t)}\left[
	\dots
	\EV_{a_h\sim\pi_{\vtheta}(\cdot\vert s_h)}\left[D_i\log\pi_{\vtheta}(a_h\vert s_h)\,\middle\vert\,s_h\right]
	\dots \,\middle\vert\,a_t
	\right]\,\middle\vert\,s_t\right]
	\dots\right]\nonumber\\
	&\quad= 0,
	\end{align}
	where the last equality is from $\EVV[a_h\sim\pi_{\vtheta}(\cdot\vert s_h)]{D_i\log\pi_{\vtheta}(a_h\vert s_h)} = 0$.
\end{proof}
This is a generalization of Lemma 5.3 from~\cite{pirotta2013adaptive}, which in turn is an adaptation of Theorem 2 from~\cite{zhao2011analysis}. In the Gaussian case, the original lemma is recovered by plugging the smoothing constant $\smm=\frac{\phimax^2}{\sigma^2}$ from Lemma~\ref{lem:gauss}. Note also that, from the definition of smoothing policy, only the second condition (\ref{eq:sm2}) is actually necessary for Lemma~\ref{lem:reinforcevar} to hold.

For the \textsc{G(PO)MDP} estimator, we obtain an upper bound that does not grow linearly with the horizon $T$:
\begin{restatable}{lemma}{gmdpvar}\label{lem:gmdpvar}
	Given a $(\sm,\smm,\smmm)$-smoothing policy class $\Pi_{\Theta}$ and an effective task horizon $T$, for every $\vtheta\in\Theta$, the variance of the \textsc{G(PO)MDP} estimator (with zero baseline) is upper-bounded as follows:
	\begin{align}
	\Var\left[\hnabla J(\vtheta{;} \Dataset)\right] \leq \frac{\smm\Rmax^2 \Big( 1-\gamma^{\top} \Big)}{N(1-\gamma)^3}.
	\end{align}
\end{restatable}
\begin{proof}
	Let $g_{\vtheta}(\tau) \coloneqq \sum_{t=0}^{T-1}\gamma^t\Rew(a_t,s_t)\left(\sum_{h=0}^{t}\nabla\log\pi_{\vtheta}(a_h\vert s_h)\right)$ with $s_t,a_t\in\tau$ for $t=0,\dots,T-1$.
	Using the definition of \textsc{G(PO)MDP}~(\ref{eq:gpomdp}) with $b=0$:
	\begin{align}
	\Var_{\Dataset\sim p_{\vtheta}}&\left[\hnabla J(\vtheta{;} \Dataset)\right]
	=\frac{1}{N}\Var_{\tau\sim p_{\vtheta}}\left[\sum_{t=0}^{T-1}\gamma^t\Rew(a_t,s_t)\left(\sum_{h=0}^{t}\nabla\log\pi_{\vtheta}(a_h\vert s_h)\right)\right] \label{eq:10.1}\\
	&\leq\frac{1}{N}\EV_{\tau\sim p_{\vtheta}}\left[\left(\sum_{t=0}^{T-1}\gamma^{\nicefrac{t}{2}}\Rew(a_t,s_t)\gamma^{\nicefrac{t}{2}}\left(\sum_{h=0}^{t}\nabla\log\pi_{\vtheta}(a_h\vert s_h)\right)\right)^2\right] \nonumber\\
	&\leq\frac{1}{N}\EV_{\tau\sim p_{\vtheta}}\left[\left(\sum_{t=0}^{T-1}\gamma^{t}\Rew(a_t,s_t)^2\right)\left(\sum_{t=0}^{T-1}\gamma^{t}\left(\sum_{h=0}^{t}\nabla\log\pi_{\vtheta}(a_h\vert s_h)\right)^2\right)\right] \label{eq:10.2}\\
	&\leq\frac{\Rmax^2(1-\gamma^{\top})}{N(1-\gamma)}\EV_{\tau\sim p_{\vtheta}}\left[\sum_{t=0}^{T-1}\gamma^{t}\left(\sum_{h=0}^{t}\nabla\log\pi_{\vtheta}(a_h\vert s_h)\right)^2\right]\nonumber\\
	&\leq\frac{\smm\Rmax^2(1-\gamma^{\top})}{N(1-\gamma)}\sum_{t=0}^{T-1}\gamma^{t}(t+1)\label{eq:10.3}\\
	&= \frac{\smm\Rmax^2(1-\gamma^{\top})}{N(1-\gamma)^3}\left[1-T\Big(\underbrace{\gamma^{\top} - \gamma^{T+1}}_{\geq 0}\Big) - \gamma^{\top}\right]\label{eq:10.4}\\
	&\leq \frac{\smm\Rmax^2 \Big( 1-\gamma^{\top} \Big)}{N(1-\gamma)^3}\nonumber,
	\end{align}
	where (\ref{eq:10.1}) is from the fact that the trajectories are \iid, (\ref{eq:10.2}) is from the Cauchy-Schwarz inequality, (\ref{eq:10.3}) is from the same argument used for (\ref{eq:9.1}) in the proof of Lemma~\ref{lem:reinforcevar}, and \eqref{eq:10.4} is from the sum of the arithmetico-geometric sequence.
\end{proof}
This is a generalization of Lemma 5.5 from~\cite{pirotta2013adaptive}. Again, in the Gaussian case, the original lemma is recovered by plugging the smoothing constant $\smm=\frac{\phimax}{\sigma^2}$ from Lemma~\ref{lem:gauss}. Note that this variance upper bound stays finite in the limit $T\to\infty$, which is not the case for \textsc{REINFORCE}.

\begin{table}[t]
	\caption{Upper bounds on the variance $\Vv(\vtheta)$ for common policy gradient estimators (single trajectory, no baseline), assuming the policy is smoothing (Definition~\ref{def:smooth}). Here $\Rmax$ is the maximum absolute-valued reward, $\gamma$ is the discount factor, $T$ is the task horizon, and the smoothing constant $\smm$ can be retrieved from Table~\ref{tab:smoothing} depending on the policy class.}
	\centering
	\begin{tabular}{cc}
		\toprule
		\textbf{REINFORCE} & \textbf{G(PO)MDP} \\
		\midrule
		$\frac{T\smm\Rmax^2(1-\gamma^{\top})^2}{(1-\gamma)^2}$
		&
		$\frac{\smm\Rmax^2 ( 1-\gamma^{\top} )}{(1-\gamma)^3}$
		\\
		\bottomrule
	\end{tabular}
	\label{tab:var}
\end{table}

\section{Analysis of Relaxed Algorithm}\label{app:relax}
In this section, we will analyze in more detailed the variants of SPG introduced in Section~\ref{sec:relax}. In particular, we will consider a very general \emph{relaxed} improvement guarantee, then we will specialize it to the baseline and milestone constraints discussed in the main paper.

The pseudocode for the relaxed version of SPG is provided in Algorithm~\ref{algo:spg2}.
\begin{algorithm}[t]
	\caption{Relaxed SPG}
	\label{algo:spg2}
	\begin{algorithmic}[1] 
		\State \textbf{Input:} initial policy parameter $\vtheta_0$, smoothness constant $L$, concentration bound $\epsilon$, failure probabilities $(\delta_k)_{k\ge 1}$, degradation thresholds $(\Delta_k)_{k\ge 1}$, mini-batch size~$n$
		\State $\alpha=\frac{1}{L}$ \Comment{fixed step size}
		\For{$k=1,2,\dots$}
		\State $i=0$, $\Dataset_{k,0} = \emptyset$
		\Do
		\State $i=i+1$
		\State Collect trajectories $\tau_{k,i,1}\dots\tau_{k,i,n}\simiid p_{\vtheta_k}$ 
		\State $\Dataset_{k,i}=\Dataset_{k,i-1}\cup\{\tau_{k,i,1}\dots\tau_{k,i,n}\}$
		\State Compute policy gradient estimate $g_{k,i} = \hnabla  J(\vtheta_k{;} \Dataset_{k,i})$
		\State $\delta_{k,i}=\frac{\delta_k}{i(i+1)}$
		\doWhile{$\epsilon(ni,\delta_{k,i}) > \frac{\norm{g_{k,i}}}{2} + \frac{L\Delta_k}{\norm{g_{k,i}}}$}
		\State $N_k=ni$, $\Dataset_k=\Dataset_{k,i}$ \Comment{adaptive batch size}
		\State Update policy parameters as $\vtheta_{k+1} \gets \vtheta_k + 
		\alpha\hnabla  J(\vtheta_k{;} \Dataset_k)$
		\EndFor
	\end{algorithmic}
\end{algorithm}
The algorithm takes as additional inputs the mini-batch size $n$ and a sequence of degradation thresholds $\Delta_k\ge 0$. Moreover, it assumes access to a generic gradient estimation error function $\epsilon$ with the following property:
\begin{assumption}\label{asm:general}
	Fixed a parameter $\vtheta\in\Theta$, a batch size $N\in\Naturals$ and a failure probability $\delta\in(0,1)$, with probability at least $1-\delta$: 
	\begin{equation*}
		\norm{\hnabla J(\vtheta{;}\Dataset)-\nabla J(\vtheta)} \le \epsilon(N, \delta),
	\end{equation*}
	where $\vert\Dataset\vert$ is a dataset of $N$ \iid trajectories collected with $\pi_{\vtheta}$ and:
	\begin{equation}
		\epsilon(N,\delta)=\mathcal{O}\left(\frac{\log(1/\delta)}{\sqrt{N}}\right).
	\end{equation}
\end{assumption}

The assumption, as the analysis that will follow, is less precise than Assumption~\ref{asm:expconc}, but more general. Indeed, it allows to use the empirical Bernstein bound from Lemma~\ref{lem:empirical_bernstein} for Softmax and other bounded-score policies.
We can prove that the per-iteration performance degradation of Algorithm~\ref{algo:spg2} is bounded by the user-defined threshold $\Delta_k$ with high probability. Of course, when $\Delta_k=0$, this is still a monotonic improvement guarantee.
\begin{theorem}
	Consider Algorithm~\ref{algo:spg2} applied to a smoothing policy, where $\hnabla J$ is an unbiased policy gradient estimator.
	Under Assumption~\ref{asm:general}, for any iteration~${k\ge 1}$, provided $\nabla J(\vtheta_k)\neq0$, with probability at least $1-\delta_k$:
	\begin{equation*}
		J(\vtheta_{k+1}) - J(\vtheta_k) \ge -\Delta_k.
	\end{equation*}
\end{theorem}
\begin{proof}
		Let the filtration $(\mathcal{F}_{k,i})_{i\ge1}$ be defined as in Section~\ref{sec:algo} and note that $N_k$ is a stopping time \wrt this filtration. Consider the event $E_{k,i}=\big\{\norm{g_{k,i}-\nabla J(\vtheta_k)}\le{\epsilon(i, \delta_{k,i})}\big\}$. By Assumption~\ref{asm:expconc}, $\Pro(\lnot E_{k,i})\le\delta_{k,i}$.
	Hence, by the same arguments used in the proof of Lemma~\ref{lem:stopping}:
	\begin{align}
		\EV[N_k] &\le \EV\left[\sum_{i=1}^\infty\mathbb{I}\left(\epsilon(ni,\delta_{k,i}) > \frac{\norm{g_{k,i}}}{2} + \frac{L\Delta_k}{\norm{g_{k,i}}}\right)\right]\\
		&\le \EV\left[\sum_{i=1}^\infty\mathbb{I}\left(\epsilon(ni,\delta_{k,i}) > \frac{\norm{g_{k,i}}}{2} \right)\right]\\
		&=\EV\left[\sum_{i=1}^\infty\mathbb{I}\left(\epsilon(ni,\delta_{k,i}) > \frac{\norm{g_{k,i}}}{2} , E_{k,i}\right)\right] \nonumber\\&\qquad+ \EV\left[\sum_{i=1}^\infty\mathbb{I}\left(\epsilon(i,\delta_{k,i}) > \frac{\norm{g_{k,i}}}{2} , \lnot E_{k,i}\right)\right] \\
		&\le\sum_{i=1}^\infty\mathbb{I}\left(\epsilon(ni,\delta_{k,i}) > \frac{\left(\norm{\nabla J(\vtheta_k)}-\epsilon(ni,\delta_{k,i})\right)}{2}\right) + \sum_{i=1}^\infty\Pro(\lnot E_{k,i}) \\
		&\le \min_{i\ge 1}\left\{\epsilon(ni,\delta_{k,i}) \le \frac{\left(\norm{\nabla J(\vtheta_k)}-\epsilon(ni,\delta_{k,i})\right)}{2}\right\} + \sum_{i=1}^\infty \delta_{k,i} \\
		&\le\min_{i\ge 1}\left\{\epsilon(ni,\delta_{k,i}) \le \frac{\norm{\nabla J(\vtheta_k)}}{3}\right\}) + \delta_k\sum_{i=1}^\infty \frac{1}{i(i+1)} \\
		&=\min_{i\ge 1}\left\{\epsilon(ni,\delta_{k,i}) \le \frac{\norm{\nabla J(\vtheta_k)}}{3}\right\}) + \delta_k,
	\end{align}
	which is finite since, by Assumption~\ref{asm:general}:
	\begin{equation}
		\epsilon(ni,\delta_{k,i})=O\left(\frac{\log(1/\delta_{k,i})}{\sqrt{ni}}\right) = O\left(\frac{\log i}{\sqrt{i}}\right).
	\end{equation}
	This shows that the inner loop of Algorithm~\ref{algo:spg2} always terminates with a finite batch size. By the same optional-stopping argument as in the proof of Theorem~\ref{th:mi}, $\EV[\hnabla J(\vtheta_k{;}\Dataset_k)]=\nabla J(\vtheta_k)$, which means the gradient estimate is unbiased. By the stopping condition, for all $k$:
	\begin{equation}\label{pp:threshold}
		\epsilon(N_k,\delta_{k,N_k}) \le \frac{\norm{\hnabla  J(\vtheta_k{;} \Dataset_k)}}{2} + \frac{L\Delta_k}{\norm{\hnabla  J(\vtheta_k{;} \Dataset_k)}},
	\end{equation}
	with probability at least $1-\sum_{i=1}^\infty\delta_{k,i}=1-\sum_{i=1}^\infty\delta_k/(i(i+1))=1-\delta_k$.
	By Theorem~\ref{th:storacle} and the choice of step size $\alpha=1/L$, with the same probability:
	\begin{align}
		J(\vtheta_{k+1}) - J(\vtheta_{k}) &\ge \frac{\norm{\hnabla J(\vtheta_k;\Dataset_k)}}{L}\left(\frac{\norm{\hnabla J(\vtheta_k;\Dataset_k)}}{2} - \epsilon(N_k, \delta_{k,N_k})\right)\\
		& \ge -\Delta_k,
	\end{align}
	where the last inequality is from~\eqref{pp:threshold}.
\end{proof}

In the following we discuss some applications of Algorithm~\ref{algo:spg2} to specific safety requirements.

\paragraph{Baseline Constraint.} A common requirement is for the updated policy not to perform (significantly) worse than a known baseline policy~\citep[\eg][]{garcelon2020conservative,laroche2019safe}. The safety constraint is thus:
\begin{equation}
	J(\vtheta_{k+1}) \ge \lambda J_b,
\end{equation}
where $J_b$ is the (discounted) performance of the baseline policy and $\lambda\in[0,1]$ is a user-defined significance parameter. Equivalently, $J(\vtheta_{k+1})-J(\vtheta_{k})\ge \lambda J_b-J(\vtheta_{k})$, and Algorithm~\ref{algo:spg2} satisfies this safety requirement if we set the degradation threshold as follows:
\begin{equation}
	\Delta_k = \max\{J(\vtheta_k) - \lambda J_b, 0\}.
\end{equation}
However, the performance of the current policy must also be estimated from data, and accidentally over-estimating it may result in excessive performance degradation. Hence, we replace it with a lower confidence bound based on the empirical Bernstein inequality~\citep{maurer2009empirical}. See Algorithm~\ref{algo:baseline_pg} for details. Note how the failure probability in line 10 is adjusted \wrt Algorithm~\ref{algo:spg2} to account for this additional estimation step. With this small caveat, the analysis of Algorithm~\ref{algo:baseline_pg} can be carried out analogously to the one of Algorithm~\ref{algo:spg2}.
\begin{algorithm}[t]
	\caption{SPG with baseline constraint}
	\label{algo:baseline_pg}
	\begin{algorithmic}[1] 
		\State \textbf{Input:} initial policy parameter $\vtheta_0$, smoothness constant $L$, concentration bound $\epsilon$, failure probabilities $(\delta_k)_{k\ge 1}$, baseline performance $J_b$, significance parameter $\lambda$, mini-batch size~$n$
		\State $\alpha=\frac{1}{L}$ \Comment{fixed step size}
		\For{$k=1,2,\dots$}
		\State $i=0$, $\Dataset_{k,0} = \emptyset$
		\Do
		\State $i=i+1$
		\State Collect trajectories $\tau_{k,i,1}\dots\tau_{k,i,n}\simiid p_{\vtheta_k}$ 
		\State $\Dataset_{k,i}=\Dataset_{k,i-1}\cup\{\tau_{k,i,1}\dots\tau_{k,i,n}\}$
		\State Compute policy gradient estimate $g_{k,i} = \hnabla  J(\vtheta_k{;} \Dataset_{k,i})$
		\State $\delta_{k,i}=\frac{\delta_k}{2i(i+1)}$
		\State Estimate performance mean $\hat{J}$ and variance $\hat{V}$  from $\Dataset_{k,i}$
		\State $\underline{J} = \hat{J} - \sqrt{\frac{2\hat{V}\log(2/\delta_{k,i})}{ni}} - \frac{7\Rmax\log(2/\delta_{k,i})}{3(1-\gamma)(ni-1)}$
		\State $\Delta_{k,i}= \max\{\underline{J} - \lambda J_b,0\}$\Comment{baseline constraint}
		\doWhile{$\epsilon(ni,\delta_{k,i}) > \frac{\norm{g_{k,i}}}{2} + \frac{L\Delta_{k,i}}{\norm{g_{k,i}}}$}
		\State $N_k=ni$, $\Dataset_k=\Dataset_{k,i}$ \Comment{adaptive batch size}
		\State Update policy parameters as $\vtheta_{k+1} \gets \vtheta_k + 
		\alpha\hnabla  J(\vtheta_k{;} \Dataset_k)$
		\EndFor
	\end{algorithmic}
\end{algorithm}

\paragraph{Milestone Constraint.}
In our numerical simulations we consider the following safety constraint:
\begin{equation}
	J(\vtheta_{k+1}) \ge \lambda \max_{h=1,\dots,k}J(\vtheta_h),
\end{equation}
which can be enforced by setting the degradation threshold in Algorithm~\ref{algo:spg2} as:
\begin{equation}
	\Delta_k = \max\left\{J(\vtheta_k) - \lambda\max_{h=1,\dots,k}J(\vtheta_h), 0\right\}.
\end{equation}
Again, we must replace the unknown performance $J(\vtheta_k)$ with a lower confidence bound. In this case, we also need to overestimate the best historical performance. See Algorithm~\ref{algo:milestone_pg} for details.
Note that this safety constraint reduces to monotonic improvement if the significance parameter is set to $\lambda=1$, since $\max_{h=1,\dots,k}J(\vtheta_h)\ge J(\vtheta_k)$.

\begin{algorithm}[t]
	\caption{SPG with milestone constraint}
	\label{algo:milestone_pg}
	\begin{algorithmic}[1] 
		\State \textbf{Input:} initial policy parameter $\vtheta_0$, smoothness constant $L$, concentration bound $\epsilon$, failure probabilities $(\delta_k)_{k\ge 1}$, baseline performance $J_b$, significance parameter $\lambda$, mini-batch size~$n$
		\State $\alpha=\frac{1}{L}$ \Comment{fixed step size}
		\State $J^\star =-\infty$
		\For{$k=1,2,\dots$}
		\State $i=0$, $\Dataset_{k,0} = \emptyset$
		\Do
		\State $i=i+1$
		\State Collect trajectories $\tau_{k,i,1}\dots\tau_{k,i,n}\simiid p_{\vtheta_k}$ 
		\State $\Dataset_{k,i}=\Dataset_{k,i-1}\cup\{\tau_{k,i,1}\dots\tau_{k,i,n}\}$
		\State Compute policy gradient estimate $g_{k,i} = \hnabla  J(\vtheta_k{;} \Dataset_{k,i})$
		\State $\delta_{k,i}=\frac{\delta_k}{2i(i+1)}$
		\State Estimate performance mean $\hat{J}$ and variance $\hat{V}$  from $\Dataset_{k,i}$
		\State $\underline{J} = \hat{J} - \sqrt{\frac{2\hat{V}\log(2/\delta_{k,i})}{ni}} - \frac{7\Rmax\log(2/\delta_{k,i})}{3(1-\gamma)(ni-1)}$
		\State $\overline{J} = \max\left\{\hat{J} + \sqrt{\frac{2\hat{V}\log(2/\delta_{k,i})}{ni}} + \frac{7\Rmax\log(2/\delta_{k,i})}{3(1-\gamma)(ni-1)}, J^\star\right\}$
		\State $\Delta_{k,i}= \max\{\underline{J} - \lambda\overline{J},0\}$\Comment{milestone constraint}
		\doWhile{$\epsilon(ni,\delta_{k,i}) > \frac{\norm{g_{k,i}}}{2} + \frac{L\Delta_{k,i}}{\norm{g_{k,i}}}$}
		\State $N_k=ni$, $\Dataset_k=\Dataset_{k,i}$ \Comment{adaptive batch size}
		\State $J^\star = \max\{\overline{J}, J^\star\}$
		\State Update policy parameters as $\vtheta_{k+1} \gets \vtheta_k + 
		\alpha\hnabla  J(\vtheta_k{;} \Dataset_k)$
		\EndFor
	\end{algorithmic}
\end{algorithm}

\section{Task Specifications}\label{app:task}
In this Appendix, we provide detailed descriptions of the control tasks used in the numerical simulations.

\subsection{LQR}\label{app:lqg}
The {LQR} is a classical optimal control problem~\citep{dorato1994linear}. It models the very general task of controlling a set of variables to zero with the minimum effort. Given a state space $\Sspace\subseteq\Reals^n$ and an action space $\Aspace\subseteq\Reals^m$, the next state is a linear function of current state and action:\footnote{A zero-mean Gaussian noise is typically added to the next state to model disturbances. However, since we always consider Gaussian policies with fixed standard deviation, we can ignore the system noise without loss of generality. Indeed, from linearity of the next state, said $\overline{a}_t$ the expected action under~\eqref{eq:lqpol}, $s_{t+1}=As_t+B\overline{a}_t+B\epsilon$, where $\epsilon\sim\mathcal{N}(0,\sigma^2\Id)$. From the property of Gaussians, we can write $\epsilon = \epsilon_a+\epsilon_b$ where $\epsilon_a\sim\mathcal{N}(0,\sigma_a^2\Id)$ is from the actual stochasticity of the agent and $\epsilon_b\sim\mathcal{N}(0,\sigma_b^2B^{\dagger})$ is the system noise, which can be subsumed by the policy noise in numerical simulations for simplicity.}
\begin{equation}
s_{t+1} = As_t + Ba_t,
\end{equation}
where $A\in\Reals^{n\times n}$ and $B\in\Reals^{n\times m}$. The reward is quadratic in both state and action:
\begin{equation}
r_{t+1} = s_t^{\top}Cs_t + a_t^{\top}Da_t,
\end{equation}
where $C\in\Reals^{n\times n}$ and $D\in \Reals^{m\times m}$ are positive definite matrices. A linear controller is optimal for this task~\citep{dorato1994linear} and can be computed in closed form with dynamic-programming techniques.  In our experiments, we always consider shallow Gaussian policies of the form:
\begin{equation}\label{eq:lqpol}
\pi(\cdot\vert s_t) = \mathcal{N}(\vtheta^{\top}s_t, \sigma^2\Id),
\end{equation}
where $\vtheta\in\Reals^n$ and $\sigma>0$ can be fixed or learned as an additional policy parameter. This version of {LQR} with Gaussian policies is also called LQG~\citep[Linear-Quadratic Gaussian Regulator,][]{peters2008reinforcement}. States and actions are clipped in practice when they happen to fall outside $\Sspace$ and $\Aspace$, respectively. We have ignored nonlinearities stemming from this fact.

The LQR problem used in Section~\ref{sec:exp} is $1$-dimensional with $\mathcal{S}=\Aspace=[-1,1]$, $A=B=C=D=1$.

\subsection{Cart-Pole}\label{app:cartpole}
This is the \texttt{CartPole-v1} environment from \texttt{openai/gym}~\citep{brockman2016openai}. It has $4$-dimensional continuous states and finite (two) actions. The goal is to keep a pole balanced by controlling a cart to which the pole is attached. Reward is $+1$ for every time-step until the pole falls. We set a maximum episode length of $100$. See the official documentation for more details~(\url{https://gym.openai.com/envs/CartPole-v1/}).




\end{appendices}



\clearpage

\bibliography{biblio}


\end{document}